\documentclass[preprint,12pt]{elsarticle}

\usepackage[colorlinks,citecolor=blue,urlcolor=blue]{hyperref}
\usepackage{amsthm}
\usepackage{amsmath}
\usepackage{amssymb}
\usepackage{amsfonts}
\usepackage{graphicx}
\usepackage{xcolor}         
\usepackage{mathrsfs}
\usepackage{commath}
\usepackage{mathtools}
\usepackage[bb=dsserif]{mathalpha}
\usepackage{bm}
\usepackage[shortlabels]{enumitem}
\usepackage{dsfont}
\usepackage{bbm}
\usepackage{caption}
\usepackage{subcaption}
\usepackage{wrapfig}
\usepackage{float}
\usepackage{threeparttable}
\usepackage{booktabs}
\usepackage{multirow}
\usepackage{pifont}
\newcommand{\cmark}{\ding{51}}
\newcommand{\xmark}{\ding{55}}

\DeclareMathOperator*{\argmin}{argmin}

\theoremstyle{plain}

\newtheorem{theorem}{Theorem}[section]

\newtheorem{lemma}[theorem]{Lemma}
\newtheorem{corollary}[theorem]{corollary}

\theoremstyle{remark}
\newtheorem{definition}[theorem]{Definition}

\newtheorem{remark}{Remark}
\newtheorem{assumption}{Assumption}

\newcommand{\floor}[1]{\lfloor #1 \rfloor}
\newcommand{\ceil}[1]{\lceil {#1} \rceil}

\journal{Nuclear Physics B}

\begin{document}

\begin{frontmatter}



\title{Concurrent Density Estimation with Wasserstein Autoencoders: Some Statistical Insights}


\author[inst1]{Anish Chakrabarty}

\affiliation[inst1]{organization={Statistics and Mathematics Unit},
            addressline={Indian Statistical Institute, Kolkata}}

\author[inst2]{Arkaprabha Basu}
\author[inst2]{Swagatam Das}

\affiliation[inst2]{organization={Electronics and Communication Sciences Unit},
            addressline={Indian Statistical Institute, Kolkata}}

\begin{abstract}
Variational Autoencoders (VAEs) have been a pioneering force in the realm of deep generative models. Amongst its legions of progenies, Wasserstein Autoencoders (WAEs) stand out in particular due to the dual offering of heightened generative quality and a strong theoretical backbone. WAEs consist of an encoding and a decoding network--- forming a bottleneck--- with the prime objective of generating new samples resembling the ones it was catered to. In the process, they aim to achieve a target latent representation of the encoded data. Our work is an attempt to offer a theoretical understanding of the machinery behind WAEs. From a statistical viewpoint, we pose the problem as concurrent density estimation tasks based on neural network-induced transformations. This allows us to establish deterministic upper bounds on the realized errors WAEs commit. We also analyze the propagation of these stochastic errors in the presence of adversaries. As a result, both the large sample properties of the reconstructed distribution and the resilience of WAE models are explored.  
\end{abstract}




\end{frontmatter}



\section{Introduction}


Variational Autoencoder (VAE) \citep{Kingma2014} is one of the earlier agents of the modern-day revolution we call deep generative modeling. Vanilla autoencoders (AE), a precursor to VAEs, being used primarily for representation learning, were incapable of adding variation to the reconstructed signal. As such, they could not `generate' new observations resembling the target. VAEs came into being with the promise of overcoming this limitation, inspiring numerous variants in the process \citep{VAE_survey}. Perhaps the one that stirs up a statistician's intrigue the most is the Wasserstein Autoencoder (WAE) \citep{WAE}. Approaching the problem from an optimal transport (OT) point of view, it achieved significant improvement in generated image quality. 

The discussion regarding deep generative models starts with an unknown target probability distribution $\mu$. A model is devised to simulate new observations from the same by learning it gradually based on samples. While it is demanding to imagine a data set consisting of images following such a distribution, they can be readily deemed as residents of a high-dimensional non-Euclidean space, perhaps manifolds. However, in our discussion, we surmise that $\mu$ is defined on a Borel subset $\mathcal{X}$ of $\mathbb{R}^d$. This becomes a reasonable starting point for our discussion based on the well-known fact that the information necessary to `represent' an image typically possesses a low-dimensional structure compared to its ambient dimension $d$ \citep{bengio2013representation}. There lie two constituents in a typical WAE model: an `encoder' ($E$), and a `decoder' ($D$). It is the encoder that explores the prospect of achieving a low-dimensional representation of the data. Sampled observations from $\mu$ are fed into the encoder, which is tasked with producing replicates of such a reduced dimension. As such, it may be viewed as a parametric class of Borel functions from $\mathcal{X}$ to the `latent space' $\mathcal{Z} \subseteq \mathbb{R}^k$, $d > k$. In practice, both encoders and decoders are parameterized by neural networks (NNs). The goal of the encoding exercise is to reach a desired distribution $\rho$ defined on this space, fittingly called the `latent law'. Evidently, there must remain some discrepancy between the encoded and the desired latent distributions. \citet{WAE} prescribe the usage of Jensen-Shannon divergence (JS) and Maximum Mean Discrepancy (MMD) to encapsulate this `latent loss'. This quantity makes a major contribution toward the overall objective that drives WAEs. It is also the target latent law that inspires smooth interpolation between modes of $\mu$ while generating new observations. 

Once the encoding is over, reconstruction must take place. Decoders can be similarly described as the class of functions (from $\mathcal{Z} \rightarrow \mathcal{X}$) that aim at inducing inverse maps to those brought in by the encoders. Encoded observations go through such a transformation in an attempt to get back to where they originally came from, $\mu$. The deviation of the regenerated distribution from the input law makes for the reconstruction error. Needless to say, in a WAE model, this loss is represented by the Wasserstein distance (WD). 

Before laying the groundwork for a comprehensive statistical analysis of WAEs, one must acknowledge the accruing wisdom that has led us to this point.

\subsection{On Related Literature}

Chasing after the remarkable empirical success, theoretical explanations corresponding to deep generative models have come a long way. Riding the late surge are Generative Adversarial Networks (GAN) \citep{arora2017generalization,liu,biau,belomestny2021rates} and their immediate descendants (e.g. WGAN \citep{biau_WGAN}, Bidirectional GAN \citep{bigan} etc.). Compared to such a sensation, VAEs seem underappreciated when it comes to statistical scrutiny of their machinery. Though philosophically distinctive, some effort has been put into establishing an equivalence between GANs and autoencoder-based models. Borrowing from GANs the adversarial behavior of partaking network components, Adversarial AE (AAE) \citep{makhzani2015adversarial} was conceived. The introduction of adversarial training into VAEs saw them expressing arbitrary class of latent laws by posing the posterior maximum likelihood estimation as a two-player game \citep{unif}. In spirit, this made VAEs at par with GANs. The resemblance between the two architectures became even more cogent when seen through the lens of variational inference \citep{hu2018on}. However, the first evidence of a VAE-variant with comparable generative performance to that of GANs came in the form of WAE \citep{WAE}. \citet{husain2019primal} supported this empirical similitude theoretically by showing a primal-dual relationship between the two objectives. This fact motivates us to transport the cumulative knowledge from statistical inquiries regarding GANs to WAEs.

The scarcity of rigorous statistical studies corresponding to WAEs makes the existing ones even more precious. It is well known by now that a VAE model with Gaussian decoders behaves similarly to Robust PCA \citep{candes2011robust} under mild assumptions. As a result, such VAEs are capable of recovering uncorrupted observations hailing from input data manifolds, fending off outliers \citep{dai2018connections}. However, an agency of WAE architectures towards robust reconstruction lies unchecked. The Gaussian assumptions on both the encoder and decoder networks also have a profound impression on the VAE's capability to reconstruct the input law. \citet{dai} show that in case the data manifold has the full ambient dimension, reaching the global minima of the VAE loss is equivalent to ensuring a successful recovery. However, for image data, where the observations typically have a lower-dimensional true representation, non-unique solutions may exist. Similar avenues for WAEs still await to be explored. In an earlier work \citep{chakrabarty2021statistical}, we set out to answer some of such questions. The paper reformulated the WAE objective as simultaneous density estimation tasks, a viewpoint adopted previously by statistical analyses of GANs. In this work, we intend to build on top of the groundwork.

\subsubsection{A Note on the Latent Dimension} 

Practitioners and theorists are often divided based on their perceptions. However, the problem that unites them in shared discomfort is the precise prescription of the latent dimension $k$. The question remains simple: `Given a set of samples from a distribution, what should be the extent of dimensionality reduction (DR) such that they can be reconstructed?'. In generative exercises, however, the data distribution lies unknown, unlike the ambient dimension of its support. Clearly, the answer should ideally be multifaceted. There lie several aspects, heavily intertwined, that contribute to the complexity. 

The first hint comes from the input data dimension itself. Signals from naturally occurring events are mere instantiations of underlying random processes. Variability in a set of observations is rooted in this very idea. The explanatory attributes and their corresponding directions encapsulate this variation, giving rise to the notion of `dimensionality' of the data. As characterized by \cite{bennett1969intrinsic}, this quantity is formally known as the \textit{embedding} (ambient) dimension \citep{eneva2002wekkem}. However, dealing with high-dimensional real datasets (e.g. images) we have come to observe that such a space tends to have a lower-dimensional structure (typically submanifolds $\mathcal{M}$) where most of the variation lies, with a high probability \citep{fefferman2016testing}. The smaller set of directions the `Manifold Hypothesis' points at is called the \textit{intrinsic dimension} (ID). While there is significant disagreement between authors regarding the exact definition of the same (e.g. Minkowski dimension), we recognize ID as the topological dimension of $\mathcal{M}$. Several attempts have been made to estimate its ID given the data distribution \citep{levina2004maximum,facco2017estimating,pope2021the}. We emphasize the importance of such an intrinsic pattern of the signal to reflect on the encoded law as well. If one goes by the notion of ID being the set of independent dimensions that capture most of the variation in the dataset, $k$ should be a near-estimate of it. 

Since WAEs are restricted to reconstructing input observations, they must preserve as much \textit{information} as possible while encoding. In our density estimation regime, the notion of `information' is somewhat different from that offered by geometry. While the responsibility to preserve local and broader geometric signatures (based on topologies) lies with the encoder transform, our impression of the statistical information being conserved is that `the estimates perform with comparable accuracy even after being pushed forward'. Section \ref{encoder} elaborates on the same. Observe that the necessity to learn a latent representation puts an upper bound to the encoded dimension. At the same time, the need to preserve information hints at the existence of a lower bound. As such, the dimensions in between these two extremities invoke a trade-off between the accuracy of achieved \textit{representation} and the amount of information lost. 

Along this discussion comes the call to clarify what we mean by a good representation. Though WAEs prioritize the task of regeneration, one must not forget the roots of its predecessors in learning \textit{disentangled} representation. Without a robust definition, the idea of disentanglement is marred by subjectivity. Most, however, deem it as the process of compartmentalizing information into groups of independent `semantic' attributes \citep{bengio2013representation,higgins2018towards}. The underlying assumption being $\mathcal{M} = \bigcup_{j=1}^{l} \mathcal{M}_{j}$, where $l$ is the number of such groups and $\mathcal{M}_{j}$ are the support submanifolds. The notion of independence may be softened to `uncorrelatedness' in case group actions are linear \citep{higgins2018towards}. \citet{yu2020learning} argue that additionally, such representations should be between-class heterogeneous and within-class homogeneous to the greatest extent. However, based on human perception, no measurement of this extent can summarize the whole picture \citep{Do2020Theory}. This discussion finds great motivation in \cite{rubenstein2018learning}'s experiments showing WAEs as efficient representation learners. With further regularization on the latent space, we may expect to enhance the efficiency in both static \citep{gaujac2021learning} and dynamic \citep{han2021disentangled} data regimes. From a statistical viewpoint, we understand disentanglement as the process of attaining a distribution with a block diagonal (axis-aligned as a special case) dispersion matrix. This should ideally be the characterization of the latent distribution. In other words, a disentangled law will be our key to the latent dimension. However, finding such a law, devoid of inductive biases, in an unsupervised setting is theoretically impossible \citep{locatello2019challenging}.

It is evident that a typical WAE model, during encoding, performs a non-linear dimensionality reduction. The standard convolutional architecture carries out a feature aggregation in the process that is intractable and is not expected to attain a disentangled law without additional regularization \citep{kim2018disentangling, mathieu2019disentangling}. Thus, instead of pursuing the optimal value of $k$ directly, we turn our focus on the transformation induced by the encoder. The resilience of such functions against \textit{distortion} along with their regularity becomes paramount in the following discussion. 





\subsection{Our Contributions}
Key highlights of our upcoming discussion are as follows:

\newlist{myromanenum}{enumerate}{1}
\setlist[myromanenum,1]{label={(\roman*)}, left=0pt, align=left, itemsep=3pt, leftmargin=*}

\begin{myromanenum}
    \item We introduce a probabilistic characterization of the notion \textit{information preservation}, which becomes the cornerstone of our depiction of ideal encoders in a WAE model [Section \ref{encoder}]. We explore the measures of discrepancies over classes of probability measures that allow information preservation based on naive estimators. Along the line, this enables us to prescribe suitable model architectures that foster consistency of estimators in the latent space.
    \item We establish deterministic upper bounds on the latent loss incurred by WAE models in a non-parametric regime [Theorem \ref{latent_con_MMD}]. Our approach turns out to be versatile in the sense that given the input data distribution ($\mu$) possesses intrinsic structures, they can be readily made free of the input dimension ($d$). The bounds can be further improved towards greater generality if the target latent law ($\rho$) is invariant to group actions. In the process, we explore the desirable properties of underlying kernels in a WAE-MMD setup that allows latent space consistency.
    \item In Section \ref{Recon_Conc}, following a density estimation approach, we derive non-asymptotic sharp upper bounds on the realized reconstruction loss in WAEs. The regeneration guarantees come with accompanying prescriptions of ideal decoder networks. The bounds hint at the extent to which optimization errors incurred in the latent space propagate to reconstruction losses. All of our theoretical results are empirically substantiated by numerical experiments based on real and simulated data sets [Section \ref{sim_latent}, \ref{sim_Recon}].
    \item We additionally examine the effects of \textit{contamination} in input data on the performance of WAEs in reconstructions [Section \ref{Robust}]. The discussion includes desirable properties of kernel estimates that limit the corruption in translated data. Derived upper bounds to regeneration errors under distribution shift test WAEs' inherent capability to offer robustness against such adversaries. 
\end{myromanenum}
To maintain lucidity in our elaborate discussion, all proofs of theorems and additional lemmas are deferred to Appendix \ref{appA}.

\section{Preliminaries}
This section is devoted to laying the groundwork in terms of basic definitions and the statistical formulation of the WAE problem thereafter. The input data space $\mathcal{X}$, equipped with the metric $c_{x}$ is taken to be Polish. For most real scenarios, a typical characterization of the same is $\mathbb{R}^d$, $d \geq 1$. We refer to the space of probability measures defined on $\mathcal{X}$ as $\mathcal{P}(\mathcal{X})$. The same conventions follow for the latent space $\mathcal{Z} \subseteq \mathbb{R}^k$ ($k < d$), equipped with the metric $c_{z}$, as well. The class of measurable functions mapping $\mathcal{X}$ to $\mathcal{Z}$ is denoted by $\mathscr{F}(\mathcal{X},\mathcal{Z})$. For ease of understanding, we abbreviate the `encoder' and `decoder' transforms as $E$ and $D$ respectively. Given non-negative real sequences $\{a_n\}_{n \in \mathbb{N}}$ and $\{b_n\}_{n \in \mathbb{N}}$, the suppression of the universal constant $C> 0$, such that $\limsup_{n \rightarrow \infty}\frac{a_n}{b_n} \leq C$, is represented as $a_n \lesssim b_n$, or $a_n = \mathcal{O}(b_n)$. We also denote $\max\{x,y\} := x \vee y$ and $\min\{x,y\} := x \wedge y$.

\begin{definition}[Push-forward]
    Given $f \in \mathscr{F}(\mathcal{X},\mathcal{Z})$, the \textit{push-forward} of $\mu \in \mathcal{P}(\mathcal{X})$ is defined as $f_{\#}\mu(\omega) =  \mu(f^{-1}(\omega))$, where $\omega \in \sigma(\mathcal{Z})$.
\end{definition}

\begin{definition}[Integral Probability Metric \citep{muller1997integral}]
    For a class of bounded, measurable evaluation functions $\mathcal{F} = \left\{f : \mathcal{X} \rightarrow \mathbb{R}\right\}$, the \textit{integral probability metric} (IPM) measuring the discrepancy between $\mu, \nu \in \mathcal{P}(\mathcal{X})$ is given by
    \begin{equation*}
        d_{\mathcal{F}}(\mu, \nu) = \sup_{f \in \mathcal{F}}\Big\{\int_{\mathcal{X}}f(x)d\mu (x) - \int_{\mathcal{X}}f(x)d\nu (x)\Big\}.
    \end{equation*}
\end{definition}

\begin{remark}
    A particular variant of this measure we frequent in our discussion is the \textit{Maximum Mean Discrepancy} (MMD). It is obtained by taking $\mathcal{F}$ as the unit ball in a reproducing kernel Hilbert space (RKHS) $\mathcal{H}$, i.e. $\mathcal{F} = \left\{f: \norm{f}_{\mathcal{H}} \leq 1\right\}$. Moreover, a continuous kernel $\kappa(\cdot, \cdot)$ based on a compact metric space that results in $\mathcal{H}$--- dense in the space of bounded continuous functions--- will compel the associated MMD to be a metric \citep{gretton2012kernel}. The $1$-Wasserstein metric also boils down to an IPM, given that the underlying class of critics $\mathcal{F}:= \mathcal{L}_{c_{x}}^1$, i.e. $1$-Lipschitz functions with respect to $c_{x}$ (\citet{villani2009optimal}, Remark 6.5). This duality may not hold in general, which is evident from the definition of the $r^{th}$ Wasserstein distance:
    \begin{equation*}
        W^{r}_{c_{x}}(\mu,\nu) = \inf_{\gamma \in \Gamma(\mu,\nu)}\Big\{\int_{\mathcal{X}\times\mathcal{X}}[c_{x}(x,y)]^{r}d\gamma(x,y)\Big\}^{\frac{1}{r}},
    \end{equation*}
where $\Gamma(\mu,\nu) = \big\{\gamma \in \mathcal{P}({\mathcal{X}\times\mathcal{X}}) : \int_{\mathcal{X}}\gamma(x,y)dy = \mu, \;\int_{\mathcal{X}}\gamma(x,y)dx = \nu\big\}$ is the set of all measure couples between $\mu$ and $\nu$; $r \in [1,\infty)$. 
    
\end{remark}

\begin{definition}[Probability space automorphism]
    Let us denote $X=(\mathcal{X}, \sigma(\mathcal{X}),\mu)$, where $\mu \in \mathcal{P}(\mathcal{X})$. We call $f : X \rightarrow X$ an \textit{automorphism} if it admits a measure preserving, essential inverse $f^{'}$ such that $f \circ f^{'} = f^{'} \circ f = \textrm{id}_{X}$, $\mu$ almost everywhere.
\end{definition}

\subsection{Problem Setup and Background}

Throughout the forthcoming discussion, we denote the input data distribution by $\mu$, and that corresponding to the latent space by $\rho$. In an earlier work, \citet{chakrabarty2021statistical} attest to the theoretical superiority of the constrained formulation of the WAE loss compared to its Lagrangian counterpart:
    \begin{equation} \label{WAE_loss_lag}
        \mathcal{L}_{c_{x},\lambda}(\mu, \rho, D) = \inf_{E \in \mathscr{F}(\mathcal{X},\mathcal{P}(\mathcal{Z}))} \left\{ W_{c_{x}}^{1}(\mu, (D \circ E)_{\#}\mu) + \lambda.\Omega(E_{\#}\mu, \rho)\right\},
    \end{equation}
where $\lambda > 0$. Moreover, to establish consistency of plug-in estimates under the empirical WAE-GAN loss \citep{WAE}, it is sufficient to consider $\Omega(\cdot,\cdot)$ as the total variation (TV) metric \citep{chakrabarty2021statistical}. This is based on the fact that TV acts as an upper bound to the Jensen-Shannon divergence (JS), classically deployed as a regularizer. This modified framework has attracted theoretical intrigue due to its equivalence with the original one under the invertibility of decoders. It is often called the $f$-WAE \citep{husain2019primal}. Building on this density-matching regime, in the current article we also analyze the WAE-MMD architecture, i.e. when $\Omega \equiv d_{\mathcal{H}}$. 

We begin our discussion by focusing on the set of solutions that bring about zero loss. This is crucial since during training, practitioners frequently achieve such near-perfect results. However, the solution maps thus obtained may result in noisy reconstructions. It stems from the fact that WAEs essentially try to solve an `inverse' problem. Our first result suggests that in case the latent space admits nontrivial automorphisms, there may exist non-unique solutions that achieve zero loss. 
    \begin{lemma}[Invariance of zero solutions \citep{Moriakov2020Kernel}] \label{lemm_inv}
        Let $Z=(\mathcal{Z},\sigma(\mathcal{Z}),\rho)$. Also, the encoder-decoder pair $(E,D)$ satisfies $\mathcal{L}_{c_{x},\lambda} = 0$ for a proability divergence $\Omega(\cdot,\cdot)$ that metrize $\mathcal{P}(\mathcal{Z})$. Then, given a non-trivial probability space automorphism $\varphi : Z \rightarrow Z$, the pair $(\varphi^{-1} \circ E, D \circ \varphi)$ also brings about zero loss.
    \end{lemma} 
Observe that one may obtain a zero solution by morphing an existing one. Rooted in the intractability of $\varphi$, this leads to confusing prescriptions for practitioners solving a WAE problem based on neural network-based maps. Moreover, the disentangled representation is sensitive to rotations of the latent embedding \citep{rolinek2019variational}. Thus, one needs to do more than just point out solutions that achieve zero loss. Also, when seen from an OT point of view, the existence and consequently approximation of such non-unique target maps become questionable. We elaborate on the same in Section \ref{encoder}. This brings us to adopting the constrained formulation with heightened conviction:
\begin{equation} \label{const}
    \inf_{E \in \mathscr{F}(\mathcal{X},\mathcal{P}(\mathcal{Z}))} \left\{ W_{c_{x}}^{1}(\mu, (D \circ E)_{\#}\mu)\right\} \enspace \textrm{subject to} \enspace \Omega(E_{\#}\mu, \rho) \leq t,
\end{equation}
where $t \geq 0$ signifies the tolerable error margin. Building on earlier foundation \citep{chakrabarty2021statistical}, we search for realistic model architectures that promote consistency of estimators, a stronger notion compared to non-asymptotic nullity of losses.


\subsection{Data Distributions}
Typically WAE-based architectures are used to deal with image data. The statistical construct we follow favors such cases without being restricted to them only. For example, pixel values of raw image data tend to lie in bounded intervals. As such, considering the support of the probability distribution--- from which they may originate--- to be bounded seems plausible. Moreover, feature-extracted image data attest to this assumption. A key aspect of input distributions that we are interested in particular is their regularity. \citet{chakrabarty2021statistical} tested WAEs' ability to reconstruct H\"{o}lder densities, based on compact supports. We extend our setup to cater to more diverse distributions. Let us describe some classes of functions that characterize the same. 

Let us denote the space of $p$-fold Lebesgue-integrable functions by $L_{p}(\mathbb{R}^d)$, endowed with the norm $\norm{\:\cdot\:}_{p}$, $p \in [1,\infty)$. The uniform norm is denoted by $\norm{\:\cdot\:}_{\infty}$.

\begin{definition}[Sobolev Space]
    Given $\alpha = (\alpha_1, \alpha_2,...\:, \alpha_d)$, $\alpha_i \in \mathbb{N}^{+} \cup \{0\}$, a multi-index such that $|\alpha| = \sum_{i=1}^{d}\alpha_i$ stands for the length, the mixed partial weak differential operator of order $|\alpha|$ is given by $D^{\alpha} = \frac{\partial^{|\alpha|}}{\partial x^{\alpha_1}_{1}...\:\partial x^{\alpha_d}_{d}}$. Here, $x^{\alpha} = x^{\alpha_1}_{1} ...\: x^{\alpha_d}_{d}$ whenever $x \in \mathbb{R}^{d}$. Under this setup, the $L^{p}$-Sobolev Space of order $m$ with radius $L \in \mathbb{R}_{\geq 0}$ is defined as 
    \begin{equation*}
        \mathcal{W}^{m,p}_{L}(\mathbb{R}^d) = \left\{f \in L_{p}(\mathbb{R}^d) : D^{\alpha}f \in L_{p}(\mathbb{R}^d) \: \forall \: |\alpha| \leq m : \norm{f}_{\mathcal{W}^{m,p}} \equiv \: \norm{f}_{p} + \sum_{\abs{\alpha}=m}\norm{D^{\alpha}f}_{p} < L\right\}.
    \end{equation*}
\end{definition}

\begin{remark} \label{remark2}
    In case $f:\mathbb{R}^d \rightarrow \mathbb{R}$ is differentiable at $x$, set $D^{\alpha}f = f^{(\alpha)}$, i.e., the classical mixed partial derivative. Also, denote by $C_{u}(\mathbb{R}^d)$ the space of uniformly continuous functions. This allows us to extend the earlier class for $p=\infty$, namely
\begin{equation*}
    \mathcal{W}^{m, \infty}_{L}(\mathbb{R}^d) = \left\{f \in C_{u}(\mathbb{R}^d) : f^{(\alpha)} \in C_{u}(\mathbb{R}^d) \: \forall \:|\alpha| \leq m : \norm{f}_{\mathcal{W}^{m}} \equiv \: \norm{f}_{\infty} + \sum_{|\alpha|=m}\norm{f^{(\alpha)}}_{\infty} < L\right\}.
\end{equation*}    
\end{remark}
We find the extension of this class for non-integer $s \in \mathbb{R}_{> 0}$, with its integer part $\floor{s}$, particularly helpful in the analysis. The following definition formalizes the same.

\begin{definition}[H\"{o}lder-Zygmund Space]
    Under the setup described in Remark (\ref{remark2}),
    \begin{equation*}
        \mathcal{C}^{s}_{L}(\mathbb{R}^d) = \left\{f \in C_{u}(\mathbb{R}^d) : \norm{f}_{\mathcal{C}^{s}} \equiv \:\norm{f}_{\mathcal{W}^{\floor{s}}} + \sum_{|\alpha| = \floor{s}} \sup_{\substack{{x \neq y} \\ {x,y \in \mathbb{R}^d}}} \frac{|D^{\alpha}f(x) - D^{\alpha}f(y)|}{|x-y|^{s-\floor{s}}} < L \right\}.
    \end{equation*}
\end{definition}
All the above functions can be shown to be inhabitants of Besov spaces with parallel definitions based on wavelets. For further elaboration, one may turn to \citet{gine2021mathematical}, Chapter 4. Now, let us denote the input density corresponding to $\mu$ by $p_{\mu}$, with respect to the Lebesgue measure. 
\begin{assumption}[Regularity of Input Law] \label{assumption1}
    There exists $m_{x} \in \mathbb{N}^{+}$ such that $p_{\mu} \in \mathcal{W}^{m_{x},p}_{L}(\Omega_{x})$, where the support $\Omega_{x} \subseteq \mathcal{X}$ is compact, $p \in [1,\infty)$. 
\end{assumption}
This assumption is put in place to give coherence to the discussion so far. We will focus on the case of unbounded domains under varying regularity as generalizations of the initial results. The more challenging of tasks is perhaps characterizing the latent distribution. In our density matching paradigm, it should ideally be a distribution that embodies the dimensionality-reduced representation of $p_{\mu}$. Let us similarly assume that $\rho$ also has the corresponding density $p_{\rho}$. 

\begin{assumption}  \label{assumption2}
    $p_{\rho}$ is based on a compact and convex support $\Omega_{z} \subseteq \mathcal{Z}$, such that there exists a positive constant $c$ satisfying $\inf_{z}p_{\mu}(z) \geq c$. 
\end{assumption}

The generative aspect of WAEs comes from their capability to simulate novel samples that resemble input observations. The generated set includes smooth interpolations between modal values of $\mu$. As such, the latent law--- encapsulating the input information into local geometric signatures--- must distribute positive mass between encoded modes. \citet{WAE} demonstrate the same fact with facial image data. This stems from the idea that the meld between two faces in the Wasserstein geodesic might result in another one, even if `unrealistic'. The convexity of the support of $p_{\rho}$, coupled with its departure from nullity, is a mathematical representation of the same philosophy. \citet{asatryan2023convenient} argue that an explicit lower bound to the density can always be found, for a slightly modified measure (Remark 3.3). As such, we assume $\rho$ to have a \textit{strong density}. Also, to conform to disentanglement, $\rho$ should ideally have a diagonal or block-diagonal dispersion matrix. In our non-parametric depiction, we keep ample room for such specifications.  

\section{Latent Space Consistency}

With the foundations laid, let us concentrate on the encoding. In an empirical WAE problem, we only have access to a set of samples $\{X_i\}^{n}_{i=1}$ drawn i.i.d. from $\mu$. Thus, the sample version of the optimization task (\ref{const}) needs to satisfy the corresponding constraint: $\Omega(E_{\#}\hat{\mu}_{n}, \rho) \leq t$, given $t \geq 0$. Here, $\hat{\mu}_{n} = \frac{1}{n}\sum^{n}_{i=1}\delta_{X_i}$ is the classical plug-in estimate. We use the same notation to signify the empirical distribution throughout the article. Observe that, the resultant set of encoder transforms that fulfill this criterion are indeed functions of $n$, i.e. $E = E(n)$. In the absence of uniqueness, our suggestions of a \textit{capable} encoder begin with a definition of its chassis: neural networks. 

\begin{definition}[Feed-Forward Neural Networks] \label{NN}
    Given $L \in \mathbb{N}^{+}$, a Neural Network (NN) with $L$ hidden layers is defined as the collection of maps $\phi : \mathbb{R}^{N_0} \longrightarrow \mathbb{R}^{N_{L+1}}$, $\{N_i\}^{L+1}_{i=0} \in \mathbb{N}^{+}$ given by 
    \begin{equation*}
        \phi(x) := A_L \circ \sigma \circ A_{L-1} \circ ... \circ \sigma \circ A_{0}(x),
    \end{equation*}
    where $A_{i}(y) = M_{i}y + b_i$; $M_i \in \mathbb{R}^{N_{i+1} \times N_{i}}$ and $b_i \in \mathbb{R}^{N_{i+1}}$ is an affinity, $i = 0,...,L$. Here, $\sigma$ signifies the activation function, applied componentwise. Under this setup, we call $W = \vee^{L}_{i=1} N_i$ the \textit{width} of the network and $L$ its \textit{depth}. Denote this collection by $\Phi(W, L)^{N_{L+1}}_{N_0}$. Additionally, the quantity $S = \sum^{L}_{i=1} N_i$ is said to be the \textit{size}. A reparameterized version of the same, given as $N(\Phi) = d + S$, denotes the \textit{number of neurons} in the network. 
\end{definition}

\begin{remark}
    Based on its simplicity and resilience to vanishing gradients, ReLU ($\sigma(x) = x \vee 0$, also called \textit{ramp} or \textit{first order spline}) has emerged as the most desired activation to practitioners. However, what we find intriguing is the remarkable capability of ReLU-based NNs to approximate regular functions \citep{yarotsky2017error, yarotsky2018optimal, petersen2018optimal}. Another activation function that is critical to our analysis is GroupSort \citep{anil2019sorting}. Preserving all the goodness offered by ReLU, it additionally provides adversarial robustness \citep{huster2019limitations}. We also stress the fact that GroupSort (equivalently OPLU, when grouping size is 2 \citep{chernodub2016norm}) NNs are better suited at universally approximating non-linear Lipschitz maps.
\end{remark}

\subsection{Encoder maps} \label{encoder}

Encoders are transformations that can be said to enforce dimensionality reduction preserving key properties of $\mu$. Though not obvious, typically, such maps enforce a non-linear reduction due to the non-linearities (activations, e.g. tanh) present in the underlying NN. The process it undergoes is significantly different from classically known DR techniques. However, in case the maps are assumed to be linear embeddings (decoder as well), latent factors obtained by a VAE tend to align along the Principal Component (PC) directions \citep{rolinek2019variational}. Regularised VAEs can also be related to the DR carried out by non-linear ICA \citep{hyvarinen1999nonlinear} under a parametric regime. The similarity stems from the achievement of identifiability of the parameters characterizing the latent law in both methods \citep{khemakhem2020variational}. This departure from traditional techniques forces us to change our viewpoint on DR as we know it. Besides, the encoding in the posterior density-matching setup of WAEs differs even further. Instead of looking at the encoder's capacity to conserve local and broader geometry of the spaces in terms of distances, we focus on its trait to limit \textit{distortions} of estimators. Let us provide a probabilistic definition of the same.

\begin{definition}[Information Preserving Transform \citep{chakrabarty2021statistical}]
    Given an estimate $\hat{\mu}_{n}$ based on $n$ observations from the distribution $\mu$ and $\epsilon > 0$, a map $I \in \mathscr{F}\left(\mathcal{X}, \mathcal{P}(\mathcal{Z})\right)$ is said to be Information Preserving of degree $r$ under the distance metric $d(\cdot,\cdot)$ if there exist constants $k_1,k_2 > 0$, such that 
    \begin{equation*}
        \mathbb{P}\left(d\left(I_{\#}\hat{\mu}_{n}, \widehat{(I_{\#}\mu)}_m\right) \leq \epsilon\right) \geq 1-k_1\exp{\left\{-k_{2}(n \wedge m)^{r}\epsilon^{2}\right\}}.
    \end{equation*}
    Here, $\widehat{(I_{\#}\mu)}_m$ denotes an estimate of the translated distribution based on $m \in \mathbb{N}^{+}$ i.i.d. samples and $r \geq 1$.
\end{definition}

The immediate question coming to mind may be what are the transformations that behave as IPTs. Precisely, the answer lies in the regularity of the functions. Though not apparent, the notion of IPTs is intrinsically related to Bourgain's discretization theorem and Lipschitz embeddings. To that end, we first explore the capability of Lipschitz continuous maps--- between the input and latent spaces--- to pose as IPTs. Let us denote by $\mathscr{F}_{L}(\mathcal{X}, \mathcal{Z})$ the class of $L$-Lipschitz functions mapping $(\mathcal{X}, c_{x})$ to $(\mathcal{Z}, c_{z})$. So far we have not imposed any regularity on the class of latent distributions. In such a general setting, the role of the underlying divergence metrizing the measure space becomes crucial. In this context, we recall the caution sounded by \citet{sriperumbudur2009integral} that the total variation metric ($d_{\mathcal{F}} \equiv d_{\textrm{TV}}$, where $\mathcal{F} = \left\{f : \norm{f}_{\infty} \leq 1\right\}$) is often unable to ensure strong consistency of estimators under them. The issue is rooted in the class of critics $\mathcal{F}$ being `too large'. The \textit{first} method to circumvent this problem is to look at IPMs employing more precise critics. 

\begin{theorem}[Information Preservation of Lipschitz Encoders] \label{theo: lip}
    Let $\mathcal{H}$ denote a class of bounded real-valued functions on $\mathcal{Z}$, such that the associated entropy has at most polynomial discrimination. In other words, there exists $q, A \in \mathbb{R}_{>0}$ such that $\forall \epsilon > 0, \;\log \mathcal{N}\left(\mathcal{H}, \norm{\cdot}_{\infty}, \epsilon\right) \leq A \epsilon^{-q}$. Then for any $g \in \mathscr{F}_{L}(\mathcal{X}, \mathcal{Z})$ there exists constants $l, E_{1}, E_{2}$ and $E_{3} >0$ such that given $0<t \leq \frac{2}{3}$,
    \begin{equation*}
      d_{\mathcal{H}}\left(g_{\#}\Tilde{\mu}_{n}, \widehat{(g_{\#}\mu)}_m\right) \leq t + \frac{l L h^{m_{x}}}{2} + \mathcal{O}(m^{-\frac{1}{q \vee 2}})    
    \end{equation*}
    holds with probability $ \geq 1 - \left(E_{1} + E_{2}(\frac{\sqrt{d}L B_{x}}{h^{d+1}t})^{d}\right) \exp\left\{-E_{3}(m \wedge nh^{d}) t^{2} \right\}$, where $\Tilde{\mu}_{n}$ is a density estimate of $\mu$ based on the Regularly invariant kernel $\kappa$, satisfying $\sup_{\kappa} \sup_{x \in \Omega_{x}} \kappa(\cdot, \cdot) \leq 1$, and with bandwidth $h \equiv h_{n} \searrow 0$.
\end{theorem}

The theorem implies that Lipschitz transforms can approximately pose as IPTs of order $1$. By choosing $h$ judiciously, one may show that the approximation error turns $o(1)$ in the asymptotic regime. We deliberately use the smoother kernel density estimate instead of the plug-in to appreciate Assumption \ref{assumption1}. The choice of the kernel function--- as regularly invariant--- is of high significance, which the proof (see Appendix \ref{appA}) demonstrates. We elaborate on the same while discussing MMDs (Definition \ref{Reg_inv_ker}). Now, the classes $\mathcal{H}$ whose entropy increase polynomially lie in abundance \citep{nickl2007bracketing}. A particular case that we emphasize on is $\mathcal{L}_{c_{z}}^{1}$, i.e. $1$-Lipschitz functions with respect to $c_{z}$. It is known that $\log\mathcal{N}\left(\mathcal{L}_{c_{z}}^{1}, \norm{\cdot}_{\infty}, \epsilon\right) \lesssim \lambda(\mathcal{Z}^{1}) \epsilon^{-k}$, where $\lambda(\mathcal{Z}^{1})$ is the Lebesgue measure of the set $\left\{z : c_{z}(z, \mathcal{Z}) < 1\right\}$ (\citet{van1996weak}, Theorem 2.7.1). The choice of critics as Lipschitz also provides a generalization over most Besov functions.

\begin{corollary} \label{cor:lip1}
    Given $g \in \mathscr{F}_{L}(\mathcal{X}, \mathcal{Z})$ and the empirical distribution $\hat{\mu}_{n}$, there exists a positive constant $E_{1}^{'}$, such that
    \begin{equation*}
        d_{\mathcal{L}_{c_{z}}^{1}}\left(g_{\#}\hat{\mu}_{n}, \widehat{(g_{\#}\mu)}_m\right) \leq t + \mathcal{O}(m^{-\frac{1}{k \vee 2}}) + \mathcal{O}(n^{-\frac{1}{d}})
    \end{equation*}
    holds with probability at least $1 - \exp\left\{-E_{1}^{'} (n \wedge m) t^{2}\right\}$.
\end{corollary}

\begin{remark}[Extension for $b$-Lipschitz critics]
    A similar result to that of the previous theorem in case of the divergence $d_{\mathcal{L}_{c_{z}}^{b}}(\cdot,\cdot)$ can be established based on the fact that $\log\mathcal{N}\left(\mathcal{L}_{c_{z}}^{b}, \norm{\cdot}_{\infty}, \epsilon\right) \leq \mathcal{N}\left(\mathcal{Z}, c_{z}, \frac{\epsilon}{8b}\right)\log\big(\frac{8}{\epsilon}\big)$ (\citet{gottlieb2017efficient}, Lemma 6), given $b>0$. Observe that, it also enables one to remove the \textit{boundedness} of the support, latent class of measures lie on. Instead, we may impose milder restrictions such as having sub-exponential tails (essentially bounded). Specifically, if $\mathbb{E}_{p}\{\norm{Z}\mathbb{1}_{\{\norm{Z}>\log(m)\}}\}=\mathcal{O}(m^{-\frac{{(\log{m})^{\delta}}}{k}})$, where $p \in \mathcal{P}(\mathcal{Z})$ and $\delta>0$; we may recover Corollary (\ref{cor:lip1}) with only an altered approximation error $\mathcal{O}(m^{-\frac{1}{k}}(\log m)^{1+\frac{1}{k}})$.
\end{remark}
Now let us focus on the \textit{second} remedy, that being more regulated classes of translated laws. This is crucial since otherwise the convergence of empirical measures under TV might become arbitrarily slow \citep{devroye1990no}. To that end, let us first recall the notion of Yatracos classes (YC) (\citet{devroye2001combinatorial}, Chapter 6). Given $\mathcal{F} : \mathcal{Z} \rightarrow \mathbb{R}$, the Yatracos class associated to it is the set system $\{z\in\mathcal{Z} : f(z) \geq g(z); \;f,g \in \mathcal{F}\}$, denoted by $\mathcal{Y}(\mathcal{F})$. In other words, it characterizes the domain in terms of candidates in a Scheff\'{e} tournament. Our next result suggests that if the Vapnik-Chervonenkis (VC) dimension of the YC corresponding to the family of latent distributions is finite, we can recover Theorem (\ref{theo: lip}). The proposition becomes quite intuitive following the maximal packing argument of \citet{van2014probability}, Theorem 7.16.

\begin{corollary} \label{cor:vc_IPT}
    Let the VC-dim$[\mathcal{Y}(\mathcal{P}(\mathcal{Z}))] = v_{z} < \infty$. Then for any $g \in \mathscr{F}_{L}(\mathcal{X}, \mathcal{Z})$ and $0<t \leq \frac{2}{3}$ the constants $E_{1}, E_{2}$ and $E_{3} >0$ are retained such that 
    \begin{equation*}
        d_{\textrm{TV}}\left(g_{\#}\Tilde{\mu}_{n}, \widehat{(g_{\#}\mu)}_m\right) \leq t + \frac{l L h^{m_{x}}}{2} + \mathcal{O}(\sqrt{v_{z}} m^{-\frac{1}{2}})
    \end{equation*}
    holds with probability $ \geq 1 - \left(E_{1} + E_{2}(\frac{\sqrt{d}L B_{x}}{h^{d+1}t})^{d}\right) \exp\left\{-E_{3}(m \wedge nh^{d}) t^{2} \right\}$, where $\Tilde{\mu}_{n}$ is a Regularly Invariant kernel (RIK) density estimate of $\mu$, defined similarly as in Theorem \ref{theo: lip}. 
\end{corollary}

There are two key highlights of the latest result that turn out to be indispensable in the upcoming discussion on latent consistency. The first aspect we emphasize is the tail condition of the target law. Sub-exponential is a fairly general notion in the sense that all bounded and sub-Gaussian distributions fall under its umbrella. Moreover, all results obtained under such a characterization can be directly extended to sub-Weibull distributions \citep{vladimirova2020sub}. In practice, WAEs are mostly trained with multivariate Gaussian as conjugate prior (and hence, posterior) latent laws \citep{WAE}, which also conforms to our argument. The second facet--- arguably the cornerstone of the analysis by \citet{chakrabarty2021statistical}, and responsible for controlling the complexity of the underlying space--- is the quantity VC-dim$[\mathcal{Y}(\cdot)]$. The finiteness assumption on the same is frequented in density estimation \citep{ashtiani2018sample, jain2020general} solely based on its plausibility. It is known that the class of $k$-dimensional Gaussian distributions have VC-dim$[\mathcal{Y}(\cdot)] = \mathcal{O}(k^2)$. The same corresponding to axis-aligned densities hailing from $k$-variate exponential families turn out to be $\mathcal{O}(k)$ (\citet{devroye2001combinatorial}, Theorem 8.1).  

Before moving on to further examples of IPTs, let us examine the worth of NN-based maps in the same context. Observe that, the transformations $A_{i}(\cdot)$ (see Definition \ref{NN}) can be easily shown to be Lipschitz continuous by limiting $\norm{M_i}_{p} = \sup_{\:\norm{y}_{p} = 1}\norm{M_{i}y}_{p} \leq t$, given $t>0$. \citet{anil2019sorting} gave simple recipes to preserve such norms in case $p=2$ and $\infty$. This fact, coupled with the Lipschitz continuity of activation functions (e.g. ReLU, leaky ReLU, GroupSort, tanh, sigmoid, etc.) typically applied, it is not difficult to imagine NN-transforms to be exactly so. However, not all such $\sigma(\cdot)$ preserve gradient norms under composition without additional regularization (e.g. ReLU). Furthermore, recovering the exact Lipschitz constant, and hence the map, often turns out to be NP-hard \citep{virmaux2018lipschitz}. So instead, we harness the approximation capabilities of deep NNs to our aid. ReLU has attracted the most attention in this regard \citep{suzuki2018adaptivity, chen2019efficient, montanelli2019new, daubechies2022nonlinear, gribonval2022approximation}. To motivate our next result, we present a simple observation:

\begin{lemma} \label{lemm_bound}
    Given $\mu_{1}, \mu_{2} \in \mathcal{P}(\mathcal{X})$ and $\phi \in \Phi(W, L)^{k}_{d}$, under arbitrary IPM $d_{\mathcal{F}}(\cdot, \cdot)$, such that $\mathcal{F} = \left\{f : \mathcal{Z} \rightarrow \mathbb{R}\right\}$ is symmetric, we obtain 
    \begin{equation*}
        d_{\mathcal{F}}(\phi_{\#}\mu_{1}, \phi_{\#}\mu_{2}) \leq 2\left[\inf_{g \in \mathscr{F}(\mathcal{X}, \mathcal{P}(\mathcal{Z}))}\norm{\phi - g}_{\infty} + \mathcal{E}(\mathcal{F}, \mathcal{L}_{c_{z}}^{1})\right] + d_{\mathcal{L}_{c_{z}}^{1}}(g_{\#}\mu_{1}, g_{\#}\mu_{2}),
    \end{equation*}
    where $\mathcal{E}(\mathcal{F}, \mathcal{L}_{c_{z}}^{1}) = \sup_{f \in \mathcal{F}} \inf_{l \in \mathcal{L}_{c_{z}}^{1}} \norm{f-l}_{\infty}$ denotes the essential disagreement between classes of critics and $c_{z} \equiv L_{1}$.
\end{lemma}

Observe that, the statement holds true under arbitrary choices of the second class of evaluation functions. We mention $\mathcal{L}_{c_{z}}^{1}$, in particular, to continue our discussion in the light of Corollary \ref{cor:lip1}. The result suggests that it is sufficient for a feed-forward NN-induced function to approximate Lipschitz maps (between the input and latent space) to behave like an IPT, approximately. Under the TV metric, the proof becomes much simpler based on the fact that $d_{\textrm{TV}}(\phi_{\#}\mu_{1}, \phi_{\#}\mu_{2}) \leq d_{\textrm{TV}}(\mu_{1}, \mu_{2})$, for $\phi \in \mathscr{F}(\mathcal{X},\mathcal{P}(\mathcal{Z}))$. However, there may arise difficulties in case the underlying distance measure is MMD. Let us first go through some rudiments of kernel methods to facilitate our investigation on the same. 

\subsection*{MMD and Kernel Mean Embedding}

The well-known Riesz representation theorem ensures the existence of a unique representer $K(x) \in \mathcal{H}$, such that $\forall f \in \mathcal{H}, f(x) = \langle f, K(x) \rangle$ for all $x \in \mathcal{X}$. In this setup, the function $\kappa(x,y) = \langle  K(x) , K(y) \rangle$ is called the \textit{reproducing kernel} of $\mathcal{H}$. The opposite characterization also holds true. By Aronszajn’s theorem, given a symmetric, positive definite $\kappa$ on $\mathcal{X} \times \mathcal{X}$, there exists a unique RKHS $\mathcal{H}_{\kappa}$. This inspires us to meaningfully narrow down our focus on the distributions $\mathcal{P}_{\kappa} = \{\pi \in \mathcal{P} : \int \sqrt{\kappa(x,x)} \abs{\pi} (dx) < \infty\}$. The MMD between two of such laws $\mu_{1}, \mu_{2}$ can be rewritten as $d_{\mathcal{H}_{\kappa}}(\mu_{1}, \mu_{2}) = \norm{K(\mu_{1}) - K(\mu_{2})}_{\mathcal{H}_{\kappa}}$, i.e. the Hilbert space distance between the respective kernel mean embeddings (KME), given by $K(\pi)(x) = \int \kappa(x,y) \pi (dy)$. For a detailed exposition on the same, one may seek refuge to \citet{sriperumbudur2010hilbert}. Since it is the kernel function that determines the nature of the RKHS, we introduce some regularities which in turn aid our cause.

\begin{definition}[Regularly Invariant Kernels] \label{Reg_inv_ker}
    A measurable function $\kappa(x,y): \mathcal{X} \times \mathcal{X} \rightarrow \mathbb{R}$ is said to be \textit{regular}, if for $N \in \mathbb{N}$ it satisfies
    \begin{enumerate}[label=(\roman*),itemsep=0.2cm]
        \item $\int_{\mathcal{X}} \sup_{v \in \mathcal{X}} \abs{\kappa(v,v-u)}{\abs{u}}^N du < \infty$, and
        \item $\int_{\mathcal{X}}\kappa(v,v+u) du = 1; \int_{\mathcal{X}}\kappa(v,v+u)u^{\alpha} du = 0$ for every $v \in \mathcal{X}$ and $\abs{\alpha} = 1,...,N-1$.
    \end{enumerate}
    If such a kernel additionally satisfies the weaker \textit{invariance} property:\\ $\left\{\int\abs{\kappa(w,v)-\kappa(w,u)}^{r}dw\right\}^{\frac{1}{r}} = \mathcal{O}\left(\norm{v-u}\right)$, given $r \geq 1$; we say it is \textit{regularly invariant}.
\end{definition}

Observe that, most kernel functions based on standard probability distributions with finite and centered moments will tend to be regular. Moreover, an immediate example of our version of invariant would be Energy kernels: $\kappa_{\alpha}(u,v) = \norm{u}^{2\alpha} + \norm{v}^{2\alpha} - \norm{u - v}^{2\alpha}$, $u,v \in \mathcal{X}$ and $\alpha \in (0,1)$ \citep{modeste2022characterization}. For further coherence, we introduce the notion of \textit{strong invariance}. A kernel function is called strongly invariant if the following holds: $\norm{K(u) - K(v)} = \mathcal{O}\left(\norm{u-v}\right)$. This in turn implies weak invariance based on the fact that
\begin{align*}
    \left\{\int\abs{\kappa(w,v)-\kappa(w,u)}^{r}dw\right\}^{\frac{1}{r}} \leq \norm{K(v) - K(u)} \left[\int\norm{K(w)}^{r}dw\right]^{\frac{1}{r}}.
\end{align*}
For example, in case of Energy kernels, $\norm{K(u) - K(v)} = 2\norm{u-v}^{\alpha}$. Based on such functions, MMD indeed promotes exponential decay in information dissipated while encoding. To set the stage for a supporting mathematical argument, let us first notice that given $\mu \in \mathcal{P}_{\kappa}(\mathcal{X})$,
\begin{equation*}
    d^{2}_{\mathcal{H}_{\kappa}}(\hat{\mu}_{n}, \mu) = \langle  K(\hat{\mu}_{n}-\mu) , K(\hat{\mu}_{n}-\mu) \rangle_{\mathcal{H}_{\kappa}} = \int_{\mathcal{X} \times \mathcal{X}} \kappa(x,y) (\hat{\mu}_{n}-\mu) \otimes (\hat{\mu}_{n}-\mu) (dxdy).
\end{equation*}
Also, $\mathbb{E}\left[d_{\mathcal{H}_{\kappa}}(\hat{\mu}_{n}, \mu)\right] \leq \left[\mathbb{E} \;d^{2}_{\mathcal{H}_{\kappa}}(\hat{\mu}_{n}, \mu)\right]^{\frac{1}{2}} \leq \sqrt{\frac{2}{n}}\sup_{x \in \Omega_{x}}\kappa(x,x)^{\frac{1}{2}}$.

\begin{theorem}[Information preservation under MMD] \label{MMD_IPT}
    Let $\mu \in \mathcal{P}_{\kappa}(\mathcal{X})$, where $\kappa(\cdot,\cdot)$ is a strongly invariant kernel satisfying $\sup_{z \in \Omega_{z}}\kappa(z,z) \leq C_{\kappa}$, such that $C_{\kappa}>0$. Given $g \in \mathscr{F}_{L}(\mathcal{X}, \mathcal{Z})$, there exists $\phi \in \Phi(W, L)^{k}_{d}$ which implies that
    \begin{align*}
        d_{\mathcal{H}_{\kappa}}&\left(\phi_{\#}\hat{\mu}_{n} , \widehat{(\phi_{\#}\mu)}_m\right) \leq (m \wedge n)^{-\frac{1}{2}}\sqrt{\frac{B\ln{(\frac{2}{\delta})}}{2}} + \sqrt{\frac{2C_{\kappa}}{m}} \\ & \qquad + \sqrt{D_{n}}\underbrace{\norm{\phi - g}_{\infty}^{\frac{1}{2}}}_\text{$(\ast)$} + \sqrt{\mathcal{O}(c_{g,n}{(d^{2}n)}^{-\frac{1}{d}}) + L (m \wedge n)^{-\frac{1}{2}} c_{g,n} \sqrt{\frac{B\ln{(\frac{2}{\delta})}}{2}}}
    \end{align*}
    holds with probability at least $1-\delta$, $\delta > 0$. Here, $B$ is a positive constant dependent on $C_{\kappa}$, and both $D_{n}$ and $c_{g,n}$ are sequences based on $n$ that $\searrow 0$ almost surely as $n \rightarrow \infty$.
\end{theorem}

This result in turn enables to showcase the information-preserving capability of NNs deploying several activation functions. Observe that, the quantity $(\ast)$ in Theorem \ref{MMD_IPT} acts as an upper bound to the departure of an NN-induced map from its exemplar $g$. The following corollaries look for the sharp values of $(\ast)$ under different circumstances.

\begin{corollary} 
    \begin{enumerate}[label=(\roman*),itemsep=0.2cm] 
        \item \textup{(Information Preservation of ReLU Encoders \citep{shen2019deep})} \\ There exists $\phi \in \Phi(W, L)^{k}_{d}$ based on ReLU activations with $W = \mathcal{O}(d\floor{N_{1}^{\frac{1}{d}}} \vee N_{1}+1)$ and $L = \mathcal{O}(N_{2})$, that satisfy Theorem \ref{MMD_IPT} for any $N_{1},N_{2} \in \mathbb{N}^{+}$, such that $(\ast) = \mathcal{O}(\sqrt{d}LB_{x}N_{1}^{-\frac{2}{d}}N_{2}^{-\frac{2}{d}})$. Here, $B_{x} :=$ diameter of $\Omega_{x}$ with respect to the metric $c_{x}$. 
        \item \textup{(Information Preservation of GroupSort Encoders \citep{tanielian2021approximating})} \\ Given $\varepsilon > 0$, there exists a GroupSort NN induced map $\phi \in \Phi(W, L)^{k}_{d}$ of depth $L = \mathcal{O}(d^{2}\log_{2}(\frac{2\sqrt{d}}{\varepsilon}))$ and size $S = (\frac{2\sqrt{d}}{\varepsilon})^{d^{2}}$, also satisfying $\norm{M_{0}}_{2,\infty} = \sup_{\:\norm{x}=1}\norm{M_{0}x}_{\infty} \leq 1$, $\max\{\:\norm{M_{i}}_{\infty};i=1,\cdots,L\} \leq 1$ and $\max\{\:\norm{b_{j}}_{\infty}; j=0,\cdots,L\} \leq \infty$, such that $(\ast) = \mathcal{O}(\varepsilon)$.
    \end{enumerate} \label{IPT_ReLU_Group}
\end{corollary} 

\begin{remark}[Barron Functions as IPT]
    Based on our previous discussion it becomes evident that being Lipschitz continuous is a desirable property for encoder transforms to behave as IPTs, approximately at the least. This very fact accentuates the importance of the class of functions known as \textit{Barron functions}. While there exist numerous characterizations of the same \citep{wojtowytsch2022representation}, we keep to the following definition.  
    \begin{definition} [\citet{caragea2020neural}] \label{def: Barron}
        A function $f: \Omega_{x}(\subset \mathbb{R}^{d}) \rightarrow \mathbb{R}$ is said to belong to Barron class with constant $C>0$ (say, $\mathcal{B}_{C}(\Omega_{x})$), if the exists $x^{'} \in \Omega_{x}$ and a measurable function $g: \mathbb{R}^{d} \rightarrow \mathbb{C}$ such that $\forall x \in \Omega_{x}$ both the conditions
        \[\int_{\mathbb{R}^{d}} \sup_{x \in \Omega_{x}} \abs{\langle \eta, x-x^{'}\rangle} \abs{g(\eta)} d\eta \leq C \quad \textrm{and} \quad f(x) - c = \int_{\mathbb{R}^{d}} (e^{i\langle x,\eta\rangle} - e^{i\langle x^{'},\eta\rangle}) g(\eta) d\eta\]
        are satisfied, where $\abs{c} \leq C$.
    \end{definition}
    For vector-valued functions $f: \Omega_{x} \rightarrow \mathbb{R}^{k}$, which we are mostly interested in, the same criteria need to be satisfied componentwise. \citet{lee2017ability} provide an equivalent definition (also based on Fourier inversion) of Barron class as well. It is intriguing to observe that $\mathcal{B}_{C}$ embeds continuously into the class of real-valued Lipschitz maps, under the $L_{1}$ metric (\citet{wojtowytsch2022representation}, Theorem 3.3). As such, Barron functions are naturally prone to preserving information while serving as encoders (Theorem \ref{theo: lip}). Now, observe that the real architectures suggested in the context of IPT so far might suffer from the curse of dimensionality. Also, they both tend to be deep, scaling exponentially with the input data dimension. While this serves our purpose in the asymptotic regime, shallow networks ($L=1$) might be of greater priority to practitioners. Meanwhile, the Barron class can be shown to accommodate all finite norm-bounded neural networks and their limits \citep{wojtowytsch2022representation}. This is crucial since it allows one to demonstrate shallow networks' capability to act as an IPT approximately. \citet{barron1993universal}, in his seminal paper, first proved that a function $f \in \mathcal{B}_{C}$ can be approximated up to arbitrary accuracy by a shallow NN deploying sigmoidal activations\footnote{A bounded measurable function $h: \mathbb{R} \rightarrow \mathbb{R}$ is said to be sigmoidal if $\lim_{x \rightarrow -\infty}h(x) = 0$ and $\lim_{x \rightarrow \infty}h(x) = 1$. Common examples of such activation functions include logistic, hyperbolic tangent, and $h(x) = \textrm{ReLU}(x) - \textrm{ReLU}(x-1)$.}. In other words, there exists $\phi \in \Phi(W, 1)^{1}_{d}$ with $N(\Phi) = m$, such that $\{\int_{\Omega_{x}}(f(x)-\phi(x))^{2} \mu (dx)\}^{\frac{1}{2}} = \norm{f - \phi} \lesssim m^{-\frac{1}{2}}$. In the asymptotic regime, however, to achieve an infinitesimally small error, the map $\phi$ approaches an infinitely wide NN. \citet{wojtowytsch2022representation} draw the same conclusion based on input distributions $\mu$ having finite second moment (Theorem 3.8). This approximation can be further improved in case $\mu$ has bounded support, conforming to Assumption \ref{assumption1}. \citet{klusowski2018approximation} show the existence of such a $\phi$, based on general Lipschitz activations, that ensures $\norm{f-\phi}_{\infty} \lesssim \sqrt{d + \log{m}}\:m^{-\frac{1}{2}-\frac{1}{d}}$. Despite being the highlight, our discussion on the approximation of Barron functions is not limited to shallow networks. The information-preserving behavior of Barron maps stays intact under compositions as well. This is again rooted in them being Lipschitz continuous. As an immediate consequence, we find an alternative avenue to show that sigmoid-activated deep encoders too act as IPTs.  

    \begin{corollary}[Information Preservation of Sigmoidal Encoders \citep{lee2017ability}] \label{IPT_sig}
        Let $\{N_i\}^{L+1}_{i=0} \in \mathbb{N}^{+}$ be a sequence of intermediate dimensions as given in Definition \ref{NN}, where $N_{0} = d$ and $N_{L+1} = k$. Also let $f_{i} : \mathbb{R}^{N_{i-1}} \rightarrow \mathbb{R}^{N_{i}}$ such that $f_{1} \in \mathcal{B}_{C_{0}}(\Omega_{0})$ and for $2\leq i \leq L+1$ and a given parameter $s>0$, $f_{i} \in \mathcal{B}_{C_{i-1}}(\Omega^{s,N_{i-1}}_{i-1})$. Here $\{C_i\}^{L}_{i=0} > 0$ and $\Omega^{s,N_{i-1}}_{i-1} := \{y = y_{1} + y_{2} : y_{1} \in \Omega_{i-1}, y_{2} \in B^{N_{i-1}}(s)\}$, $B^{N_{i-1}}(s)$ being the $L_{2}$ ball of radius $s$ in $\mathbb{R}^{N_{i-1}}$, that satisfy $f_{i}(\Omega_{i-1}) \subseteq \Omega_{i}$, $1\leq i \leq L+1$. In particular, define $\Omega_{0} \equiv \Omega_{x}$ and $\Omega_{L+1} \equiv \Omega_{z}$. Under this setup, given $f_{1:L+1} := f_{1} \circ f_{2} \circ \cdots \circ f_{L+1}$ and $\varepsilon > 0$, there exists an $L$-deep sigmoid-activated $\phi$, with $\mathcal{O}(N_{i}\varepsilon^{-2})$ nodes in the $i$th layer, such that $(\ast) = \mathcal{O}(\varepsilon)$.
    \end{corollary}
\end{remark}

With the desired regularity of an ideal encoder specified, we move closer to testing whether the constraint, as in (\ref{const}), is met. We reiterate that the problem at hand is rather the sample equivalent of the same. While our probabilistic notion based on estimators provides a view into the solution, encoding might also be seen in the light of local geometry. This not only provides further clarity to the definition of representation but also opens up a pathway to understanding decoding in greater detail. Observe that, the process of encoding seeks to learn embeddings onto the latent space. A meaningful characterization of the same might be a map that aims to preserve pairwise distances between discrete points (the sample set) residing in the input space \citep{courty2018learning}. It also turns out to be the cornerstone of the techniques relying on the latent space to form clusters \citep{jiang2017clustering}. Such a map $E : \Omega_{x} \rightarrow \Omega_{z}$ satisfying $ac_{x}(x,y) \leq c_{z}(E(x), E(y)) \leq Ac_{x}(x,y)$ is said to be bi-Lipschitz (BL) with distortion $D(E) = \frac{A}{a} < \infty$. These immediately fit the bill of IPTs. Moreover, the inverse of such embeddings turns out to be Lipschitz as well. We will later see that this fact can be exploited to aid in efficient decoding. Another feature that stands out is that $E$ restricts the encoded law from being degenerated at a point. Now, the immediate question that arises is whether such embeddings exist. The first affirmative evidence was presented by \citet{JL1984}, taking both $c_{x}$ and $c_{z}$ to be $L_{2}$ in their respective spaces. 

\begin{lemma}[Johnson-Lindenstrauss Embedding] \label{JL}
    Given a set of size $n$ from $\Omega_{x}$ and $0<\varepsilon < \frac{1}{2}$, there exists a Bi-Lipschitz map $E: \mathbb{R}^{d} \rightarrow \mathbb{R}^{k}$ with distortion $(1+\varepsilon)$, such that $k = \mathcal{O}(\frac{\log(n)}{\varepsilon^{2}})$.
\end{lemma}

This result additionally ties the extent of distortion to the latent dimension, shedding new light on our previous discussion on the optimal value of $k$. In fact, the bound in the lemma turns out to be optimal up to constant factors \citep{larsen2017optimality}. Later, \citet{bourgain1985lipschitz} also showed the existence of BL transforms that achieve encoding onto $k = \mathcal{O}(\log^{2}n)$ with distortion $\lesssim \log(n)$. For now, the existence of Lipschitz encoders, acting as benchmarks to NN-based maps, will be sufficient. To that end, we turn to the following result.

\begin{lemma}[\citet{bartal2011dimensionality}] \label{bartal}
    Given $X \subseteq \Omega_{x}$, for every $\varepsilon > 0$, there exists a finite $1$-Lipschitz embedding $E: X \rightarrow \mathbb{R}^{k}$ such that $k = \mathcal{O}\left(\varepsilon^{-2} d^{*}(X)(\frac{\log(d^{*}(X))}{\varepsilon})\right)$, where $d^{*}(X)$ is the doubling dimension\footnote{The doubling dimension is defined as $d^{*}(X) = \log_{2}\lambda$, where $\lambda \geq 1$ (doubling constant) is the smallest number such that at most $\lambda$ balls of half radius are needed to cover every ball in $X$.} of $X$.  
\end{lemma}

This is an exact deterministic answer to our search for an ideal encoder. Such a map can immediately be extended to the whole input space using Kirzbraun's theorem. To show latent space consistency, first observe that given a metric $\Omega$ and encoder $E$, the realized latent loss turns out to be $\Omega(E_{\#}\hat{\mu}_{n}, \rho)$. It can be fragmented into the following parts based on the independent sources of variation: 
\begin{equation}\label{encode}
    \Omega(E_{\#}\hat{\mu}_{n}, \rho) \leq \underbrace{\Omega(E_{\#}\hat{\mu}_{n}, \widehat{(E_{\#}\mu)}_m)}_\text{Information dissipated} + \underbrace{\Omega(\widehat{(E_{\#}\mu)}_m, \rho)}_\text{Estimation error}.
\end{equation}
While a suitably chosen IPT takes care of the first part, the latter embodies the error committed trying to estimate $\rho$ using samples from the encoded law. In a \textit{lossless encoding} ($E_{\#}\mu \:{=}_{d} \:\rho$), it will boil down to the usual estimation error. Otherwise, one might be left with a surplus error due to the discrepancy between $\widehat{(E_{\#}\mu)}_m$ and a certain $\hat{\rho}_{m}$. \citet{chakrabarty2021statistical} give a non-asymptotic upper bound to the same based on the empirical Yatracos minimizer of $\hat{\rho}_{m}$ (Theorem 1). In the process, they assume the finiteness of VC-dim$[\mathcal{Y}(\cdot)]$ corresponding to $\mathcal{P}(\mathcal{Z})$. Before stating such results, let us take a quick look at the `lossless' setting itself.

\begin{remark}[Lossless encoding in WAEs]
    This occurs only when $E$ is an exact measure preserving map. Since we allow the distributions to have densities, the idea translates to having a change of variables given as $\int_{\Omega_{z}}p_{\rho}(z) dz = \int_{\Omega_{x}}(p_{\rho} \circ E)(x) \:[J_{E}(x)] dx$, where in general, $J_{E}(x)$ is the Clarke differential or generalized Jacobian at $x$ \citep{chiappori2017ot} \footnote{See the exact form the transported density achieves under the co-area formula in \citet{mccann2020optimal}, Section 2. In general cases concerning transformations between variables, the surplus multiplicand is rather $\textrm{vol}[J_{E}(x)]:=$ product of singular values of the $k \times d$ Jacobian matrix $J_{E}$ \citep{ben1999change}.}. This boils down to the more familiar $\abs{\textrm{det}(\nabla E(x))}$, given $E: \mathbb{R}^{d} \rightarrow \mathbb{R}^{d}$ is injective on its domain and continuously differentiable. Consequently, our desired distribution $p_{\rho}$ turns out to be the density corresponding to a $k$-marginal of $E_{\#}\mu$ \footnote{Such transformations can merely be of the Rosenblatt type. Given that $p_{\rho} \in \mathcal{C}^{m_{z}}_{L}(\Omega_{z})$, for some $m_{z} \in \mathbb{R}_{> 0}$; $E$ can be shown to be smooth in the sense of H\"{o}lder-Zygmund \citep{asatryan2023convenient}.}. The existence, let alone the regularity of such a map is not automatically guaranteed. In case both $\mu$ and $\rho$ are nonatomic, such that $\mu$ vanishes on all Lipschitz $(d-1)$-surfaces, there exists a unique $E$ ($\mu$ a.e.), that offers a lossless encoding \citep{McCann1995Optimal}. Brenier theorem \citep{brenier1991polar} argues the same under the additional assumption that the distributions have finite variance. Though obtained under restrictive scenarios, a Brenier map pushing forward the standard Gaussian measure to a uniformly log-concave target distribution would be locally Lipschitz \citep{caffarelli2000monotonicity,COLOMBO20211007}. While this belongs to the class of possible encoders under a special case, it is rather challenging to verify that minimizing the WAE loss attains such a solution. Moreover, given a sample (semi-discrete) problem like ours, the optimal map is likely to be discontinuous. Even if they do not, the injectivity will be sacrificed when the supports are unbounded since they cannot be continuously embedded at the same time.
\end{remark}

Hence, we focus on finding the tolerable margins of losses incurred during encoding instead. (\ref{encode}) becomes a platform on which the exploration of deterministic upper bounds rests. To obtain an upper bound in the case of WAE-GANs, first notice that given $\rho_{1}, \rho_{2} \in \mathcal{P}(\mathcal{Z})$, equipped with corresponding densities such that $\rho_{1} \ll \rho_{2}$,
\begin{align*}
    \textrm{JS}(\rho_{1}, \rho_{2}) \leq \left[\pi \ln\left(\frac{1}{\pi}\right) + (1-\pi) \ln\left(\frac{1}{1-\pi}\right)\right] d_{\textrm{TV}}(\rho_{1}, \rho_{2}) \leq \ln(2) \:d_{\textrm{TV}}(\rho_{1}, \rho_{2})
\end{align*}
(this form is also called the \textit{Information Transmission Rate} \citep{topsoe2000some}), $0\leq \pi \leq 1$. While easier to calculate, it is often technically challenging to deal with JS from a density estimation perspective. By convention, it is assigned value $+\infty$ in case the underlying distributions do not have densities, and as a result, does not metrize $\mathcal{P}(\mathcal{Z})$ in general. However, when taken under square root, JS follows the triangle inequality \citep{endres2003new}. Now, given a Lipschitz encoder $E$, let us consider the realized latent loss under JS-divergence due to the RIK estimator $\Tilde{\mu}_{n}$ (as discussed in Theorem \ref{theo: lip}), defined as $\frac{d\Tilde{\mu}_{n}}{dx} = \frac{1}{nh^d}\sum^{n}_{i=1} K\left(\frac{x}{h}, \frac{x_i}{h}\right)$, $x \in \Omega_x$ where $h_{n} \rightarrow 0$. Fragmenting the same based on unique sources of variation yields,
\begin{align} \label{JS_ineq}
    f(\pi)^{-1}\:\textrm{JS}(E_{\#}\Tilde{\mu}_{n}, \rho) - \Delta_{E,n} \leq d_{\textrm{TV}}\left(E_{\#}\Tilde{\mu}_{n}, \widehat{(E_{\#}\mu)}_n\right) + d_{\textrm{TV}}\left(\widehat{\rho}_n,\rho\right),
\end{align}
given that $\Delta_{E,n} = d_{\textrm{TV}}(\widehat{(E_{\#}\mu)}_n,\widehat{\rho}_n)$, which essentially (in asymptotic regime) determines how much latent regularization can be tolerated. As such, the choice of the optimal encoder must be the one that minimizes $\Delta_{E,n}$. It coincides with the minimum distance estimator (\citet{devroye2001combinatorial}, Theorem 6.4) of $\rho$ amongst encoded candidates ($E_{n}^{*} = \argmin d_{\textrm{TV}}(\widehat{(E_{\#}\mu)}_n,\widehat{\rho}_n)$). Since the other error terms shrink arbitrarily asymptotically ($n \rightarrow \infty$ with fixed $d,k$), given $t \in \mathbb{R}_{>0}$ (as in \ref{const}), we only need to ensure that $\Delta_{E_{n}^{*},n} < t$.   

\begin{remark}[Cost to the Scheff\'{e} Tournament]
    Substituting the plug-in estimates $\widehat{(E_{\#}\mu)}_n$ and $\widehat{\rho}_n$ with smoother alternatives might be computationally beneficial. A successful search for the Scheff\'{e} tournament-winning encoder takes quadratic time \citep{acharya2014sorting}. Instead, let us consider estimators $\widetilde{(E_{\#}\mu)}$ and $\Tilde{\rho}$ respectively, such that $\frac{d\widetilde{(E_{\#}\mu)}}{dz}, \frac{d\Tilde{\rho}}{dz} \in L_{1}(\mathbb{R}^{k})$. The benefit is rooted in viewing the problem from an OT standpoint. Let us write,
    \begin{align}
        &2 d_{\textrm{TV}}\left(\widetilde{(E_{\#}\mu)}_n , \Tilde{\rho}_n\right) = \norm{\widetilde{(E_{\#}\mu)}_n - \Tilde{\rho}_n}_{1} \nonumber \\ &\leq \underbrace{\norm{\widetilde{(E_{\#}\mu)}_n - K_{h}(\widetilde{(E_{\#}\mu)}_n)}_{1}}_{\text{(i)}} + \underbrace{\norm{K_{h}(\widetilde{(E_{\#}\mu)}_n) - K_{h}(\Tilde{\rho}_n)}_{1}}_{\text{(ii)}} + \underbrace{\norm{K_{h}(\Tilde{\rho}_n) - \Tilde{\rho}_n}_{1}}_{\text{(iii)}} \label{rem_scheffe},
    \end{align}
    where given $\rho_{1} \in \mathcal{P}(\mathcal{Z})$, $K_{h}(\rho_{1}) = \int_{\Omega_z} K_{h}(.,y)\rho_{1}(dy) = \frac{1}{h^k} \int_{\Omega_z} K(\frac{.}{h},\frac{y}{h})\rho_{1}(dy)$ defines the convolution with RI kernel $K$. Also, $\frac{y}{h} = (\frac{y_1}{h},...,\frac{y_k}{h})^{'}$, for $h > 0$. The terms (i) and (iii) both $\rightarrow 0$ as $h \rightarrow 0$ (\citet{gine2021mathematical}, Proposition 4.3.31). On the other hand,  
    \begin{align}
        \norm{K_{h}(\widetilde{(E_{\#}\mu)}_n) - K_{h}(\Tilde{\rho}_n)}_{1} &\leq \int \Big\{\frac{1}{h^k} \int \abs{K(\frac{z}{h},\frac{y}{h}) -  K(\frac{z}{h},\frac{y^{'}}{h})} dz \Big\}d\Pi(y,y^{'}) \label{scheffe_jen} \\ &= \int \Bigg\{\frac{\int \abs{K(z^{'},\frac{y}{h}) -  K(z^{'},\frac{y^{'}}{h})} dz^{'}}{\norm{y-y^{'}}}\Bigg\} \:\norm{y-y^{'}} d\Pi(y,y^{'}) \nonumber \\ &\lesssim \frac{1}{h} \int \norm{y-y^{'}} d\Pi(y,y^{'}) \label{scheffe_was}, 
    \end{align}
    where (\ref{scheffe_jen}) is due to Jensen's inequality and $\Pi$ denotes a coupling between $\widetilde{(E_{\#}\mu)}_n$ and $\Tilde{\rho}_n$. The invariance of $K$ implies the inequality (\ref{scheffe_was}), which holds for all such measure couples. Hence, given $c_{z} \equiv L_{2}$, the quantity (ii) $\lesssim \frac{1}{h}d_{\mathcal{L}_{c_{z} }^{1}}(\widetilde{(E_{\#}\mu)}_n, \Tilde{\rho}_n)$. As such, to obtain an optimal encoder $E_{n}^{*}$--- achieving latent consistency--- it is sufficient to compute $\Delta^{'}_{E_{n}^{*},n} = \inf_{E}d_{\mathcal{L}_{c_{z} }^{1}}(\widetilde{(E_{\#}\mu)}_n, \Tilde{\rho}_n)$ instead. This is highly maneuverable computationally due to the sheer attention the problem has received in recent years. One can achieve a complexity of $\Tilde{O}(\frac{n^{\frac{9}{4}}}{t} \wedge \frac{n^{2}}{t^{2}})$ \citep{dvurechensky18a}, even beyond what Sinkhorn's algorithm offers.
\end{remark}


Inequality (\ref{JS_ineq}) provides a clear pathway to a non-asymptotic upper bound to the realized latent loss in a WAE-GAN setup. Given Assumptions \ref{assumption1} and \ref{assumption2}, we begin with $\Tilde{\mu}_{n}$, an RIK density estimate (strongly invariant) of $\mu$ based on bandwidth $h \equiv h_{n} \rightarrow 0$ as $n \rightarrow \infty$, such that $\frac{nh^{d}_{n}}{\abs{\log h^{d}_{n}}} \rightarrow \infty$. Corollary \ref{cor:vc_IPT} implies the existence of a positive constant $E_{1}^{''}$ such that  \begin{equation}
    d_{\textrm{TV}}\left(g_{\#}\Tilde{\mu}_{n}, \widehat{(g_{\#}\mu)}_n\right) \geq t + \mathcal{O}(h^{m_{x}} \vee \sqrt{v_{z}} n^{-\frac{1}{2}})
\end{equation}
holds with probability $ \leq E_{1}^{''} \exp\left\{-E_{3}(1 \wedge h^{d}) nt^{2} \right\}$, whenever for $t \geq \sqrt{\frac{\abs{\log h_{n}}}{nh^{d}_{n}}}$ \citep{gine2002rates} and $g \in \mathscr{F}_{L}(\mathcal{X}, \mathcal{Z})$. This eventually determines the rate associated with the probabilistic statement
\begin{equation*}
    \textrm{JS}(g_{\#}\Tilde{\mu}_{n}, \rho) - f(\pi)\sup_{\rho^{\otimes n}}\Delta_{L,n} = o_{\mathbb{P}}(1),
\end{equation*}
obtained as a consequence of (\ref{JS_ineq}) and assuming VC-dim$[\mathcal{Y}(\mathcal{P}(\mathcal{Z}))] = v_{z} < \infty$. Here, $\Delta_{L,n} = \inf_{g \in \mathscr{F}_{L}(\mathcal{X}, \mathcal{Z})}d_{\textrm{TV}}(\widehat{(g_{\#}\mu)}_n,\widehat{\rho}_n)$ and the supremum is taken over naive estimators constructed based on $n$ replicates from $\rho$. If one employs instead a NN-based encoder $\phi \in \Phi(W, L)^{k}_{d}$ according to our previous prescriptions--- for example Corollary \ref{IPT_ReLU_Group} (i) or (ii)--- an additional estimation error is duly incurred. This, along with the realized $\Delta_{\Phi,n}$ contributes to the extent of tolerable latent loss.


\begin{remark}
    There are some interesting implications in The WAE-GAN regime if along with Assumption \ref{assumption2}, there exists $m_{z} \in \mathbb{R}_{>0}$, such that $p_{\rho} \in \mathcal{C}^{m_{z}}_{L^{'}}(\Omega_{z})$, $L^{'} > 0$. Observe that, due to its definition
\begin{align}
    \textrm{JS}(E_{\#}\Tilde{\mu}_{n}, \rho) &= \pi \textrm{KL}\left(E_{\#}\Tilde{\mu}_{n} \Big| \mathcal{M}_{E_{\#}\Tilde{\mu}_{n},\rho}(\pi)\right)+ (1-\pi)\textrm{KL}\left(\rho \Big| \mathcal{M}_{E_{\#}\Tilde{\mu}_{n},\rho}(\pi)\right) \nonumber \\ & \geq \frac{1}{2}\left[\pi d^{2}_{\textrm{TV}}\left(E_{\#}\Tilde{\mu}_{n}, \mathcal{M}_{E_{\#}\Tilde{\mu}_{n},\rho}(\pi)\right) + (1-\pi)d^{2}_{\textrm{TV}}\left(\rho, \mathcal{M}_{E_{\#}\Tilde{\mu}_{n},\rho}(\pi)\right)\right] \label{Pinsker}\\ & = \frac{1}{2}\pi(1-\pi) d^{2}_{\textrm{TV}}\left(E_{\#}\Tilde{\mu}_{n}, \rho\right) \geq \frac{1}{8}\pi(1-\pi)\norm{E_{\#}\Tilde{\mu}_{n} - \rho}^{2} \label{TV_ineq},
\end{align}
where we define the the mixture as $\mathcal{M}_{E_{\#}\Tilde{\mu}_{n},\rho}(\pi) = \pi E_{\#}\Tilde{\mu}_{n} + (1-\pi) \rho$. The step (\ref{Pinsker}) is due to Pinsker's inequality. We reach (\ref{TV_ineq}) using 
\begin{align*}
    d_{\textrm{TV}}\left(E_{\#}\Tilde{\mu}_{n}, \mathcal{M}_{E_{\#}\Tilde{\mu}_{n},\rho}(\pi)\right) &= \sup_{\omega \in \sigma(\mathcal{Z})}\abs{E_{\#}\Tilde{\mu}_{n}(\omega) - \pi E_{\#}\Tilde{\mu}_{n}(\omega) - (1-\pi) \rho(\omega)} \\ & = (1-\pi)\sup_{\omega \in \sigma(\mathcal{Z})}\abs{E_{\#}\Tilde{\mu}_{n}(\omega) -  \rho(\omega)} = (1-\pi)d_{\textrm{TV}}\left(E_{\#}\Tilde{\mu}_{n},\rho\right).
\end{align*}
Typically, the value of $\pi$ is taken to be $1/2$. Now,
\begin{equation*}
    \norm{E_{\#}\Tilde{\mu}_{n} - \rho}^{2} \geq \inf_{\Tilde{\rho}_{n}} \norm{\Tilde{\rho}_{n} - \rho}^{2},
\end{equation*}
where the infimum is taken over the class of RIK density estimates based on $n$ i.i.d. samples from $p_{\rho}$. Such estimators, under the $L_{2}$ loss tend to have the optimal convergence rate, i.e. $\inf_{\Tilde{\rho}_{n}} \mathbb{E}\norm{\Tilde{\rho}_{n} - \rho}^{2} \gtrsim n^{-\frac{2m_{z}}{2m_{z}+k}}$ (\citet{van2000asymptotic}, Theorem 24.4). As such, this gives us a sharp lower bound for latent performance.    
\end{remark}

Since the latent distribution embodies the hidden representation in input images, it must also remain invariant to certain deformations. For example, latent codes corresponding to an image of a lesion and its rotated counterpart ($\textrm{SO}(k)$) should ideally appear equiprobably. Generative models achieve this by ensuring group symmetry in the target space \citep{Birrell2022StructurepreservingG}. Now, given a group $\Sigma$ (an ordered pair of a nonempty set and a binary operation, satisfying the group axioms), a \textit{group action} $\varphi$ on $\mathcal{Z}$ is an automorphism defined as $\varphi_{\sigma} = \varphi(\sigma, \cdot) : \mathcal{Z} \rightarrow \mathcal{Z}$, for all $\sigma \in \Sigma$, also satisfying $\varphi_{\sigma_{1}} \circ \varphi_{\sigma_{2}} = \varphi_{\sigma_{1}\cdot \sigma_{2}}$, $\forall \sigma_{1}, \sigma_{2} \in \Sigma$. The following definition gives such a specific characterization to $\rho$.
\begin{definition}[Invariant Distributions]
    Given a group $\Sigma$, the class of $\Sigma$-invariant probability distributions on $\mathcal{Z}$ is defined as 
    \begin{equation*}
        \mathcal{P}_{\Sigma}(\mathcal{Z}) = \{\mathbb{P} \in \mathcal{P}(\mathcal{Z}) : \mathbb{P} = (\varphi_{\sigma})_{\#}\mathbb{P}, \forall \sigma \in \Sigma\}.
    \end{equation*}
\end{definition}
Throughout the paper, we only consider finite groups, i.e. $\abs{\Sigma} < \infty$. This makes the representation of the underlying space under group actions much easier. To that end, let us introduce the \textit{fundamental domain} $\mathcal{Z}_{0} \subset \mathcal{Z}$, which is defined such that the subsets $\sigma \mathcal{Z}_{0}$, $\sigma \in \Sigma$ form a locally finite cover of $\mathcal{Z}$ without sharing common interior points. This translates to saying that there exists a unique $z_{0} \in \mathcal{Z}_{0}$ corresponding to each $z \in \mathcal{Z}$ such that $z=\sigma z_{0}$ \citep{chen2023sample}. In order to adapt to the estimation of $\Sigma$-invariant latent distributions, we make additional assumptions for the underlying kernels in MMDs.

\begin{assumption}[Group Invariant Kernels] \label{group_ker}
    The kernel $\kappa(\cdot,\cdot)$ satisfies $\forall \sigma (\neq \textrm{id}) \in \Sigma$
    \begin{enumerate}[label=(\roman*),itemsep=0.2cm]
        \item Given $z,z^{'} \in \mathcal{Z}$, $\kappa(\sigma z, \sigma z^{'}) = \kappa(z, z^{'})$, and
        \item There exists $0<\varsigma_{\kappa,\Sigma}<1$ such that $\kappa(\sigma z, z) \leq \varsigma_{\kappa,\Sigma}C_{\kappa}$, where $z \in \mathcal{Z}_{0}$.
    \end{enumerate}
\end{assumption}

\begin{theorem}[Latent Space Consistency in WAE-MMD under Invariance] \label{latent_con_MMD}
    Let, $\rho \in \mathcal{P}_{\Sigma}(\Omega_{z})$ such that $\abs{\Sigma} < \infty$. Also, let $\kappa(\cdot,\cdot)$ be strongly invariant satisfying Assumption \ref{group_ker} such that $\sup_{z \in \Omega_{z}}\kappa(z,z) \leq C_{\kappa}$, for $C_{\kappa} > 0$. Then, there exists a probabilistic encoder $\phi \in \Phi(W, L)^{k}_{d}$, based on ReLU activations with $W = \mathcal{O}(d\floor{N_{1}^{\frac{1}{d}}} \vee N_{1}+1)$ and $L = \mathcal{O}(N_{2})$ such that given $\delta>0$, we have with probability $1-\delta$
    \begin{align*}
        d_{\mathcal{H}_{\kappa}}&\left(\phi_{\#}\hat{\mu}_{n}, \rho\right) - \sqrt{C_{\kappa}} \sup_{\rho^{\otimes n}}\Delta_{\Phi,n} \leq \sqrt{c_{g,n}L}\left(\frac{\max\{B^{2}_{x}, 4C_{\kappa}[1+\varsigma_{\kappa,\Sigma}\left(\abs{\Sigma}-1\right)]\}}{2n}\ln{(\frac{2}{\delta})}\right)^{\frac{1}{4}} \\ &+ \mathcal{O}(\sqrt{c_{g,n}}(d^{2}n)^{-\frac{1}{2d}}) + \sqrt{\frac{2C_{\kappa}}{n}}\left[1 + \sqrt{\frac{1+\varsigma_{\kappa,\Sigma}\left(\abs{\Sigma}-1\right)}{\abs{\Sigma}}}\right] + \mathcal{O}(\sqrt{d D_{n}}N_{1}^{-\frac{2}{d}}N_{2}^{-\frac{2}{d}}),
    \end{align*}
    where both $D_{n}$ and $c_{g,n}$ are sequences based on $n$ that $\searrow 0$ almost surely as $n \rightarrow \infty$.
\end{theorem}

\begin{remark}[Mitigating Curse of Dimensionality] \label{curse_dim}
    The term contributing to a slower convergence rate (second on the RHS) due to its dependence on $d$ is rooted in the estimation error under the Wasserstein metric (see proof). While we do not allow the input dimension $d$ to grow as a function of $n$, it being inherently large degrades the sharpness of the non-asymptotic bound. In search of a remedy, we recall the solution \citet{chakrabarty2021statistical} resorted to, namely the 1-upper Wasserstein dimension ($d^{*}_{1}$). It is typically smaller compared to the Minkowski–Bouligand dimension. However, in case $\mu$ is essentially supported on a `latent' regular space, e.g. a compact $d^{'}$-dimensional differentiable manifold ($d^{'} < d$), we have $d^{*}_{1} = d^{'}$ \citep{Weed2017Sharp}. This suits our discussion since such a phenomenon regularly occurs in image datasets \citep{pope2021the}. The definition of $d^{*}_{1}$ goes as follows
    \begin{align*}
        d^{*}_{1}(\mu) = \inf\{s \in (2,\infty): \limsup_{\varepsilon \rightarrow 0}\frac{\mathcal{N}_{\varepsilon}(\mu, \varepsilon^{\frac{s}{s-2}})}{-\log(\varepsilon)} \leq s\},
    \end{align*}
    where $\mathcal{N}_{\varepsilon}$ denotes the covering number. Now, if $s > d^{*}_{1}(\mu)$, the second term on the right hand of the inequality in Theorem \ref{latent_con_MMD} can be replaced by $\mathcal{O}(\sqrt{c_{g,n}}n^{-\frac{1}{2s}})$. We also point out that the other term carrying $d$ in the exponent does not contribute to asymptotic rates since the encoders constructed are usually of fixed proportions. Since the upper bound becomes $o(1)$, using the Borel-Cantelli lemma, the latent WAE-MMD error deviated from $\sqrt{C_{\kappa}} \sup_{\rho^{\otimes n}}\Delta_{\Phi,n}$ vanishes almost surely. 
\end{remark}



\subsection{Simulations} \label{sim_latent}
To validate our findings empirically, we carry out experiments on both real and synthetic data [Fig \ref{fig:dataset}]. The existing data set we work on is MNIST, consisting of 70,000 2D images of hand-written digits. The other data set, `Five-Gaussian' is a collection of 50,000 observations, drawn independently at random out of five trivariate Gaussians with unit dispersion and mean at five vertices of a unit cube. We run both WAE-GAN and WAE-MMD to reconstruct observations from the two data sets. All the experiments were carried out on an RTX 3090 GPU.




\begin{figure}[ht]
    \centering
    \includegraphics[width=.7\linewidth]{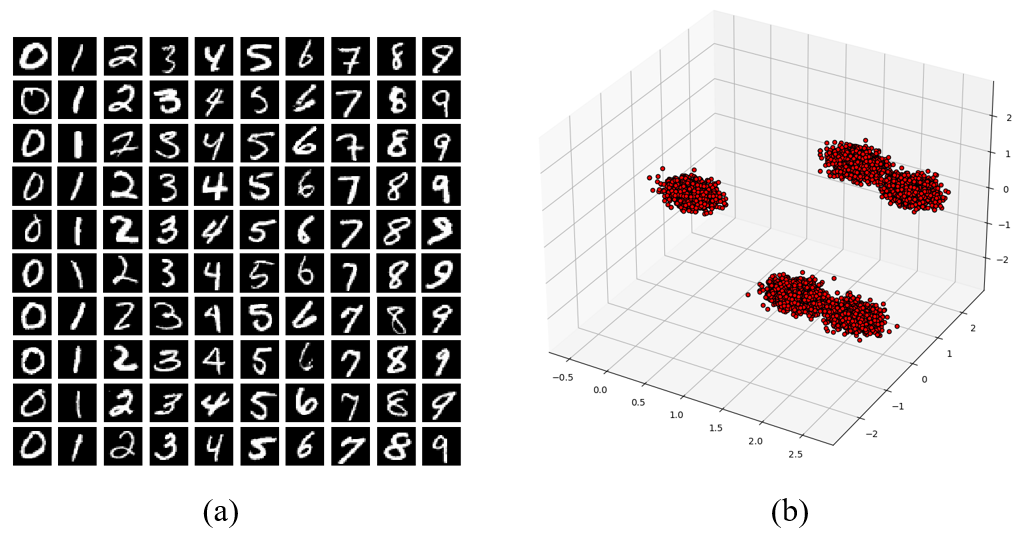}
    \caption{The (a) MNIST and (b) Five-Gaussian data sets.}
    \label{fig:dataset}
\end{figure}

\textit{Five-Gaussian:} The encoders we use in the case of Five-Gaussian data, map the points to a two-dimensional latent space. The first kind we deploy is $4$-deep and is based on ReLU activation. The last layer uses an additional rescaling to span the target support and mitigate zero-inflation. To suit our theoretical specifications, we experiment with diverse latent distributions. Namely, Bivariate standard Gaussian, and the classes of bivariate distributions having Beta and Exponential marginals respectively. We call them Beta and Exponential \textit{copulas} by somewhat abusing the terminology. This way we encompass unbounded supports, multimodality, and skewed densities. Firstly, we train the model based on the entire sample ($n=50,000$) to obtain the nearest estimate of the population loss. Our goal now is to observe the propagation of the losses as we gradually increase the sample size $n = (1000, 3000, 5000, \cdots, 50,000)$, drawn uniformly at random. Our findings from 20 runs corresponding to each $n$ are given as follows.  

\begin{figure}[ht]
    \centering 
\begin{subfigure}{0.29\textwidth}
  \includegraphics[width=\linewidth]{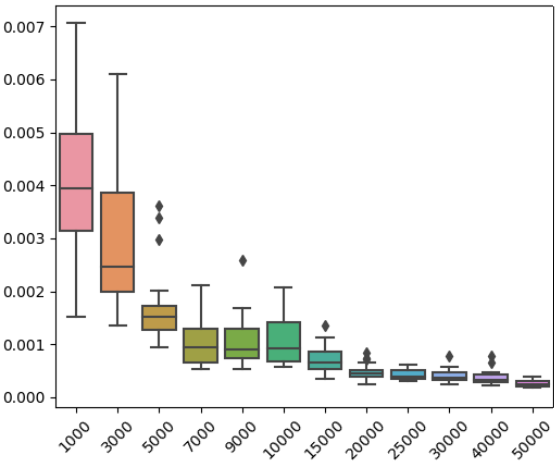}
  \caption{Gaussian}
  \label{fig:JS_Gauss}
\end{subfigure}\hspace{3pt} 
\begin{subfigure}{0.32\textwidth}
  \includegraphics[width=\linewidth]{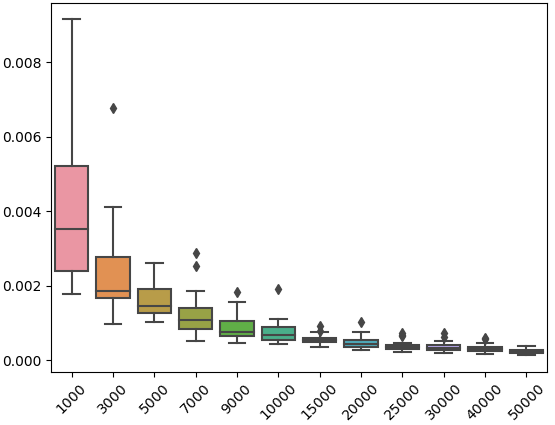}
  \caption{Beta copula}
  \label{fig:JS_Beta}
\end{subfigure}\hspace{3pt} 
\begin{subfigure}{0.32\textwidth}
  \includegraphics[width=\linewidth]{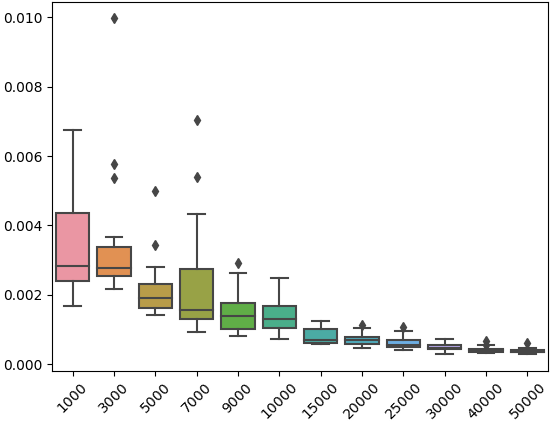}
  \caption{Exponential copula}
  \label{fig:JS_Exp}
\end{subfigure}\hfil 
\begin{subfigure}{0.6\textwidth}
  \includegraphics[width=\linewidth]{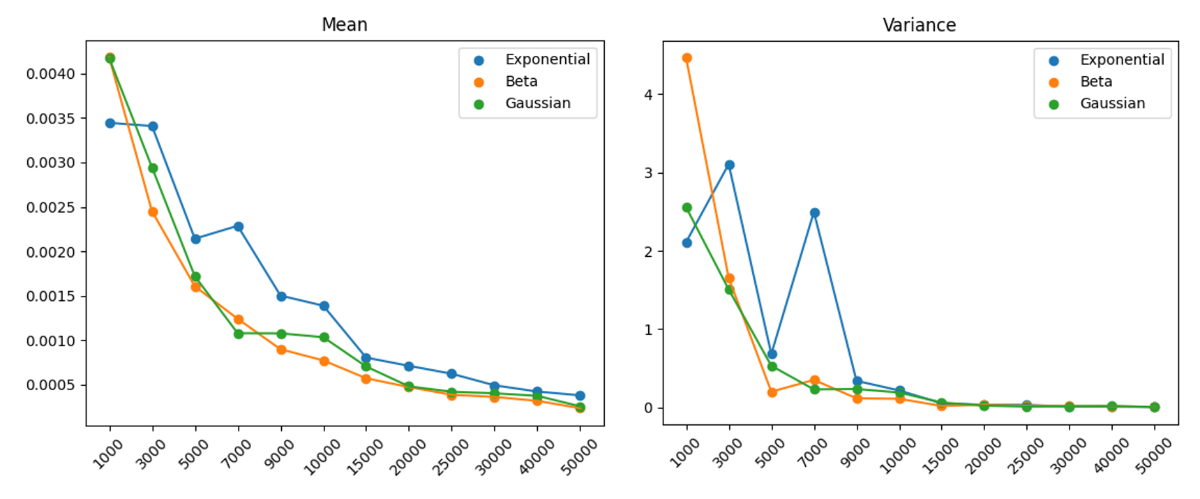}
  \caption{Mean and Variance of latent loss}
  \label{fig:JS_Mean_Var}
\end{subfigure}\hfil 
\caption{Latent loss under Jensen-Shannon divergence and ReLU encoders. The Lagrangian weight assigned to the latent space, as given in \ref{WAE_loss_lag}, is taken to be $\lambda = 0.2$. We consider both Gaussian and Exponential marginal densities standard. The parameters for Beta marginals are taken as $(0.5,0.8)$.}
\label{fig:JS}
\end{figure}

The illustrations show that for all three latent laws, the sample losses approach their population counterpart with diminishing variance. In search of the sharpest asymptotic rate associated--- even beyond the theoretically achievable $\mathcal{O}(n^{-\frac{1}{2}})$--- we observe the movement of the loss multiplied by $n$. In our experiments [Fig \ref{fig:JS*n}], such a sequence also tends to converge to a small constant, namely the population error margin. This is an empirical guarantee to the impression \citet{asatryan2023convenient} [Remark 4.6] had ($\hat{d}_{\textrm{JS}} \sim n^{-1}$) in a parametric GAN generation. However, using a GroupSort activated encoder (grouping size $2$) one may observe the approximate rate of $\mathcal{O}(n^{-\frac{1}{2}})$ in diminishing MMD losses [Fig \ref{fig:GB_recon_MMD_Group} (b), (c)]. The choice of regularizing parameter $\lambda$ is chosen based on a trade-off between quality reconstruction and latent performance. 

\begin{figure}[ht]
\begin{subfigure}{0.29\textwidth}
  \includegraphics[width=\linewidth]{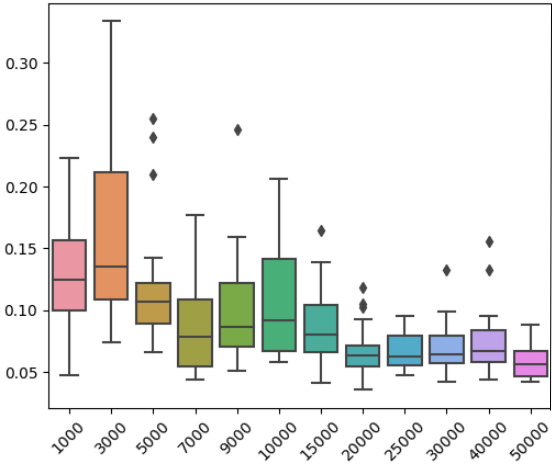}
  \caption{Gaussian}
  \label{fig:JS_Gauss*n}
\end{subfigure}\hspace{2pt} 
\begin{subfigure}{0.32\textwidth}
  \includegraphics[width=\linewidth]{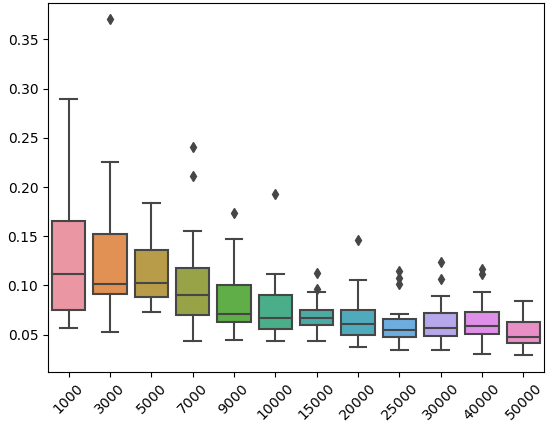}
  \caption{Beta}
  \label{fig:JS_Beta*n}
\end{subfigure}\hspace{2pt} 
\begin{subfigure}{0.32\textwidth}
  \includegraphics[width=\linewidth]{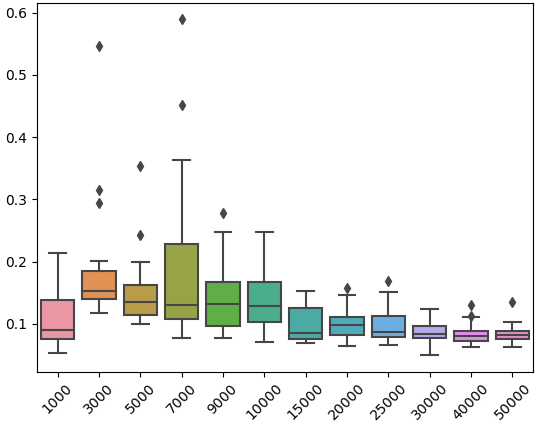}
  \caption{Exponential}
  \label{fig:JS_Exp*n}
\end{subfigure} 
\caption{Propagation of sample corrected ($\times n$) latent JS loss under ReLU encoders.}
\label{fig:JS*n}
\end{figure}

We follow the same experimental protocol for WAE-MMD [Fig \ref{fig:MMD_total}], taking the latent distribution as bivariate standard Gaussian. Here also, a similar trait is noticed. The observed MMD losses gradually decrease to their population counterpart with diminishing variability. The rate of convergence however, becomes comparable to $\mathcal{O}(n^{-\frac{1}{2}})$, attesting the theoretical result. While the latent loss--- asymptotically at the least--- moves close to nullity, the bin estimates corresponding to the target and the encoded law must differ. This discrepancy, as we have already discussed, is rooted in the information preserved.  

\begin{figure}[ht]
    \centering 
\begin{subfigure}{0.31\textwidth}
  \includegraphics[width=\linewidth]{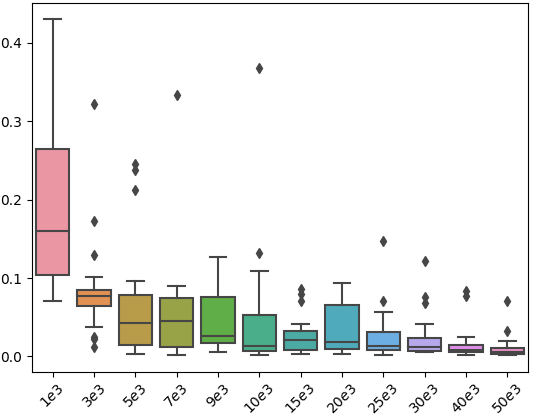}
  \caption{}
  \label{fig:MMD}
\end{subfigure} 
\begin{subfigure}{0.32\textwidth}
  \includegraphics[width=\linewidth]{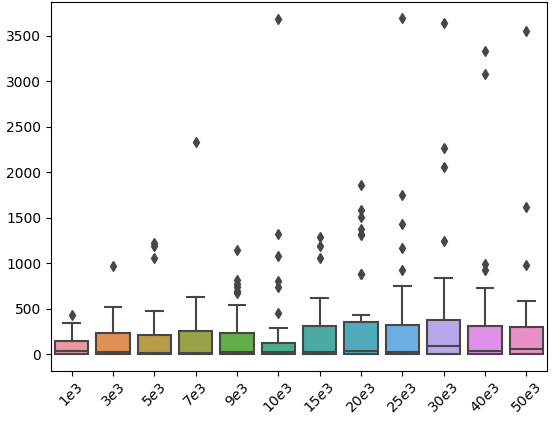}
  \caption{}
  \label{fig:MMD*n^1/3}
\end{subfigure} 
\caption{Propagation of (a) sample MMD losses and (b) sample corrected ($\times n^{\frac{1}{2}}$) MMD losses, trained with the Lagrangian parameter $\lambda = 0.2$ using ReLU encoders.}
\label{fig:MMD_total}
\end{figure}

We study the concentration of encoded bins in contrast with the latent ones over regular intervals of 200 training epochs. It is fascinating to observe the rearrangement of density as the losses slowly diminish. We present a detailed commentary on the same in the Appendix [Fig \ref{fig:JS_Beta_Latent_whole}, \ref{fig:MMD_Gauss_Latent_whole}]. Here instead, we show the encoded estimates after the completion of 2000 epochs [Fig \ref{fig:Hist_latent}]. The information retention can be readily identified from the high-density areas representing the distinct clusters. The quantile-quantile (QQ) plots [Fig \ref{fig:MMD_Gauss_qqplot}] between the empirical encoded distribution and the targets also tell the same story. To check the extent of approximation, we also perform multivariate goodness-of-fit tests (see Appendix \ref{appB}).        

\begin{figure}[ht]
     \centering 
\begin{subfigure}{0.32\textwidth}
  \includegraphics[width=\linewidth]{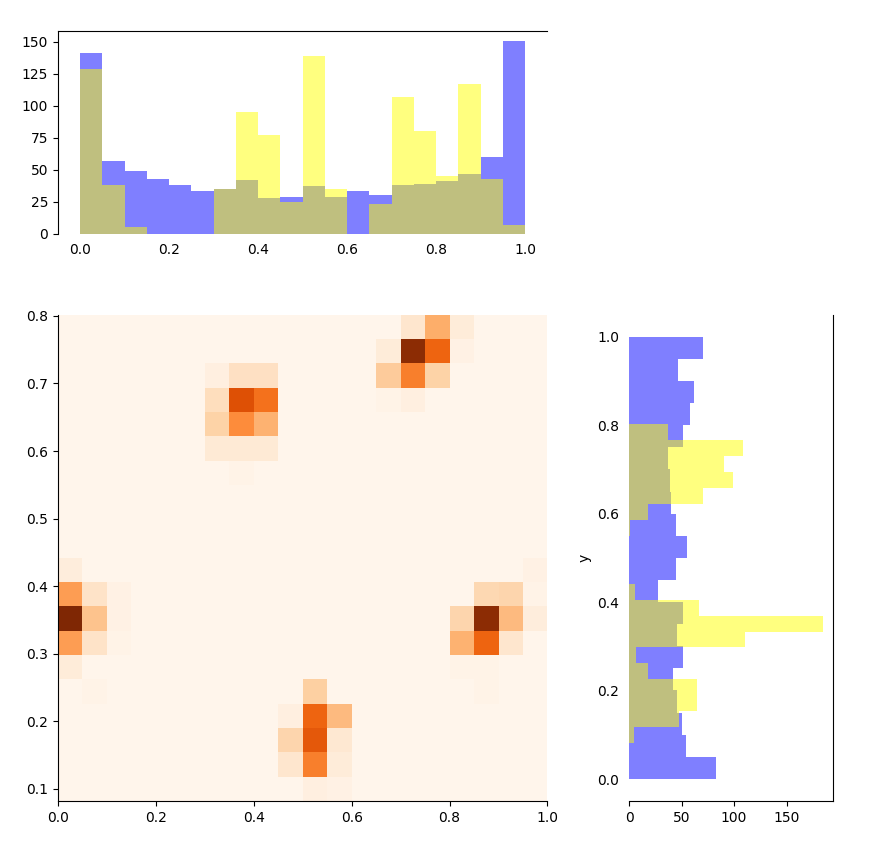}
  \caption{}
  \label{fig:Bin_JS_Beta}
\end{subfigure} 
\begin{subfigure}{0.32\textwidth}
  \includegraphics[width=\linewidth]{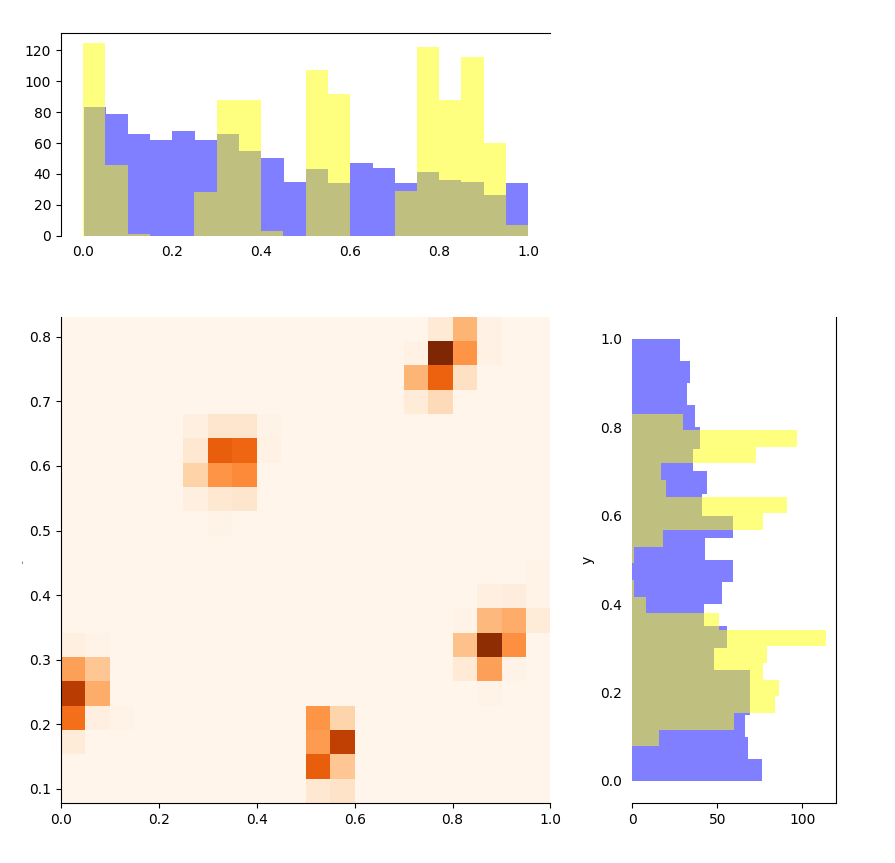}
  \caption{}
  \label{fig:Bin_JS_Exp}
\end{subfigure} 
\begin{subfigure}{0.32\textwidth}
  \includegraphics[width=\linewidth]{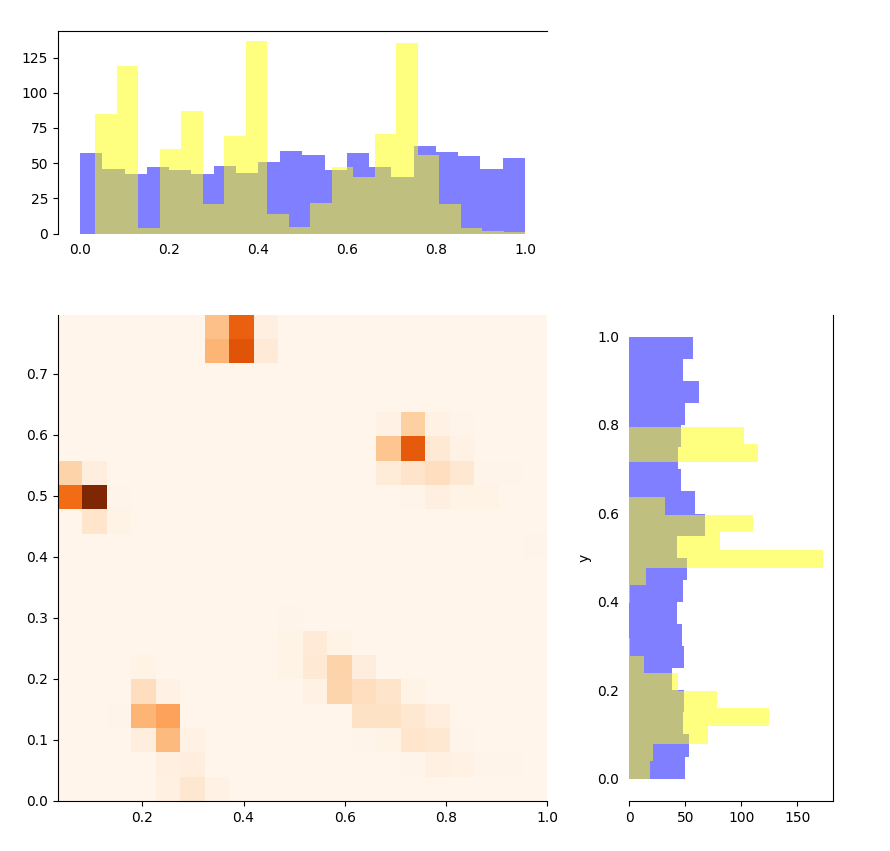}
  \caption{}
  \label{fig:Bin_MMD_Gauss}
\end{subfigure} 
\caption{Bin estimates of ReLU-encoded ({\color{yellow}yellow}) vs latent distribution ({\color{blue}blue}) in case of the Five-Gaussian data. Under the JS loss, we observe (a) Beta $(0.5,0.8)$ and (b) standard Exponential copula, (c) shows standard Gaussian under MMD loss. (Effective range of values scaled to aid visualization)} 
\label{fig:Hist_latent}
\end{figure}

For checking the efficacy of GroupSort activations in encoding, we repeat the experiment in a WAE-GAN setup [Fig \ref{fig:Groupsort}]. The regularization remains at $\lambda =0.2$ and grouping size is taken as $2$ (OPLU). Quite similar to previous observations, the losses tend to decrease at a familiar rate. 
\begin{figure}[ht]
\begin{subfigure}{0.32\textwidth}
  \includegraphics[width=\linewidth]{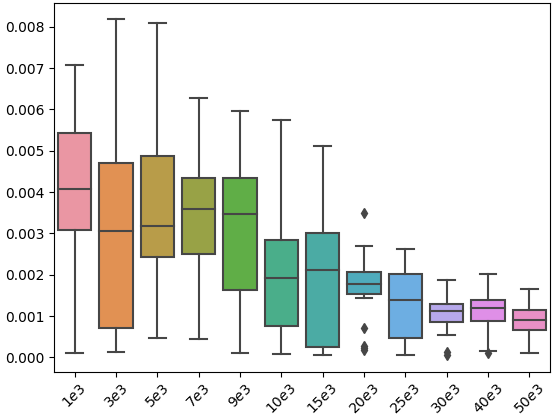}
  \caption{Gaussian}
  \label{fig:Groupsort_Gauss}
\end{subfigure}\hspace{2pt} 
\begin{subfigure}{0.32\textwidth}
  \includegraphics[width=\linewidth]{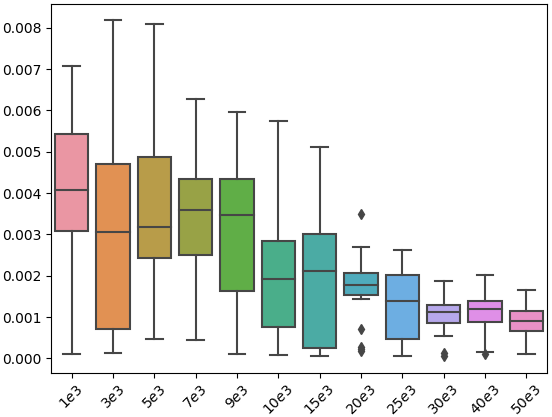}
  \caption{Beta}
  \label{fig:Groupsort_Beta}
\end{subfigure}\hspace{2pt} 
\begin{subfigure}{0.32\textwidth}
  \includegraphics[width=\linewidth]{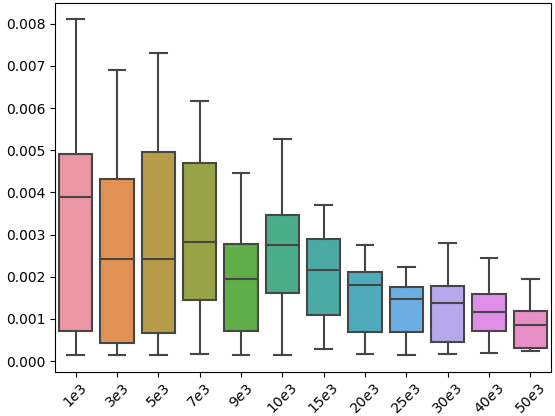}
  \caption{Exponential}
  \label{fig:Groupsort_Exp}
\end{subfigure} 
\caption{Latent JS loss under Groupsort encoders of grouping size 2.}
\label{fig:Groupsort}
\end{figure}

\textit{MNIST:} The individual images in MNIST are of size $(28 \times 28)$. In vectorized form, the input observations are reduced to $d=512$. Here also we deploy a $4$-deep ReLU encoder with layers of width 512-256-128-64$=k$. During decoding, the output tensor is reshaped to have a size $(\textrm{batch size}, 1, 28, 28)$, which in turn enables us to calculate the reconstruction loss. Keeping in mind the high dimensionality of the latent space, we only consider a multivariate standard Gaussian target. The findings from training runs on both WAE-GAN and WAE-MMD with regularization $\lambda=0.2$ are obtained as in Fig \ref{fig:MNIST_Gauss_latent}. 
\begin{figure}[ht]
     \centering 
\begin{subfigure}{0.34\textwidth}
  \includegraphics[width=\linewidth]{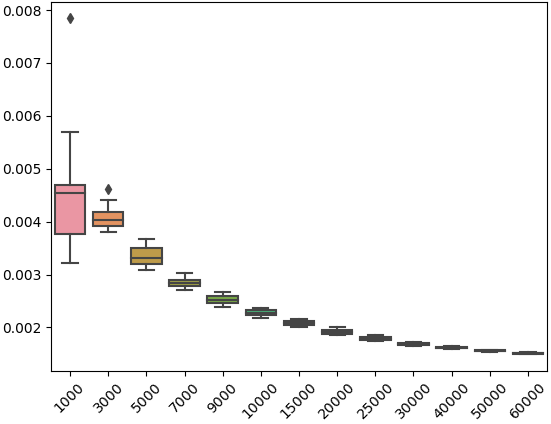}
  \caption{}
  \label{fig:MNIST_JS_Gauss}
\end{subfigure} 
\hspace{4pt}
\begin{subfigure}{0.32\textwidth}
  \includegraphics[width=\linewidth]{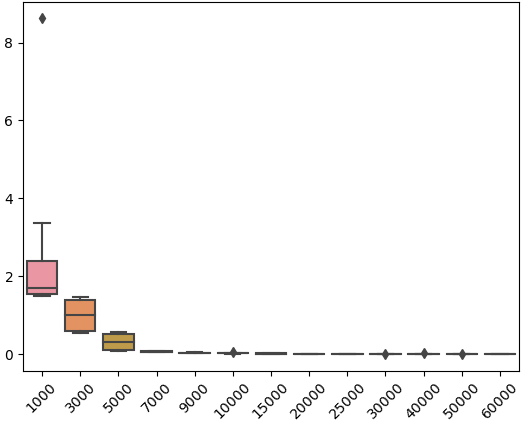}
  \caption{}
  \label{fig:MNIST_MMD_Gauss}
\end{subfigure} 
\caption{Latent (a) JS and (b) MMD loss for MNIST data set with Gaussian targets.}
\label{fig:MNIST_Gauss_latent}
\end{figure}

\section{Reconstruction Consistency} \label{Recon_Conc}
With a clearer understanding of WAEs' performance towards meeting the constraint it was formulated under, we move on to its main objective. If we position ourselves along the flow of information--- first through the encoder and now at the footsteps of the decoder--- we have ourselves a density estimation task in higher dimensions. This is typical of inverse models and the underlying goal is to utilize the encoded information at one's disposal to reach as close as to $\hat{\mu}_{n}$. Needless to say, it comes as a consequence of the error $W_{c_{x}}^{1}(\hat{\mu}_{n}, (D \circ E^{*}_{n}(t))_{\#}\hat{\mu}_{n})$ being minimized, given the optimal encoder $E^{*}_{n}(t)$ incurring latent loss $\leq t$.

In spirit, the role of the decoder is somewhat similar to a generative map. The aspects in which they differ from those in a GAN architecture are mainly twofold. Firstly there is no dynamic critic in the form of a discriminator to guide its learning. The role is taken up by $\mathcal{L}_{c_{x}}^1$ only. However, on the upside, while the latent distribution in GANs is non-informative, to begin with, WAEs have latent laws with input information `preserved'. During decoding, this very information needs to be used with utmost efficacy. Since the resemblance between spaces of different dimensions in a sample problem lies in the local geometry based on pairwise distances between samples, \citet{chakrabarty2021statistical} identify quasi-isometries as ideal decoders. The definition provided in their article however indicates a bi-Lipschitz characterization. In general, $\mathbb{R}^k$ is not quasi-isometric to $\mathbb{R}^d$, $d > k$. However, there may exist bi-Lipschitz maps from $\Omega_{z} \rightarrow \mathbb{R}^l$, $l \leq d$, which can thus be used to form outer extensions mapping $\mathbb{R}^{k} \rightarrow \mathbb{R}^{d}$ \citep{mahabadi2018nonlinear}. Such extensions typically preserve distortions up to constant factors and as a result, do not depreciate the asymptotic behavior of estimates post-translation. Another technique of ensuring bi-Lipschitzness is to search for \textit{regular}\footnote{A Lipschitz map $f:(\mathcal{Z},c_{z}) \rightarrow (\mathcal{X},c_{x})$ is said to be regular if there exists a constant $C>0$ such that given any ball $B$ in $\mathcal{Z}$, $f^{-1}(B)$ can be covered by at most $C$ balls, each of radius $C.\textrm{rad}(B)$ \citep{david2000regular}.} maps $f:\mathbb{R}^{k} \rightarrow \mathbb{R}^{d}$. These in turn make the restriction of $f$ to $\overline{\Omega}_{z}$ (closure) to be BL. While bi-Lipschitz transforms acting as decoders automatically enforce non-degeneracy in reconstructed signals, to obtain upper bounds of associated errors it is sufficient to have Lipschitz decoders. Most theoretical studies on GANs tend to impose this restriction on generators. The existence of such a benchmark Lipschitz transform, however, is readily guaranteed if encoders are considered according to Lemma \ref{JL}. This also supports the practical convention of building the decoder as exactly the inverse of $E^{*}(t)$.

We, on the contrary, prescribe constructing a decoder map that utilizes the information encapsulated in the latent law wholly and efficiently. To demonstrate the notion, let us fragment the reconstruction loss, given an encoder $E$, as follows
\begin{equation} \label{decoding}
    W_{c_{x}}^{1}(\mu, (D \circ E)_{\#}\hat{\mu}_{n}) \leq W_{c_{x}}^{1}((D \circ E)_{\#}\hat{\mu}_{n}, D_{\#}\rho) + \underbrace{W_{c_{x}}^{1}(D_{\#}\rho, \hat{\mu}_{n})}_\text{Decoder Translation Error} + W_{c_{x}}^{1}(\hat{\mu}_{n}, \mu).
\end{equation}

Observe that the first term on the right is essentially the propagated disagreement at the latent space. If the metric, measuring the discrepancy follows a data-processing inequality, we can expect a non-asymptotic upper bound of the same order as that obtained on the latent error. As such, $D$ must be constructed keeping in mind the sole aim of minimizing the semi-discrete translation error. The following lemma provides the backbone of the construction.

\begin{lemma}[\citet{yang2022capacity}] \label{yang}
    Let $\nu$ be an univariate absolutely continuous distribution and $\hat{\mu}_{n} \in \mu^{\otimes n}$. Given $W \geq 7d + 1$ and $L \geq 2$, there exists a NN transform based on ReLU activation $\phi^{'} \in \Phi(W,L)_{1}^{d}$ such that $\forall \varepsilon > 0$
    \begin{equation*}
        W_{c_{x} \equiv L_{1}}^{1}(\hat{\mu}_{n}, \phi^{'}_{\#}\nu) \leq \varepsilon,
    \end{equation*}
    whenever $n \leq \frac{W-d-1}{2}\floor{\frac{W-d-1}{6d}}\floor{\frac{L}{2}} + 2$.
\end{lemma}

The transport hinted in the lemma is essentially a piecewise linear map, Lipschitz continuous on bounded balls. The restriction on $n$ indicates the number of breakpoints. Since the result holds for all probability measures $\nu$ having densities, the only modification required in our case is projecting $\rho$ onto $\mathbb{R}$, using linear maps beforehand. Given such a linear map $D_{0}: \Omega_{z} \rightarrow \mathbb{R}$, and $\phi^{'}$ according to lemma \ref{yang}, the desired decoder is given by $D = \phi^{'} \circ D_{0}$. This operation also preserves the Lipschitz continuity in the resultant decoder. Since there is no unique way of selecting the linear transform, one may use instead a pooled distribution. As such, $D = \phi^{'} \circ \sum_{i=1}^{N_{3}}D_{i} \equiv \sum_{i=1}^{N_{3}}\phi^{'} \circ D_{i}$, where $D_{1},\cdots,D_{N_{3}}$ are individual linear maps aimed at preserving different aspects of $\rho$. This is especially useful when $\abs{\Sigma} > 1$. In this case also $D$ turns out to be Lipschitz (since the property itself stems from individual components and summation--- that too without scaling--- only changes the associated constant). The following result formalizes our discussion.

\begin{theorem}[Reconstruction Consistency in a Latent-Consistent WAE] \label{Recon_con}
    Given a margin of latent error $t > 0$, let $E^{*}_{n}(t)$ be an optimal encoder satisfying latent consistency under the metric $d_{\textrm{TV}}$. Then, there exists a decoder $D \in \Phi(W,L)_{k}^{d}$, $d \geq 3$ based on ReLU activations, with $W \geq 7d + 1$ and $L \geq 3$ such that
    \begin{align*}
        \mathbb{E}\left[W_{c_{x}}^{1}(\mu, (D \circ E^{*}_{n}(t))_{\#}\hat{\mu}_{n})\right] - \mathcal{O}(t) \lesssim n^{-\frac{1}{d}},
    \end{align*}
    where $n = \mathcal{O}(\frac{W^{2}L}{d})$.
\end{theorem}

The theorem reveals the extent to which the realized latent loss can potentially amplify during reconstruction. The corresponding excess error always stays $\mathcal{O}(n^{-\frac{1}{d \vee 2}})$, with high probability (using McDiarmid's inequality). The result also allows for the formulation of WAEs to be made general by replacing $W_{c_{x}}^{1}$ with $W_{c_{x}}^{p}$, $1 \leq p$. In such a case, the observation: $W^{p}(\mu_{1}, \mu_{2}) \leq B^{\frac{p-1}{p}}_{x} W^{1}(\mu_{1}, \mu_{2})^{\frac{1}{p}}$, $\mu_{1}, \mu_{2} \in \mathcal{P}(\mathcal{X})$ coupled with Theorem \ref{Recon_con} provides the regeneration guarantee.

\begin{remark}[Reconstruction as a Consequence of Latent Consistency in WAE-MMD] \label{rem_MMD}
    Considering a latent error $d_{\textrm{TV}}\left(E^{*}_{n}(t)_{\#}\hat{\mu}_{n}, \rho\right) \leq t$ unites both WAE models (WAE-GAN and WAE-MMD) in showing reconstruction consistency based on the fact that TV is a natural upper bound to both JS and MMD. However, if the bounded kernel $\kappa$ applied in a latent space (in WAE-MMDs) is integrally strictly positive definite and follows strong invariance, we have the partial inequality--- given Assumption \ref{assumption1} and \ref{assumption2}--- as follows
    \begin{align*}
        W_{c_{x}}^{1}((D \circ E)_{\#}\hat{\mu}_{n}, D_{\#}\rho) &\lesssim W_{c_{z}}^{1}(E_{\#}\hat{\mu}_{n}, \rho) \\ &\leq C_{\varepsilon} \:d_{\mathcal{H}_{\kappa}}\left(E_{\#}\hat{\mu}_{n}, \rho\right) + \varepsilon,
    \end{align*}
    for some $C_{\varepsilon} > 0$ and $\forall \varepsilon > 0$ (\citet{modeste2022characterization}, Proposition 3.9). The first inequality is due to the Lipschitz continuity of $D$. Now, if MMD, equipped with the same $\kappa$ is applied on the input space as well, the same $D$ constructed so far satisfies $d_{\mathcal{H}_{\kappa}}\left(D_{\#}\rho, \hat{\mu}_{n}\right) < \varepsilon$, $\forall \varepsilon > 0$ (\citet{yang2022capacity}, Lemma 3.3). As a result, the right-hand side of inequality (\ref{decoding}) can be written entirely in terms of MMD. Hence, Theorem \ref{latent_con_MMD} can be readily plugged in to obtain a deterministic upper bound to the realized reconstruction error.
\end{remark}

\subsection{Simulations} \label{sim_Recon}

We continue with the earlier experimental setup to provide empirical validation. In fact, reconstruction outputs are obtained simultaneously with latent results. Conforming to our previous prescription, we employ $4$-deep decoders for the Five Gaussian data set. On the other hand, to reconstruct observations corresponding to MNIST, we take a pragmatic approach while choosing network widths ($k=$64-128-256-512). The final output tensor is suitably reshaped to have the size of an image (batch size, 1, 28, 28).

\begin{figure}[ht]
    \centering 
\begin{subfigure}{0.32\textwidth}
  \includegraphics[width=\linewidth]{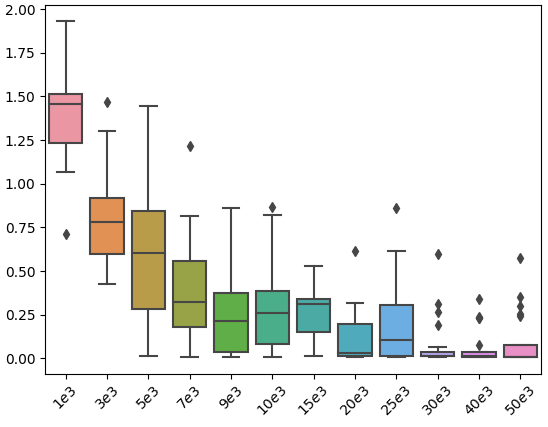}
  \caption{Gaussian}
  \label{fig:MMD_recon_Gauss_Relu}
\end{subfigure}\hspace{3pt} 
\begin{subfigure}{0.32\textwidth}
  \includegraphics[width=\linewidth]{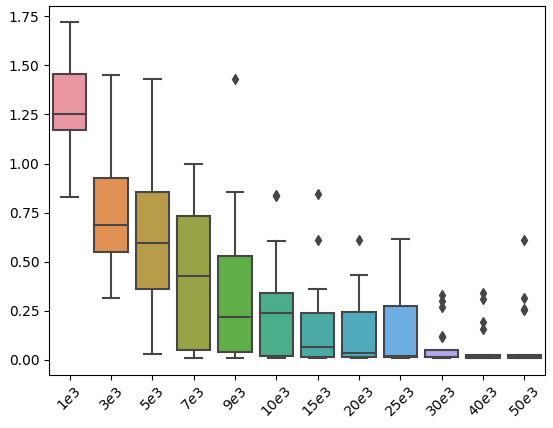}
  \caption{Beta copula}
  \label{fig:MMD_recon_GBeta_Relu}
\end{subfigure}\hspace{3pt} 
\begin{subfigure}{0.32\textwidth}
  \includegraphics[width=\linewidth]{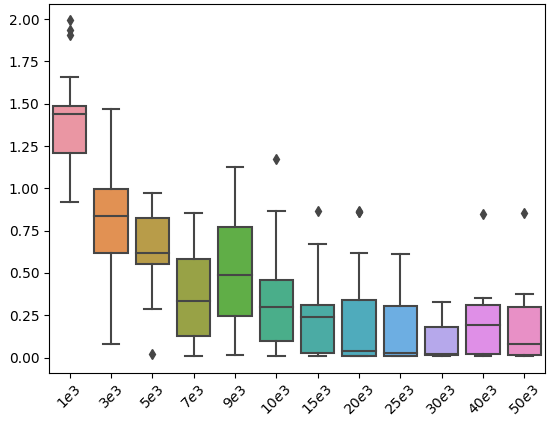}
  \caption{Exponential copula}
  \label{fig:MMD_recon_Exp_Relu}
\end{subfigure}\hfil 
\caption{Wasserstein reconstruction loss for the three latent distributions under MMD using ReLU encoders. The penalization handed to the latent loss is kept at $\lambda = 0.2$.}
\label{fig:MMD_GB_recon}
\end{figure}

The diminishing trait of error values with shrinking variance is evident in Fig \ref{fig:MMD_GB_recon}. The limiting margin of error where the sequence converges (for all latent distributions under consideration) remains well below the tolerable latent loss. This further attests to our theoretical bound. The optimizations not only make the error eventually vanish, but also produce perceptually alike samples [Fig \ref{fig:JS_GB_recon_photo}]. Reconstruction errors corresponding to Five Gaussian replicates in a WAE-MMD setup also tend to follow a convergence rate $\mathcal{O}(n^{-\frac{1}{2}})$ [Fig \ref{fig:GB_recon_MMD_Group}]. This is a validation to the Remark \ref{rem_MMD}, even under the deployment of GroupSort encoders. Reconstructions of MNIST also result in photo-realistic copies of the input law [Fig \ref{fig:MNIST_Gauss_recon}]. The corresponding errors exhibit sharply decaying behavior under both WAE architectures. 

\begin{figure}[ht]
     \centering 
\begin{subfigure}{0.34\textwidth}
  \includegraphics[width=\linewidth]{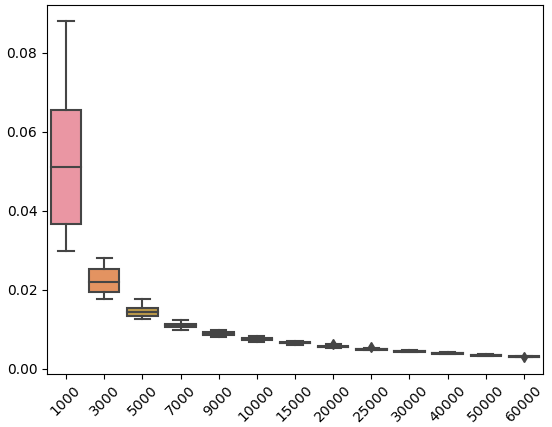}
  \caption{}
  \label{fig:MNIST_recon_JS_Gauss}
\end{subfigure} 
\hspace{3pt}
\begin{subfigure}{0.34\textwidth}
  \includegraphics[width=\linewidth]{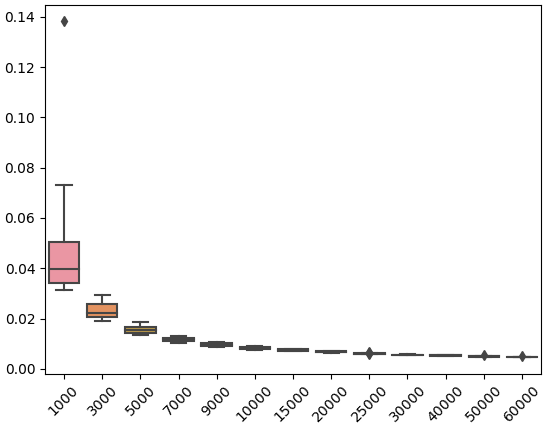}
  \caption{}
  \label{fig:MNIST_recon_MMD_Gauss}
\end{subfigure}\hfill 
\begin{subfigure}{0.9\textwidth}
  \includegraphics[width=\linewidth]{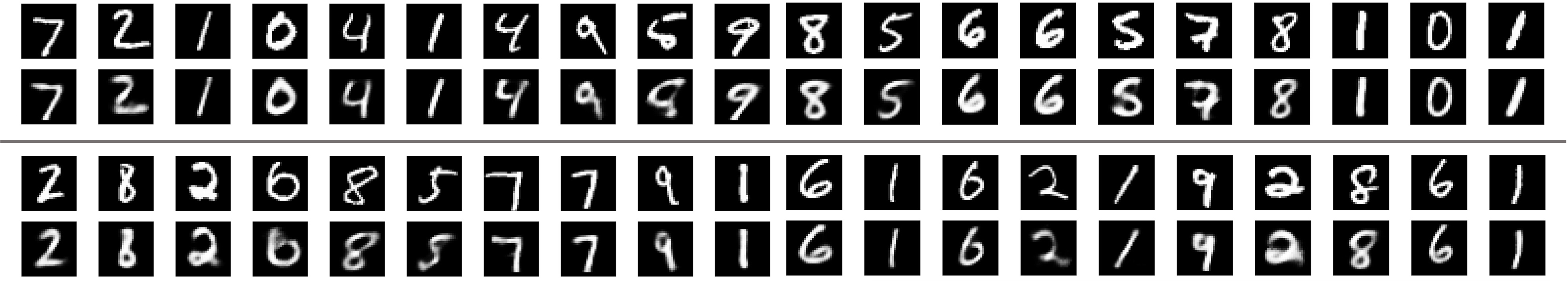}
  \caption{\centering{Reconstruction under latent JS loss}}
  \label{fig:MNIST_recon_Js_Gauss_sample}
\end{subfigure}\hfil 
\caption{MNIST reconstruction error given Gaussian latent laws under (a) JS and (b) MMD latent loss, using ReLU encoders. In (c), the odd rows hold the input digits and the even ones are their reconstructed counterparts.}
\label{fig:MNIST_Gauss_recon}
\end{figure}

\begin{figure}[ht]
    \centering 
\begin{subfigure}{0.32\textwidth}
  \includegraphics[width=\linewidth]{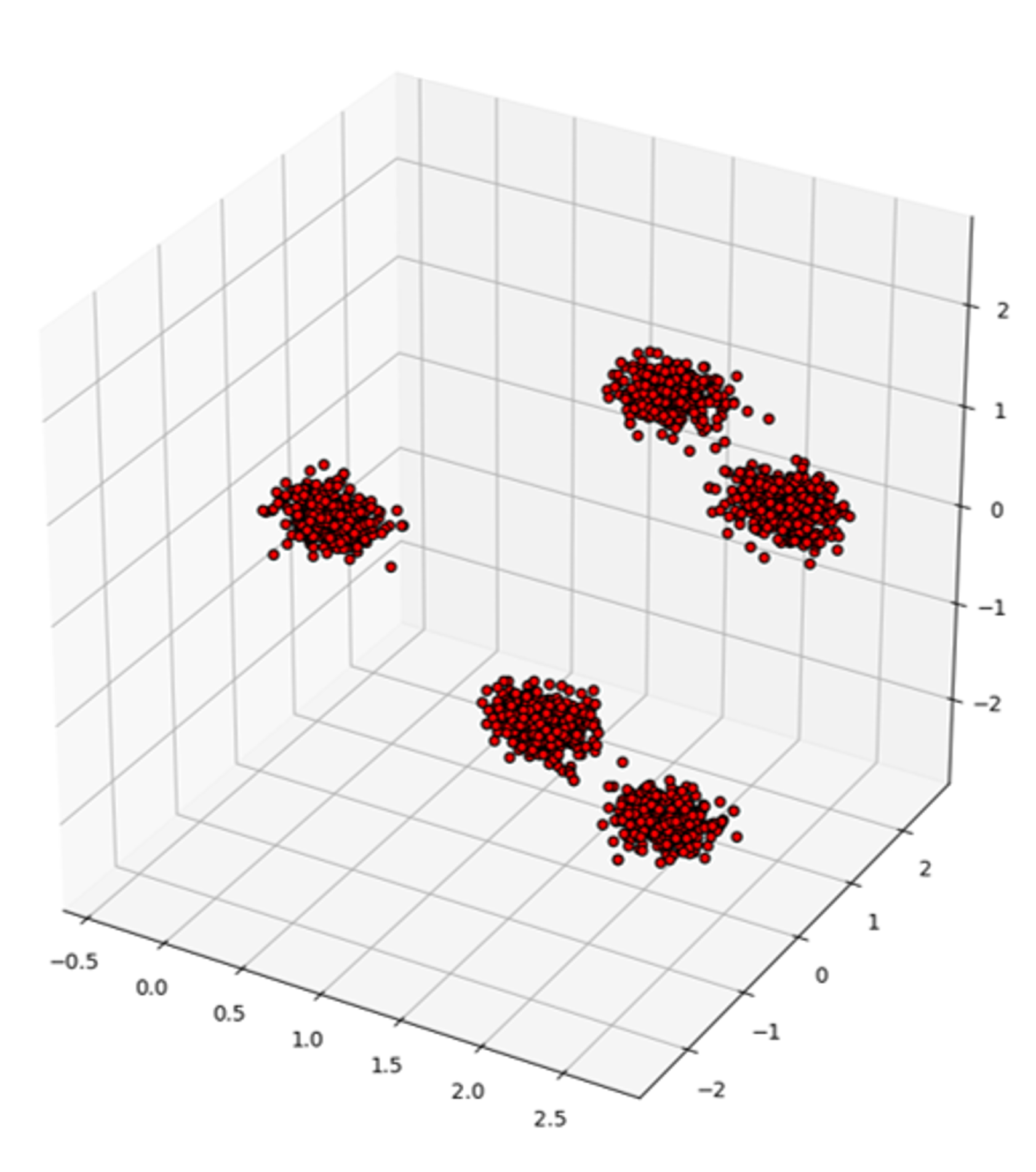}
  \caption{Input data}
  \label{fig:Actual_GB}
\end{subfigure}\hspace{3pt} 
\begin{subfigure}{0.32\textwidth}
  \includegraphics[width=\linewidth]{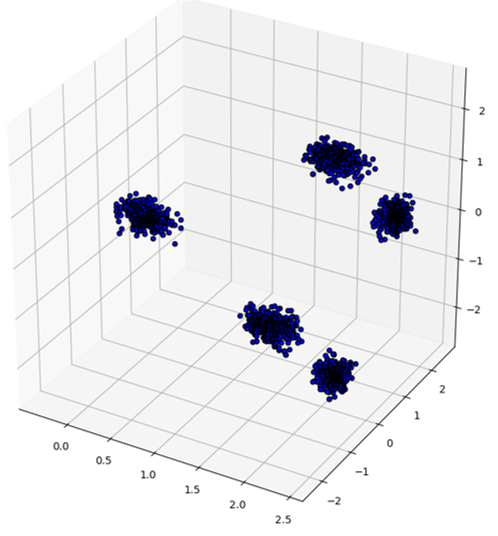}
  \caption{Gaussian}
  \label{fig:JS_recon_GB_Gauss}
\end{subfigure}\hspace{3pt} 
\begin{subfigure}{0.32\textwidth}
  \includegraphics[width=\linewidth]{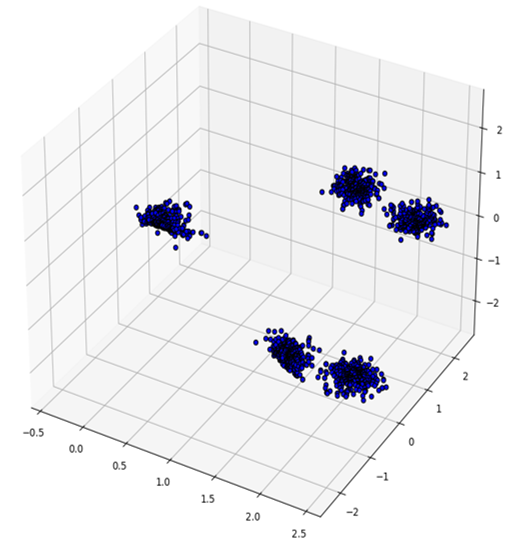}
  \caption{Beta copula}
  \label{fig:JS_recon_GB_Beta}
\end{subfigure}\hfil 
\caption{Reconstructed samples ($n=10,000$) from Five Gaussian dataset under JS latent loss for given latent distributions, using ReLU encoders after $1500$ epochs.}
\label{fig:JS_GB_recon_photo}
\end{figure}

\section{Robustness to Distribution Shift} \label{Robust}

The quality of deep generative model outputs is often marred by contamination in data. Images, and consequently WAEs are very much susceptible to such adversarial corruption. In our recommendation, we have prioritized the free flow of information through the model networks in a WAE. As such, a corrupted set of input samples runs the risk of spoiling all the downstream tasks. Thus, to get a comprehensive look at the machinery of WAE, one must test its innate capability to preserve regeneration quality under the influence of adversaries. Some of the well-recognized models in robust statistics include \textit{Adaptive}, \textit{Oblivious} and the \textit{Huber contamination model} \citep{chen2016robust, zhu2022generalized}. In Huber contamination, instead of assuming independent replicates from the input density $p_{\mu}$, it is assumed that there lies a probability $\epsilon >0$ that the sample comes from a \textit{contamination} $p_{c} \in \mathcal{P}(\mathcal{X})$. The density $p_{c}$ is independent of the input law and remains unknown. As such, input observations
\begin{equation} \label{corr}
    X_{1}, \cdots , X_{n} \sim \Tilde{p}:=(1-\epsilon)p_{\mu} + \epsilon p_{c}
\end{equation}
are what we have at hand. Under such a setup, \citet{liu2019robust} showed that given $p_{\mu} \in \mathcal{C}^{s}_{L}(\Omega_{x})$, kernel density estimates incur euclidean losses $\mathcal{O}(n^{-\frac{2s}{2s+1}} \vee \epsilon^{\frac{2s}{s+1}})$. The underlying kernels are considered to be square-integrable and bounded, with centered moments. This result is particularly motivating since given that $D \circ E \approx \textrm{id}$ a.e., it gives us a bound on the regeneration error for VAEs. 

The regime we consider in our following discussion is closer to oblivious contamination. We assume that the \textit{contaminated input distribution} $\Tilde{p}$ is such that $\mathbb{E}_{X \sim \Tilde{p}, Y \sim p_{\mu}} c_{x}(X,Y) \leq \epsilon$, a more general notion compared to Huber. Observe that, such a criterion automatically implies $1$-Wasserstein contamination under metric $c_{x}$ \citep{liu2022wrob}. No additional assumption on the regularity of $\Tilde{p}$ is assumed. The goal lies the same as before: reconstructing $p_{\mu}$ based on an input estimator. In other words, in the absence of additional regularization, we check for the extent of inherent distributional robustness WAEs possess. 

We have already seen that both the encoder and decoder, under careful construction can follow Lipschitz continuity. As a result, their composition behaves similarly. To generalize such composite maps, in this section, we consider maps $\mathcal{G} : \mathcal{X} \rightarrow \mathcal{X}$ which induce group actions. Now, if $G \in \mathcal{G}$ satisfies information preservation, it is approximately equivalent to constructing an estimator based on translated observations rather than translating a pre-constructed estimator to approach $p_{\mu}$. As such, given replicates $X_{1}, \cdots, X_{n} \sim \Tilde{p}$, we essentially need to look for upper bounds to the loss $W_{c_{x}}^{1}(\hat{p}_{n}, p_{\mu}) \lesssim \norm{\hat{p}_{n} - p_{\mu}}_{1}$, where $\hat{p}_{n} \in L_{1}(\mathbb{R}^{d})$ is based on $\{G(X_{i})\}^{n}_{i=1}$.

To cope with the adversary, first, we modify the properties of the regularly invariant kernels we utilized earlier. We call a kernel $\kappa(x,y): \mathcal{X} \times \mathcal{X} \rightarrow \mathbb{R}$ \textit{transformation invariant} if given action $G \in \mathcal{G}$, $\kappa(G(x),y) = \kappa(x,G^{-1}(y))$ \citep{liu2022learning}. The following theorem provides the extent of WAEs' resilience based on such kernel estimates.

\begin{theorem}[Reconstruction Consistency under Contamination] \label{conta}
    Let the contaminated distribution $\Tilde{p}$ be such that $\sup_{\Omega_{x}} \Tilde{p}(x) < \infty$. Also, let the kernel $\kappa$ be regular, translation invariant with respect to $L_{1}$ (\ref{Reg_inv_ker}) and transformation invariant which satisfies $\int_{\mathcal{X}}\kappa^{2}(v,v-u) du < \infty$. Then, given any $G \in \mathcal{G}$, a kernel density estimate $\hat{p}_{h} \equiv \hat{p}_{h,n}$ based on $\kappa$ satisfies
    \begin{equation*}
        \mathbb{E}\abs{\hat{p}_{h}(0) - p_{\mu}(0)} \lesssim n^{-\frac{m_{x}}{d + 2m_{x}}} \vee \epsilon^{\frac{m_{x}}{2d+m_{x}}}.
    \end{equation*}
\end{theorem}

\subsection{Simulations} \label{sim_rob}

During empirical validation, based on the diverse natures of their support, the thickness of tails, and central tendencies, we select three potential contaminating distributions: Standard Gaussian, Cauchy, and Dirichlet. While the Five Gaussian data set is corrupted with the latter two, we apply all three on MNIST. For utmost rigor in our experiments, we adopt an elaborate contaminating regime. We not only vary the proportion of observations getting corrupted but also regulate the extent of it. For example, there may be a set of input observations $\{X_{i}\}_{i=1}^{n}$, out of whom $\floor{\frac{n}{2}}$ are replaced by $ (0.8)X_{i} + (1-0.8)Y_{i}$, where $\{Y_{i}\}^{i \in \mathscr{C}}$ are replicates from a contaminating law. $\mathscr{C}$ denotes the indexes receiving the corruption, $\abs{\mathscr{C}} = \floor{\frac{n}{2}}$. We refer to the mixing proportion as \textit{level} ($\alpha$). This regime generalizes the entire contamination landscape in statistics. The specific choice of parameters for Dirichlet is taken as $(5,3,5)$. Perhaps the most interesting observation from our huge body of experiments is some of the near-accurate reconstructions.   

\begin{figure}[ht]
    \centering 
\begin{subfigure}{0.32\textwidth}
  \includegraphics[width=\linewidth]{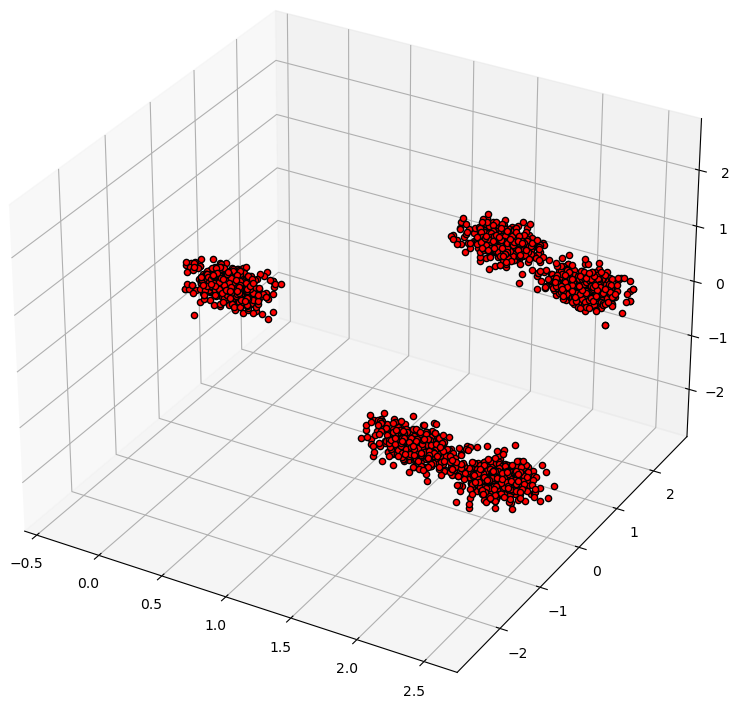}
  \caption{Actual Data}
  \label{fig:Target_FG_Dir_JS}
\end{subfigure}\hspace{3pt} 
\begin{subfigure}{0.32\textwidth}
  \includegraphics[width=\linewidth]{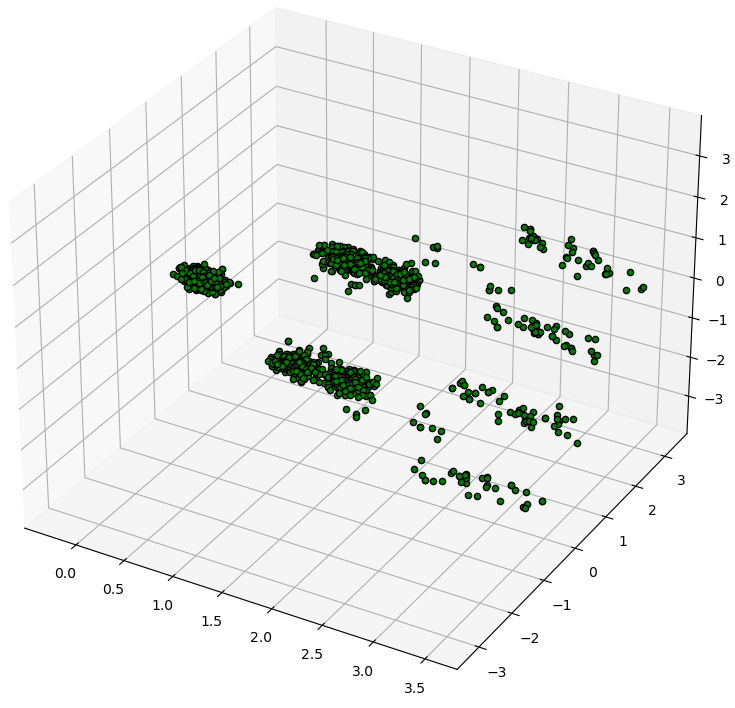}
  \caption{Contaminated Data}
  \label{fig:Cont_FG_Dir_JS}
\end{subfigure}\hspace{3pt} 
\begin{subfigure}{0.32\textwidth}
  \includegraphics[width=\linewidth]{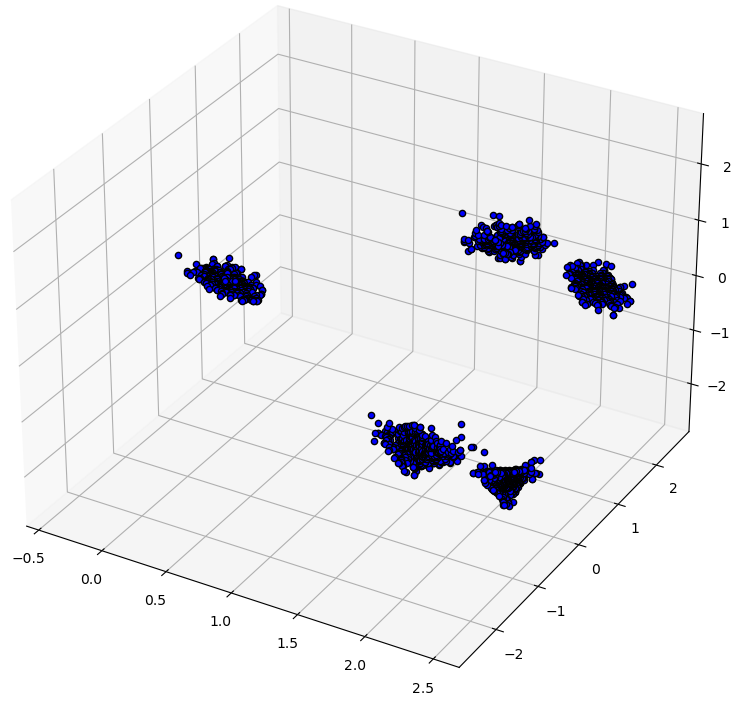}
  \caption{Reconstructed Data}
  \label{fig:Recon_FG_Dir_JS}
\end{subfigure}\hfil 
\caption{Reconstructed samples ($n=10,000$) from Five Gaussian dataset with half the observations contaminated at level $0.2$, under JS latent loss. The corrupting distribution is taken to be Dirichlet($5,3,5$).}
\label{fig:JS_Robust_FG}
\end{figure}

In the case of the MNIST data set, even at a significant level of contamination, the reconstruction errors continue to converge to a near-zero value. The corresponding reconstructed samples are of remarkably sound resolution, given the regenerative capability of WAEs in general [Fig \ref{fig:MNIST_robust_MMD_plot}]. We also study the effect of varying $\alpha$ on reconstructed image quality [Fig \ref{fig:diffusion_robust}], which hints at tolerable levels of contamination.

\begin{figure}[ht]
     \centering 
\begin{subfigure}{0.32\textwidth}
  \includegraphics[width=\linewidth]{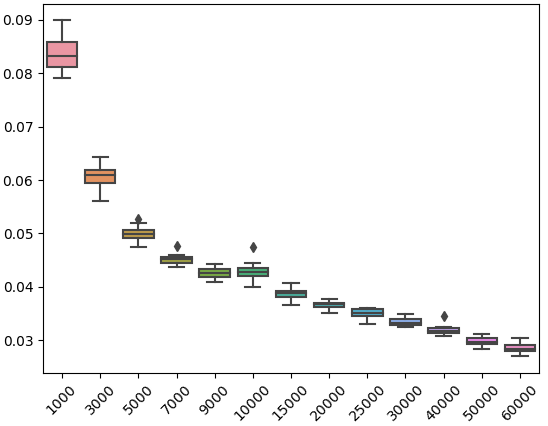}
  \caption{Gaussian}
  \label{fig:MNIST_robust_Gauss_MMD}
\end{subfigure} 
\begin{subfigure}{0.32\textwidth}
  \includegraphics[width=\linewidth]{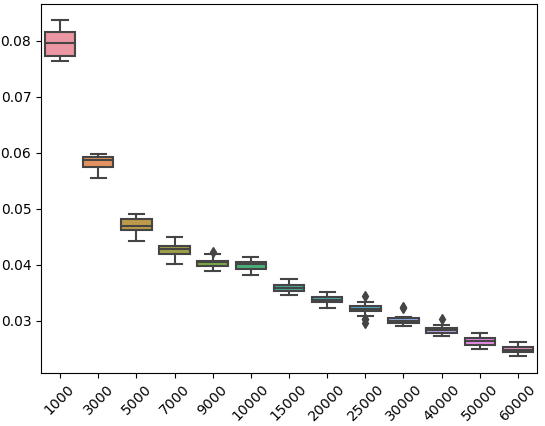}
  \caption{Cauchy}
  \label{fig:MNIST_robust_Cauchy_MMD}
\end{subfigure} 
\begin{subfigure}{0.32\textwidth}
  \includegraphics[width=\linewidth]{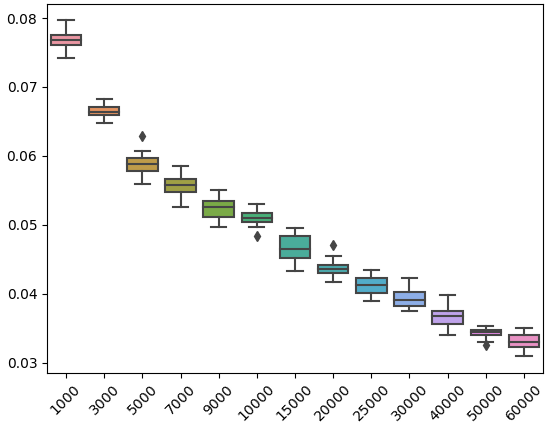}
  \caption{Dirichlet}
  \label{fig:MNIST_robust_Dir_MMD}
\end{subfigure}\hfill 
\begin{subfigure}{\textwidth}
  \includegraphics[width=\linewidth]{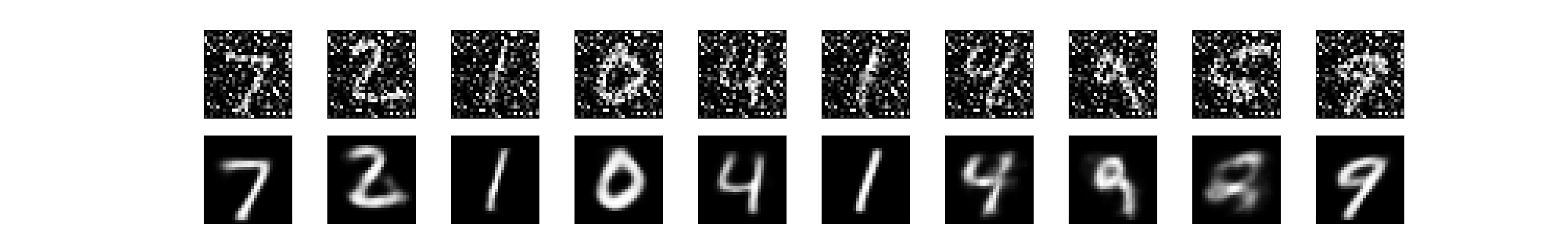}
  \caption{\centering{Reconstruction under Cauchy contamination}}
  \label{fig:MNIST_robust_recon_MMD}
\end{subfigure}\hfil 
\caption{Reconstruction errors incurred by a ReLU-induced WAE-MMD for MNIST, under different contaminating distributions at level $0.2$. In all the experiments, the latent distribution is kept standard Gaussian. In (d), the first row represents contaminated samples (standard Cauchy at level $0.2$) and the second row contains their reconstructed counterparts.}
\label{fig:MNIST_robust_MMD_plot}
\end{figure}

\section{Discussion and Future Work}

In this paper, we provide statistical guarantees regarding the concurrent tasks Wasserstein autoencoders carry out, i.e. achieving both the latent and reconstruction benchmark laws. Under probabilistic characterization of the input data, we establish deterministic upper bounds to both losses. Our non-parametric estimation approach caters to both WAE architectures, namely WAE-MMD and WAE-GAN. In the process, we find out sufficient properties an encoder must possess to become information preserving. This particular notion further enables us to prescribe to a practitioner building an encoder, the architectural specifications of an ideal neural network. Deployment of such a network, in turn, aids the latter process of reconstruction. In a WAE-MMD framework, we explore the sufficient conditions the deployed kernel needs to satisfy to ensure latent consistency. We put to test our theoretical findings in simulations based on real and synthetic data sets. The phenomenon of information preservation in the latent space is fascinating to witness in the flesh. The encoded distributions tend to maximize their alignment with the latent target without losing the local geometric properties of the input. Similarly, we recommend decoder frameworks that achieve near-perfect regenerations. Such results are also substantiated by accompanying numerical experiments that show sharply decaying losses over varying sample sizes. Finally, in a density estimation setup, we test the degree of robustness WAEs hold naturally, against distribution shifts without additional regularization.      

We dedicate the rest of the section to pointing out possibilities our theoretical framework spawns. The first question arises regarding the modes of the distributions involved. Regenerated samples using a WAE are not known to be plagued by `mode collapse' as severely as a vanilla GAN output. However, it is not uncommon for real data distributions ($\mu$) to have non-convex support. In such a case, the OT (due to Brenier) map between the latent distribution and $\mu$ will mostly be discontinuous. Moreover, NN-based transforms often fail to universally approximate such discontinuous functions. This may lead to significant mismatches between the supports of $\mu$ and $D_{\#}\rho$, however good the representative samples may be. As such, there always lies an innate possibility of missing out on modes of $\mu$, even if the effect is benign to the eye. The question also involves the role of underlying divergences. In generative modeling, they are typically judged based on their capacity of \textit{mode covering} and \textit{mode seeking} \citep{Li2023mode}. A mode-covering divergence tends to prioritize the spread of masses to all target modes and as a result, generates out-of-sample observations. While mode-seeking distances avoid doing so, their conservative mass assignment leads to the model missing out on one or several modes. Unfortunately, both TV (weakly mode-seeking) and JS (uniformly mode-seeking) lean towards the later characterization. As such, in case $\rho$ is multi-modal and there exists a mismatch between the number of modes of $\mu$ (unknown) and $\rho$, WAEs operate under a heightened risk of losing information on some modes. In our non-parametric setup, we do not specify the modality of input and latent distributions. This creates an interesting prospect to study the effect of varying numbers of modes in $\rho$ on information preservation and reconstruction. Future work may also look into the tolerable modality of input laws in a WAE-MMD before mode collapse transcends benignity. This is particularly intriguing since the mode-seeking properties of MMD remain unexplored.  

Another question that remains closely related to that regarding modes, is the uniformity of samples over them. The model may also fail to recognize a mode of $\mu$ if observations from it are vastly outnumbered. We have already seen the denoising capability of WAEs. If the input data suffer such \textit{imbalance}, it is not improbable for the model to treat them as outliers, especially if the modes are well-separated. As such, along with the modality, formulating appropriate estimators that ensure proportionate participation and accurate reconstruction under imbalance must be taken up further.     


\section*{Appendix: Technical Proofs} \label{appA}

\begin{proof}[Proof of Lemma \ref{lemm_inv}]
    Let $E:\mathcal{X} \rightarrow \mathcal{Z}$ and $D:\mathcal{Z} \rightarrow \mathcal{X}$ be measurable maps--- namely a encoder-decoder pair--- that satisfy
    \begin{equation} \label{lossless}
        W_{c_{x}}^{1}(\mu, (D \circ E)_{\#}\mu) + \lambda.\Omega(E_{\#}\mu, \rho) = 0,
    \end{equation}
    given $\lambda > 0$. Since $\Omega$ is a divergence metric metrizing the underlying class of probability distributions, it is implied that $E_{\#}\mu = \rho$. As such, we observe a lossless encoding in the population sense. Now,
    \begin{align*}
        E_{\#}\mu = \rho \implies (D \circ E)_{\#}\mu = D_{\#}\rho \overset{(1)}{\implies} \mu = D_{\#}\rho,
    \end{align*}
    where ($1$) is due to (\ref{lossless}). This hints towards an absolute information preservation in reconstruction based on the fact that $D \circ E = \textrm{id}_{X}$ a.e.

    Given a probability space automorphism $\varphi$, $\forall A \in \mathcal{Z}$
    \begin{equation*}
        \rho(\varphi(A)) = \rho(\varphi^{-1}(\varphi(A))) = \rho(A),
    \end{equation*}
    since $\varphi^{-1}$ also becomes an automorphism. Hence,
    \begin{align*}
        &W_{c_{x}}^{1}(\mu, (D \circ \varphi) \circ (\varphi^{-1} \circ E)_{\#}\mu) + \lambda.\Omega((\varphi^{-1} \circ E)_{\#}\mu, \rho) \\ =& \; W_{c_{x}}^{1}(\mu, (D \circ (\varphi \circ \varphi^{-1}) \circ E)_{\#}\mu) + \lambda.\Omega(\varphi^{-1}_{\#} (E_{\#}\mu), \rho) \\ =& \; W_{c_{x}}^{1}(\mu, (D \circ E)_{\#}\mu) + \lambda.\Omega(\varphi^{-1}_{\#} \rho, \rho) = 0.
    \end{align*}
    As such, the encoder-decoder pair $(\varphi^{-1} \circ E, D \circ \varphi)$ is also a zero-solution of the population loss function. Also, in case $\varphi \neq \textrm{id}$, the pair clearly differs from $(E,D)$ a.e.
\end{proof}

\begin{proof}[Proof of Theorem \ref{theo: lip}]
    Given a transform $g \in \mathscr{F}_{L}(\mathcal{X}, \mathcal{Z})$, the information dissipated can be decomposed using the triangle inequality as follows
    \begin{align}
        d_{\mathcal{H}}\left(g_{\#}\Tilde{\mu}_{n}, \widehat{(g_{\#}\mu)}_m\right) \leq d_{\mathcal{H}}\left(g_{\#}\Tilde{\mu}_{n}, g_{\#}\mu\right) + d_{\mathcal{H}}\left(g_{\#}\mu, \widehat{(g_{\#}\mu)}_m\right). \label{main_ineq}
    \end{align}
    Here, $\Tilde{\mu}_{n}$ denotes the RIK estimator (see Definition \ref{Reg_inv_ker}), defined as $\frac{d\Tilde{\mu}_{n}}{dx} = \frac{1}{nh^d}\sum^{n}_{i=1} \kappa\left(\frac{x}{h}, \frac{x_i}{h}\right) = \hat{p}_{h}(x)$, $x \in \Omega_x$ where $h$ is the bandwidth. Also, since the kernels are bounded, there exists $B > 0$ such that $\sup_{\Omega_{x}}\kappa(\cdot, \cdot) = B \leq 1$. In most cases, with choices of kernels being distributions themselves, the modal values tend to satisfy this criterion. Since members of $\mathcal{H}$ are bounded, total variation turns out to be a natural upper bound for the associated loss. In particular,
    \begin{align}
        d_{\mathcal{H}}\left(g_{\#}\Tilde{\mu}_{n}, g_{\#}\mu\right) & = \sup_{h \in \mathcal{H}} \int h(z) \:d(g_{\#}\Tilde{\mu}_{n} - g_{\#}\mu) \nonumber \\ & = L \sup_{h' \in \frac{1}{L}\mathcal{H} \circ g} \int h^{'}(x) \:d(\Tilde{\mu}_{n} - \mu) \nonumber \\ & \leq \frac{L}{2}\int \abs{\hat{p}_{h}(x) - p_{\mu}(x)} dx \\ & \leq \frac{L}{2}\left\{\int \abs{\hat{p}_{h}(x) - \mathbb{E}[\hat{p}_{h}(x)]} dx + \norm{\mathbb{E}[\hat{p}_{h}(x)] - p_{\mu}}_{1}\right\} \label{part1},
    \end{align}
    where $\mathcal{H} \circ g = \{h \circ g : h \in \mathcal{H}\}$ such that $\norm{h \circ g}_{\infty} = \max_{x} \abs{h(g(x))} < \infty$. The first inequality is obtained by taking the supremum over all bounded real-valued functions on $\mathcal{X}$. As such, it is sufficient to find the concentration of $\hat{p}_{h}$ around $p_{\mu}$ under the essential supremum norm. 

    The bias term under the norm satisfies
    \begin{align}
        \mathbb{E}[\hat{p}_{h}(x)] - p_{\mu}(x) &= \frac{1}{h^d} \int \kappa\left(\frac{x}{h}, \frac{y}{h}\right) p_{\mu}(y) dy - p_{\mu}(x) \nonumber \\ & = \int \kappa\left(\frac{x}{h}, \frac{x}{h} - u\right) [p_{\mu}(x - hu) - p_{\mu}(x)] du \label{bias1} \\ &= \int \kappa\left(\frac{x}{h}, \frac{x}{h} - u\right) \sum_{\abs{\alpha} \leq m_{x}-1} \frac{D^{\alpha} p_{\mu}(x)}{\alpha !} (-uh)^{\alpha} du \label{bias2} \\ & \quad + m_{x} \int \kappa\left(\frac{x}{h}, \frac{x}{h} - u\right) \sum_{\abs{\alpha} = m_{x}} \frac{(-uh)^{\alpha}}{\alpha!} \int^{1}_{0} (1-t)^{m_{x}-1} D^{\alpha} p_{\mu}(x - tuh) dt \nonumber \\ & = \int \int^{1}_{0} \kappa\left(\frac{x}{h}, \frac{x}{h} - u\right) (-u)^{m_{x}} h^{m_{x}} \frac{(1-t)^{m_{x}-1}}{(m_{x}-1)!}   D^{m_{x}} p_{\mu}(x - tuh) dt \:du, \label{bias3}
    \end{align}
    where (\ref{bias1}) is obtained using the change in variables $\frac{y}{h} = \frac{x}{h} - u$. In (\ref{bias2}), we use Taylor's expansion for multivariate functions, in particular, $p_{\mu} \in \mathcal{W}^{m_{x},p}_{L}(\Omega_{x})$. The first part of the sum vanishes due to the regularity of underlying kernels (see Definition \ref{Reg_inv_ker}). Now, using Minkowski's inequality for integrals, given $p=1$ we get
    \begin{align}
        \norm{\mathbb{E}[\hat{p}_{h}(x)] - p_{\mu}}_{1} &\leq h^{m_{x}} \norm{D^{m_{x}} p_{\mu}}_{1} \int \sup_{v} \abs{\kappa(v,v-u)}{\abs{u}}^{m_{x}} du \int^{1}_{0} \frac{(1-t)^{m_{x}-1}}{(m_{x}-1)!} dt  \nonumber \\ &\lesssim h^{m_{x}}, \label{part2}
    \end{align}
    again due to the regularity of $\kappa$ and Assumption \ref{assumption1}.

    Let us define $\mathscr{K} = \left\{f(x,\cdot) = \frac{1}{h^{d}} \kappa\left(\frac{x}{h}, \frac{\cdot}{h}\right): x \in \Omega_{x}\right\}$. Now, given the invariance of the underlying kernels $\kappa$, the bracketing number of $\mathscr{K}$ turns out to satisfy
    \begin{equation*}
        \mathcal{N}_{[\:]}(\mathscr{K}, L_{1}(\mu), \epsilon) \leq E_{\kappa}\left(\frac{L \sqrt{d} B_{x}}{h^{d+1}\epsilon}\right)^{d},
    \end{equation*} 
    where $E_{\kappa} > 0$ is an universal constant depending on $d$ (\citet{van2000asymptotic}, Example 19.7). Also, observe that 
    \begin{align*}
        \textrm{Var}(f) \leq \int f^{2} d\mu \leq \norm{f}_{\infty} \int \abs{f} d\mu \leq \frac{B}{h^{d}},
    \end{align*}
    where the last inequality utilizes the fact $\sup_{f} \mathbb{E}\abs{f} \leq \mathbb{E}[\sup_{f} \abs{f}]$. Hence, using Bernstein's inequality (see, \citet{yukich1985density} for the detailed bracketing argument) we infer that
    \begin{equation}
        \mathbb{P}^{n}\left(\sup_{x} \abs{\hat{p}_{h}(x) - \mathbb{E}[\hat{p}_{h}(x)]} > \epsilon \right) \leq 4 E_{\kappa}\left(\frac{L \sqrt{d} B_{x}}{h^{d+1}\epsilon}\right)^{d} \exp\left\{-\frac{En \epsilon^2 h^d}{B}\right\}, \label{part3}
    \end{equation}
    where $E > 0$ is an universal constant\footnote{\citet{lafferty2008concentration} derive the sharp value of $E$ under Lipschitz continuous kernels.} and $0< \epsilon \leq \frac{2}{3}$. This enables us to state a concentration bound for $d_{\mathcal{H}}\left(g_{\#}\Tilde{\mu}_{n}, g_{\#}\mu\right)$ readily using (\ref{part1}) and (\ref{part2}).

    Bounding the expected estimation error based on the translated data (second term in \ref{main_ineq}) is comparatively straightforward. We present the refined Dudley's entropy integral, which becomes the cornerstone of the rest of the proof.
    \begin{lemma} \label{Dudley}
        Given a symmetric class of functions $\mathcal{H}$, satisfying $\sup_{h \in \mathcal{H}} \norm{h}_{\infty} \leq M$, $M > 0$ we have 
        \begin{equation*}
            \mathbb{E}[d_{\mathcal{H}}\left(\hat{\rho}_{m}, \rho\right)] \leq 2\inf_{\delta \in (0,M)} \left(2\delta + \frac{12}{\sqrt{m}}\int^{M}_{\delta}\sqrt{\log \mathcal{N}\left(\mathcal{H}, \norm{\cdot}_{\infty}, \epsilon\right)} d\epsilon\right),
        \end{equation*}
        where $\rho \in \mathcal{P}(\mathcal{Z})$.
    \end{lemma}
    As such, given that the entropy corresponding to the underlying critic functions satisfy polynomial discrimination, we obtain an upper bound corresponding to the infimum as follows
    \begin{equation}
        \mathbb{E}\left[d_{\mathcal{H}}\left(g_{\#}\mu, \widehat{(g_{\#}\mu)}_m\right)\right] \lesssim m^{-\frac{1}{q \vee 2}}. \label{part4}
    \end{equation}
    To obtain a probabilistic concentration inequality corresponding to the same error observe that
    \begin{equation*}
        d_{\mathcal{H}}\left(g_{\#}\mu, \widehat{(g_{\#}\mu)}_m\right) = \sup_{h \in \mathcal{H}}\Big\{\frac{1}{m}\sum^{m}_{i=1} h(Y_{i})- \mathbb{E}_{g_{\#}\mu}h\Big\} := W(Y_{1}, \cdots , Y_{m}),
    \end{equation*}
    given $Y_{1}, \cdots , Y_{m} \sim g_{\#}\mu$. As such, for $y_{1}, \cdots, y_{m}, y^{'}_{m}$
    \begin{align*}
        &\abs{W(y_{1}, \cdots , y_{m}) - W(y_{1}, \cdots , y^{'}_{m})} \\ =& \frac{1}{m}\abs{\sup_{h \in \mathcal{H}}\sum^{m}_{i=1} \left(h(y_{i})- \mathbb{E}_{g_{\#}\mu}h\right) - \sup_{h^{'} \in \mathcal{H}}\sum^{m-1}_{i=1} \left(h^{'}(y_{i})- \mathbb{E}_{g_{\#}\mu}h^{'}\right) + h^{'}(y^{'}_{m}) - \mathbb{E}_{g_{\#}\mu}h^{'}} \\ \leq & \frac{1}{m}\abs{\sup_{h \in \mathcal{H}} h(y_{m}) - h(y_{m}^{'})} \leq \frac{2M}{m}.
    \end{align*}
    Applying McDiarmid's inequality, 
    \begin{equation*}
        d_{\mathcal{H}}\left(g_{\#}\mu, \widehat{(g_{\#}\mu)}_m\right) \leq \epsilon + \mathcal{O}(m^{-\frac{1}{q \vee 2}})
    \end{equation*}
    holds with probability $\geq 1 - \exp\left\{-\frac{n \epsilon^2}{2 M^2}\right\}$, where $\epsilon > 0$. This bound, along with (\ref{part3}) ensure the existence of constants $l, E_{1}, E_{2}$ and $E_{3} >0$ that proof the theorem. Observe that, the bound satisfies for arbitrary choices of $h$. To ensure that the realized error is indeed $o(1)$ with high probability, we specify $h := h_{n} = (\frac{1}{n})^{\xi}$ such that $\xi \geq \frac{1}{d}$. 
\end{proof}

\begin{proof}[Proof of Lemma \ref{lemm_bound}]
    Given any $\phi \in \Phi(W, L)^{k}_{d}$ and $\mu_{1}, \mu_{2} \in \mathcal{P}(\mathcal{X})$, by the definition of IPMs
    \begin{align*}
        d_{\mathcal{L}_{c_{z}}^{1}}(\phi_{\#}\mu_{1}, \phi_{\#}\mu_{2}) &= \sup_{l \in \mathcal{L}_{c_{z}}^{1}} \mathbb{E}_{\mu_{1}} [l \circ \phi] - \mathbb{E}_{\mu_{2}} [l \circ \phi] \\ &= \sup_{l \in \mathcal{L}_{c_{z}}^{1}} \left\{\mathbb{E}_{\mu_{1}} [l \circ \phi] - \mathbb{E}_{\mu_{1}} [l \circ g] + \mathbb{E}_{\mu_{2}} [l \circ g] - \mathbb{E}_{\mu_{2}} [l \circ \phi] + \mathbb{E}_{\mu_{1}} [l \circ g] - \mathbb{E}_{\mu_{2}} [l \circ g]\right\} \\ &\leq \sup_{l \in \mathcal{L}_{c_{z}}^{1}} \left\{ \mathbb{E}_{\mu_{1}}\abs{l \circ \phi - l \circ g} + \mathbb{E}_{\mu_{2}}\abs{l \circ g - l \circ \phi} + \mathbb{E}_{\mu_{1}} [l \circ g] - \mathbb{E}_{\mu_{2}} [l \circ g]\right\} \\ &\leq 2\norm{\phi - g}_{\infty} + \sup_{l \in \mathcal{L}_{c_{z}}^{1}} \mathbb{E}_{\mu_{1}} [l \circ g] - \mathbb{E}_{\mu_{2}} [l \circ g],
    \end{align*}
    where the first inequality is due to Jensen's inequality and in the second one we use the fact that $l \in \mathcal{L}_{c_{z}}^{1}$, given $c_{z} \equiv L_{1}$. The arbitrary choice of $g \in \mathscr{F}(\mathcal{X}, \mathcal{P}(\mathcal{Z}))$ makes the result hold for the infimum as well. As such,
    \begin{equation} \label{bnd_1}
        d_{\mathcal{L}_{c_{z}}^{1}}(\phi_{\#}\mu_{1}, \phi_{\#}\mu_{2}) \leq 2 \inf_{g \in \mathscr{F}(\mathcal{X}, \mathcal{P}(\mathcal{Z}))}\norm{\phi - g}_{\infty} + d_{\mathcal{L}_{c_{z}}^{1}}(g_{\#}\mu_{1}, g_{\#}\mu_{2}).
    \end{equation}
    Now, under the critic $\mathcal{F}$, $\forall \varepsilon > 0$ there exists $f_{\varepsilon} \in \mathcal{F}$ such that
    \begin{align*}
        d_{\mathcal{F}}(\phi_{\#}\mu_{1}, \phi_{\#}\mu_{2}) &\leq \mathbb{E}_{\phi_{\#}\mu_{1}} [f_{\varepsilon}] - \mathbb{E}_{\phi_{\#}\mu_{2}} [f_{\varepsilon}] + \varepsilon \\ &= \mathbb{E}_{\phi_{\#}\mu_{1}} [f_{\varepsilon} - l] + \mathbb{E}_{\phi_{\#}\mu_{2}} [l - f_{\varepsilon}] + \mathbb{E}_{\phi_{\#}\mu_{1}} [l] - \mathbb{E}_{\phi_{\#}\mu_{2}} [l] + \varepsilon \\ &\leq 2\norm{f_{\varepsilon} - l}_{\infty} + \sup_{l \in \mathcal{L}_{c_{z}}^{1}} \mathbb{E}_{\phi_{\#}\mu_{1}} [l] - \mathbb{E}_{\phi_{\#}\mu_{2}} [l] + \varepsilon.
    \end{align*}
    Similarly, taking infimum over all such choices of $l$
    \begin{equation} \label{bnd_2}
        d_{\mathcal{F}}(\phi_{\#}\mu_{1}, \phi_{\#}\mu_{2}) \leq 2 \mathcal{E}(\mathcal{F}, \mathcal{L}_{c_{z}}^{1}) + d_{\mathcal{L}_{c_{z}}^{1}}(\phi_{\#}\mu_{1}, \phi_{\#}\mu_{2}) + \varepsilon.
    \end{equation}
    The inequalities (\ref{bnd_1}) and (\ref{bnd_2}) together proof the lemma. 
\end{proof}

\begin{proof}[Proof of Theorem \ref{MMD_IPT}]
    Given any NN-induced map $\phi$, using the triangle inequality on MMDs one can write
    \begin{equation} \label{thm4.8_1}
        d_{\mathcal{H}_{\kappa}}\left(\phi_{\#}\hat{\mu}_{n} , \widehat{(\phi_{\#}\mu)}_m\right) \leq \underbrace{d_{\mathcal{H}_{\kappa}}\left(\phi_{\#}\hat{\mu}_{n} , \phi_{\#}\mu\right)}_\text{(i)} + \underbrace{d_{\mathcal{H}_{\kappa}}\left(\phi_{\#}\mu , \widehat{(\phi_{\#}\mu)}_m\right)}_\text{(ii)}. 
    \end{equation}
    Since the underlying kernels are bounded to begin with, an immediate upper bound for (ii) might be: $d_{\mathcal{H}_{\kappa}}(\phi_{\#}\mu , \widehat{(\phi_{\#}\mu)}_m) \leq \sqrt{\sup_{z \in \Omega_{z}}\kappa(z,z)}\: d_{\textrm{TV}}(\phi_{\#}\mu , \widehat{(\phi_{\#}\mu)}_m)$ (\citet{sriperumbudur2009integral}, Theorem 14 (ii)). Being larger in general, TV may enforce information preservation onto MMDs. However, the very property of boundedness of the kernels enables us to show the concentration of empirical measures under MMD as well. Observe that, for bounded kernels, MMD satisfies the bounded difference inequality with the universal upper bound $2m^{-1}\sqrt{\sup_{z \in \Omega_{z}}\kappa(z,z)}$. As such, using McDiarmid's inequality
    \begin{equation} \label{thm4.8_2}
        \mathbb{P}\left(d_{\mathcal{H}_{\kappa}}(\phi_{\#}\mu , \widehat{(\phi_{\#}\mu)}_m) \leq \mathbb{E}[d_{\mathcal{H}_{\kappa}}(\phi_{\#}\mu , \widehat{(\phi_{\#}\mu)}_m)] + t\right) \geq 1 - e^{-\frac{mt^{2}}{2C_{\kappa}}},
    \end{equation}
    where $C_{\kappa}$ is a positive constant such that $\sup_{z \in \Omega_{z}}\kappa(z,z) \leq C_{\kappa}$. Furthermore, we observe that $\mathbb{E}\left[d_{\mathcal{H}_{\kappa}}(\phi_{\#}\mu , \widehat{(\phi_{\#}\mu)}_m)\right] \leq \left[\mathbb{E} \:d^{2}_{\mathcal{H}_{\kappa}}(\phi_{\#}\mu , \widehat{(\phi_{\#}\mu)}_m)\right]^{\frac{1}{2}} \leq \sqrt{\frac{2C_{k}}{m}}$ (\citet{briol2019statistical}, lemma 2). 
    In pursuit of establishing an upper bound to (i), one needs additional enforcement. The first of which is presented as the following lemma.

    \begin{lemma} \label{MMD_1}
        Given arbitrary $\phi, g \in \mathscr{F}(\mathcal{X},\mathcal{Z})$ and $\mu_{1}, \mu_{2} \in \mathcal{P}_{\kappa}(\mathcal{X})$ such that the underlying kernel function $\kappa$ is strongly invariant, there exists a constant $D > 0$ (dependant on $\kappa$ and the latent dimension) for which
        \[d^{2}_{\mathcal{H}_{\kappa}}(\phi_{\#}\mu_{1}, \phi_{\#}\mu_{2}) \leq D\norm{\phi - g}_{\infty} + d^{2}_{\mathcal{H}_{\kappa}}(g_{\#}\mu_{1}, g_{\#}\mu_{2}).\]
    \end{lemma}

    \begin{proof}[Proof of lemma \ref{MMD_1}]
    Observe that
        \begin{align*}
            &\abs{d^{2}_{\mathcal{H}_{\kappa}}(\phi_{\#}\mu_{1}, \phi_{\#}\mu_{2}) - d^{2}_{\mathcal{H}_{\kappa}}(g_{\#}\mu_{1}, g_{\#}\mu_{2})} \\ = &  \abs{\int_{\Omega_{x} \times \Omega_{x}}\left[\kappa(\phi(x),\phi(y)) - \kappa(g(x),g(y))\right] (\mu_{1}-\mu_{2}) \otimes (\mu_{1}-\mu_{2}) (dxdy)} \\ = & \Bigg|\int_{\Omega_{x} \times \Omega_{x}}\left[\kappa(\phi(x),\phi(y)) - \kappa(g(x),\phi(y)) + \kappa(g(x),\phi(y)) - \kappa(g(x),g(y))\right] \\ & \mspace{250mu} (\mu_{1}-\mu_{2}) \otimes (\mu_{1}-\mu_{2}) (dxdy)\Bigg| \\ \overset{(1)}{=} & \Bigg|\int_{\Omega_{x}}\left[K(\phi_{\#}\mu_{1} - \phi_{\#}\mu_{2})(\phi(x)) - K(\phi_{\#}\mu_{1} - \phi_{\#}\mu_{2})(g(x))\right] (\mu_{1}-\mu_{2}) (dx) \\ & \mspace{50mu} + \int_{\Omega_{x}}\left[K(g_{\#}\mu_{1} - g_{\#}\mu_{2})(\phi(y)) - K(g_{\#}\mu_{1} - g_{\#}\mu_{2})(g(y))\right] (\mu_{1}-\mu_{2}) (dy)\Bigg| \\ \leq &\int_{\Omega_{x}}\abs{K(\phi_{\#}\mu_{1} - \phi_{\#}\mu_{2})(\phi(x)) - K(\phi_{\#}\mu_{1} - \phi_{\#}\mu_{2})(g(x))} \abs{\mu_{1}-\mu_{2}} (dx) \\ & \mspace{50mu} + \int_{\Omega_{x}}\abs{K(g_{\#}\mu_{1} - g_{\#}\mu_{2})(\phi(y)) - K(g_{\#}\mu_{1} - g_{\#}\mu_{2})(g(y))} \abs{\mu_{1}-\mu_{2}} (dy), 
        \end{align*}
        where ($1$) is due to the fact
        \begin{align*}
            &\int \kappa(\phi(x),\phi(y)) (\mu_{1}-\mu_{2}) \otimes (\mu_{1}-\mu_{2}) (dxdy) \\ = & \int \kappa(\phi(x),y) (\mu_{1}-\mu_{2}) \otimes (\phi_{\#}\mu_{1} - \phi_{\#}\mu_{2}) (dxdy) = \int K(\phi_{\#}\mu_{1} - \phi_{\#}\mu_{2})(\phi(x)) (\mu_{1}-\mu_{2}) (dx).
        \end{align*}
        Now,
        \begin{align*}
            &\abs{K(\phi_{\#}\mu_{1} - \phi_{\#}\mu_{2})(\phi(x)) - K(\phi_{\#}\mu_{1} - \phi_{\#}\mu_{2})(g(x))} \\ &\leq \int_{\Omega_{z}} \abs{\kappa(\phi(x),y) - \kappa(g(x),y)} \abs{\phi_{\#}\mu_{1} - \phi_{\#}\mu_{2}} (dy) \\ &\overset{(2)}{\leq} \int_{\Omega_{z}} \norm{K(\phi( x)) - K(g(x))} \sqrt{\kappa(y,y)} \abs{\phi_{\#}\mu_{1} - \phi_{\#}\mu_{2}} (dy) \\ &\overset{(3)}{\lesssim} \norm{\phi(x) - g(x)} \int_{\Omega_{z}} \sqrt{\kappa(y,y)} \abs{\phi_{\#}\mu_{1} - \phi_{\#}\mu_{2}} (dy),
        \end{align*}
        where ($2$) is due to the reproducing kernel property, coupled with the Cauchy-Schwartz inequality. The strong invariance of the underlying kernel inspires ($3$). Noticing the quantity under the integral to be finite, we obtain 
        \begin{align*}
            &\int_{\Omega_{x}}\abs{K(\phi_{\#}\mu_{1} - \phi_{\#}\mu_{2})(\phi(x)) - K(\phi_{\#}\mu_{1} - \phi_{\#}\mu_{2})(g(x))} \abs{\mu_{1}-\mu_{2}} (dx) \\ &\lesssim \int_{\Omega_{x}} \norm{\phi(x) - g(x)} \abs{\mu_{1}-\mu_{2}} (dx) \lesssim \norm{\phi - g}_{\infty},
        \end{align*}
        since the dominating measure is sigma-finite. The suppressed constant is namely $k$ (the latent dimension). Similarly, observing $\int_{\Omega_{z}} \sqrt{\kappa(y,y)} \abs{g_{\#}\mu_{1} - g_{\#}\mu_{2}} (dy) < \infty$ in addition, we conclude
        \begin{align*}
            \abs{d^{2}_{\mathcal{H}_{\kappa}}(\phi_{\#}\mu_{1}, \phi_{\#}\mu_{2}) - d^{2}_{\mathcal{H}_{\kappa}}(g_{\#}\mu_{1}, g_{\#}\mu_{2})} \lesssim \norm{\phi - g}_{\infty}.
        \end{align*}
        As such, there indeed exists a constant that satisfies the lemma. 
    \end{proof}
    \noindent
    Now, let us choose in particular $g \in \mathscr{F}_{L}(\mathcal{X}, \mathcal{Z})$. For ease of understanding, we continue with the distributions $\mu_{1}, \mu_{2} \in \mathcal{P}_{\kappa}(\mathcal{X})$. Using the reproducing property again, we get
    \begin{align*}
        d^{2}_{\mathcal{H}_{\kappa}}(g_{\#}\mu_{1}, g_{\#}\mu_{2}) &= \int_{\Omega_{x} \times \Omega_{x}} \kappa(g(x),g(y)) (\mu_{1}-\mu_{2}) \otimes (\mu_{1}-\mu_{2}) (dxdy) \\ &= \int_{\Omega_{x}} K(g_{\#}\mu_{1} - g_{\#}\mu_{2})(g(x)) (\mu_{1}-\mu_{2}) (dx).
    \end{align*}
    While proving Lemma \ref{MMD_1}, we have observed that the function $K(g_{\#}\mu_{1} - g_{\#}\mu_{2})$ is Lipschitz continuous with accompanying constant $c_{g}^{(\mu_{1},\mu_{2})} = \int_{\Omega_{z}} \sqrt{\kappa(y,y)} \abs{g_{\#}\mu_{1} - g_{\#}\mu_{2}} (dy)$. We mention that for Energy kernels, the same function rather turns out to be H\"{o}lder continuous. As such, 
    \begin{align*}
        d^{2}_{\mathcal{H}_{\kappa}}(g_{\#}\mu_{1}, g_{\#}\mu_{2}) &\leq c_{g}^{(\mu_{1},\mu_{2})} \sup_{f \in \mathcal{L}_{c_{z} \equiv L_{2}}^1} \left[\int_{\Omega_{x}} f(g(x)) (\mu_{1}-\mu_{2}) (dx)\right] \\ &\overset{(4)}{=} c_{g}^{(\mu_{1},\mu_{2})} d_{\mathcal{L}_{c_{z}}^{1}}(g_{\#}\mu_{1}, g_{\#}\mu_{2}) \leq c_{g}^{(\mu_{1},\mu_{2})}L d_{\mathcal{L}_{c_{x}}^{1}}(\mu_{1},\mu_{2}),  
    \end{align*}
    where ($4$) is due to the Kantorovitch-Rubinstein duality. Hence, we may write
    \begin{align} \label{thm4.8_3}
        d^{2}_{\mathcal{H}_{\kappa}}\left(\phi_{\#}\hat{\mu}_{n} , \phi_{\#}\mu\right) \leq D_{n}\norm{\phi - g}_{\infty} + c_{g}^{(\hat{\mu}_{n},\mu)}L d_{\mathcal{L}_{c_{x}}^{1}}(\hat{\mu}_{n},\mu),
    \end{align}
    where $D_{n}$ and $c_{g}^{(\hat{\mu}_{n},\mu)}$ are no longer constants, but sequences based on $n$ that converge to $0$ almost surely as $n \rightarrow \infty$ (by dominated convergence theorem). In particular,
    \begin{equation*}
        c_{g}^{(\hat{\mu}_{n},\mu)} = c_{g,n} = \int_{\Omega_{z}} \sqrt{\kappa(y,y)} \abs{g_{\#}\hat{\mu}_{n} - g_{\#}\mu} (dy)=\int_{\Omega_{x}} \sqrt{\kappa(g(y),g(y))} \abs{\hat{\mu}_{n} - \mu} (dy)
    \end{equation*}
    and $D_{n} = o(c_{g,n} \vee c_{\phi,n})$. Applying the concentration of $\hat{\mu}_{n}$ around $\mu$ under the metric $d_{\mathcal{L}_{c_{x}}^{1}}$, along with the inequalities \ref{thm4.8_1} and \ref{thm4.8_2} we infer that given $t>0$
    \begin{equation*}
        \mathbb{P}\left(d_{\mathcal{H}_{\kappa}}\left(\phi_{\#}\hat{\mu}_{n} , \widehat{(\phi_{\#}\mu)}_m\right) \leq t + \sqrt{\frac{2C_{\kappa}}{m}} + \sqrt{D_{n}\norm{\phi - g}_{\infty}} + \sqrt{\mathcal{O}(c_{g}^{(\hat{\mu}_{n},\mu)}{(d^{2}n)}^{-\frac{1}{d}}) + c_{g}^{(\hat{\mu}_{n},\mu)}Lt}\right)
    \end{equation*}
    remains at least $1 - 2\exp{\left(-\frac{2(m \wedge n)t^{2}}{B}\right)}$, where $B = \max\{B^{2}_{x},4C_{\kappa}\}$. The quantity $B_{x}$ symbolises the diameter of $\Omega_{x}$ under the metric $c_{x}$.
\end{proof}

\begin{proof}[Proof of Corollary \ref{IPT_sig}]
    The proof takes inspiration from Theorem 3.5 of \citet{lee2017ability}. First, let us construct $\phi_{1:L+1} = \phi_{1} \circ \cdots \circ \phi_{L+1}$ where the individual functions are defined as 
    \[\phi_{i} : \mathbb{R}^{N_{i-1}} \rightarrow \mathbb{R}^{N_{i}} \;\textrm{such that} \;\left(\phi_{i}(x)\right)_{j} = c_{ij0} + \sum_{k=1}^{r_{i}} c_{ijk} \sigma\left(\langle m_{ijk}, x\rangle + b_{ijk}\right),\]
    where $c_{ijk}, b_{ijk} \in \mathbb{R}$ and $m_{ijk} \in \mathbb{R}^{N_{i-1}}$ are model parameters in accordance with Definition \ref{NN}. Since these are shallow networks, the quantity $r_{i}, 1 \leq i \leq L+1$ denotes the number of nodes they have. Observe that, following Barron's argument, $f_{1}: \mathbb{R}^{d} \rightarrow \mathbb{R}^{N_{1}}$ can be approximated using $\phi_{1}$ (under the $L_{2}$ norm, with respect to $\mu$). The same result holds for all intermediate individual pieces, with respect to measures at their respective domains. Our goal is to approximate the composition of them all. To invoke an induction argument, let us first consider a sequence of nested sets $\{S_{i}\}_{i=1}^{L+1} \subseteq \mathbb{R}^{d}$ such that $S_{i} = S_{i-1} \cap \{x: \phi_{1:i-1}(x) \in \Omega^{s,N_{i-1}}_{i-1}\}$. The measure $\mu$, restricted onto $S_{i}$, when pushed-forward by the map $\phi_{1:i-1}$ yields $\mu^{i} = (\phi_{1:i-1})_{\#}(\mathbb{1}_{S_{i}}\mu)$ (need not be a probability measure). The support of $\mu^{i}$ is thus obtained to be $\phi_{1:i-1}(S_{i}) \subseteq \Omega^{s,N_{i-1}}_{i-1}$. Let us denote the Lipschitz constant corresponding to $f_{i}$ by $\norm{f_{i}}_{\mathcal{B}}$ \citep{wojtowytsch2022representation}. Now, given any $\varepsilon > 0$, define $r_{i} = \ceil{\frac{4C^{2}_{i}N_{i}}{\varepsilon^{2}}}$. Now,
    \begin{align}
        &\left(\int_{\mathbb{R}^{d}} \mathbb{1}_{S_{i}} \norm{f_{1:i} - \phi_{1:i}}^{2} d\mu\right)^{\frac{1}{2}} \nonumber \\ \overset{(1)}{\leq} & \left(\int_{\mathbb{R}^{d}} \mathbb{1}_{S_{i}} \norm{f_{i} \circ f_{1:i-1} - f_{i} \circ \phi_{1:i-1}}^{2} d\mu\right)^{\frac{1}{2}} + \left(\int_{\mathbb{R}^{N_{i-1}}}  \norm{f_{i} - \phi_{i}}^{2} d(\phi_{1:i-1})_{\#}(\mathbbm{1}_{S_{i}}\mu)\right)^{\frac{1}{2}} \nonumber \\ \leq & \:\norm{f_{i}}_{\mathcal{B}} \left(\int_{\mathbb{R}^{d}} \mathbb{1}_{S_{i}} \norm{f_{1:i-1} - \phi_{1:i-1}}^{2} d\mu\right)^{\frac{1}{2}} + \varepsilon \nonumber \\ \overset{(2)}{\leq} & \:\norm{f_{i}}_{\mathcal{B}} \left(\int_{\mathbb{R}^{d}} \mathbb{1}_{S_{i-1}} \norm{f_{1:i-1} - \phi_{1:i-1}}^{2} d\mu\right)^{\frac{1}{2}} + \varepsilon \nonumber \\ \leq & \left[1+\norm{f_{i}}_{\mathcal{B}}+\norm{f_{i}}_{\mathcal{B}}\norm{f_{i-1}}_{\mathcal{B}}+\cdots+\prod_{j=1}^{i}\norm{f_{j}}_{\mathcal{B}}\right]\varepsilon \leq \frac{\left\{\bigvee_{1}^{i}\norm{f_{j}}_{\mathcal{B}}\right\}^{i+1} - 1}{\bigvee_{1}^{i}\norm{f_{j}}_{\mathcal{B}}-1} \varepsilon \nonumber,
    \end{align}
   where $(1)$ is due to the triangle inequality. The second term in the same stage turns out to be $\leq \varepsilon$ using Barron's theorem \citep{barron1993universal, lee2017ability} given our specific choice of $r_{i}$. The inequality $(2)$ is based on the filtration offered by $S_{i}$. Observe that, in case the maximum of the Lipschitz constants is exactly $1$, the upper bound becomes $i\varepsilon$.
   This implies that
   \begin{align*}
       \mu\left(S_{i-1} \cup \{x: \norm{f_{1:i-1}(x) - \phi_{1:i-1}(x)} \geq s\}\right) \leq \left[\frac{\left\{\bigvee_{1}^{i-1}\norm{f_{j}}_{\mathcal{B}}\right\}^{\overline{i-1}+1} - 1}{\bigvee_{1}^{i-1}\norm{f_{j}}_{\mathcal{B}}-1}\right]^{2} \cdot \frac{\varepsilon^{2}}{s^{2}}.
   \end{align*}
   As such, 
   \begin{align}
       \mu(S_{i}) &= \mu(S_{i-1}) - \mu(S_{i-1} \cup \{x: \norm{f_{1:i-1}(x) - \phi_{1:i-1}(x)} \geq s\}) \nonumber \\ &\geq \mu(S_{i-1}) - \left[\frac{\left\{\bigvee_{1}^{i-1}\norm{f_{j}}_{\mathcal{B}}\right\}^{i} - 1}{\bigvee_{1}^{i-1}\norm{f_{j}}_{\mathcal{B}}-1}\right]^{2} \cdot \frac{\varepsilon^{2}}{s^{2}} \geq 1 - \frac{\varepsilon^{2}}{s^{2}} \sum_{l=1}^{i-1}\left[\frac{\left\{\bigvee_{1}^{l}\norm{f_{j}}_{\mathcal{B}}\right\}^{l+1} - 1}{\bigvee_{1}^{l}\norm{f_{j}}_{\mathcal{B}}-1}\right]^{2} \label{cor_barr}
   \end{align}
   In other words, there exists $\phi = \phi_{1:L+1}$ such that $\left(\int_{\mathbb{R}^{d}} \mathbb{1}_{S} \norm{f_{1:L+1} - \phi}^{2} d\mu\right)^{\frac{1}{2}} \leq \mathcal{O}(\varepsilon)$ for a set $S \subset \mathbb{R}^d$ satisfying inequality \ref{cor_barr}. Now, owing to the compactness of the base domain and its regularity, the range of $\phi_{i} = (\phi_{i1}, \phi_{i2}, \cdots, \phi_{iN_{i}})$ is always contained in a space of finite diameter. Hence,
   \begin{align*}
       \int_{\Omega_{x}} \norm{f_{1:i} - \phi_{1:i}}^{2} d\mu \leq &\int_{S_{i}} \norm{f_{1:i} - \phi_{1:i}}^{2} d\mu + \int_{S^{c}_{i}} \norm{f_{1:i} - \phi_{1:i}}^{2} d\mu \\ \leq &\left[\frac{\left\{\bigvee_{1}^{i}\norm{f_{j}}_{\mathcal{B}}\right\}^{i+1} - 1}{\bigvee_{1}^{i}\norm{f_{j}}_{\mathcal{B}}-1}\right]^{2} \cdot {\varepsilon}^{2} + \mathcal{O}\left(\frac{\varepsilon^{2}}{s^{2}}\right).
   \end{align*}
   Taking the square root and choosing $i$ corresponding to the ultimate layer we prove the result.
\end{proof}



\begin{proof}[Proof of Theorem \ref{latent_con_MMD}]
Given an encoder $\phi \in \Phi(W, L)^{k}_{d}$ and an empirical distribution corresponding to $\mu$ (based on $n$ i.i.d. replicates), let us fragment the realized latent WAE-MMD loss as usual,
\begin{align} \label{MMD_frag}
    d_{\mathcal{H}_{\kappa}}\left(\phi_{\#}\hat{\mu}_{n}, \rho\right) \leq d_{\mathcal{H}_{\kappa}}\left(\phi_{\#}\hat{\mu}_{n}, \widehat{(\phi_{\#}\mu)}_n\right) + d_{\mathcal{H}_{\kappa}}\left(\widehat{(\phi_{\#}\mu)}_n, \hat{\rho}_{n}\right) + d_{\mathcal{H}_{\kappa}}\left(\hat{\rho}_{n}, \rho\right)
\end{align}
which holds for all empirical $\hat{\rho}_{n}$, based on $n$ i.i.d. samples from $\rho$. Observe that, the first quantity is the information dissipated during encoding and can be put under a deterministic upper bound using Theorem \ref{MMD_IPT}. The second quantity, due to the observation $\norm{\kappa(\cdot,z) - \kappa(\cdot,z^{'})}_{\mathcal{H}_{\kappa}} \leq 2\sqrt{C_{\kappa}} \mathbb{1}_{z \neq z^{'}}$, satisfies
\begin{align} \label{upper_MMD}
    d_{\mathcal{H}_{\kappa}}\left(\widehat{(\phi_{\#}\mu)}_n, \hat{\rho}_{n}\right) \leq \sqrt{C_{\kappa}} d_{\textrm{TV}}\left(\widehat{(\phi_{\#}\mu)}_n, \hat{\rho}_{n}\right).
\end{align}
Now, recall that $\rho \in \mathcal{P}_{\Sigma}(\Omega_{z})$. As such, the observations from $\rho$, forming its empirical counterpart, are first projected onto the subset of $\Sigma$-invariant distributions. This is termed \textit{symmetrization} and the corresponding operator enabling it, $S^{\Sigma} : \mathcal{P}(\mathcal{Z}) \rightarrow \mathcal{P}(\mathcal{Z})$ is defined as 
\begin{align*}
    \mathbb{E}_{S^{\Sigma}[\rho]} f = \int_{\mathcal{Z}} \left[\int_{\Sigma} f(\varphi_{\sigma}(z)) \mu_{\Sigma}(d\sigma)\right] d\rho(z) = \mathbb{E}_{\rho} \mathbb{E}_{\mu_{\Sigma}} [f \circ \varphi_{\sigma}],
\end{align*}
where $\mu_{\Sigma}$ is the Haar measure on the compact group $\Sigma$ and $f$ denotes any bounded measurable function. In other words, the recipe to obtain samples from $S^{\Sigma}[\rho]$ in general is to draw i.i.d. $\{z_{i}\}^{n}_{1} \sim \rho$ and $\{\sigma_{i}\}^{n}_{1} \sim \mu_{\Sigma}$ independently, followed by the operation $\varphi_{\sigma_{j}}(z_{i})$ \citep{Birrell2022StructurepreservingG}. This in turn enables us to narrow down the class of critic functions under MMD to its $\Sigma$-invariant subset $\mathcal{H}^{\Sigma}_{\kappa}$. We use the same idea to obtain an upper bound to the remaining estimation error in (\ref{MMD_frag}) as follows. Let,
    \begin{align}
        \gamma(z_{1}, z_{2}, \cdots, z_{n}) &= d_{\mathcal{H}_{\kappa}}(\hat{\rho}_{n}, \rho) =  d_{\mathcal{H}^{\Sigma}_{\kappa}}(\hat{\rho}_{n}, \rho) \nonumber \\ &= \sup_{\norm{f}_{\mathcal{H}_{\kappa}} \leq 1} \big\{\mathbb{E}_{S^{\Sigma}[\hat{\rho}_{n}]} f - \mathbb{E}_{\rho} f\big\} \nonumber \\ &= \sup_{\norm{f}_{\mathcal{H}_{\kappa}} \leq 1} \big\{\frac{1}{n\abs{\Sigma}} \sum_{i=1}^{n} \sum_{j=1}^{\abs{\Sigma}} f(\sigma_{j} z_{i})- \mathbb{E}_{\rho} f\big\} \nonumber \\ &\leq \sup_{\norm{f}_{\mathcal{H}_{\kappa}} \leq 1} \abs{\frac{1}{n\abs{\Sigma}} \sum_{i=1}^{n} \sum_{j=1}^{\abs{\Sigma}} f(\sigma_{j} z_{i})- \mathbb{E}_{\rho} f}.
    \end{align}
    To look for the specific constant satisfying the bounded difference property, observe that 
    \begin{align}
        &\abs{\gamma(z_{1},\cdots, z_{i}, \cdots, z_{n}) - \gamma(z_{1},\cdots, z^{'}_{i}, \cdots, z_{n})} \nonumber \\ &\leq \sup_{\norm{f}_{\mathcal{H}_{\kappa}} \leq 1} \abs{\frac{1}{n\abs{\Sigma}} \sum_{j=1}^{\abs{\Sigma}} f(\sigma_{j} z_{i})- f(\sigma_{j} z^{'}_{i})} \nonumber \\ &\leq \frac{1}{n\abs{\Sigma}} \Big\{\norm{\sum_{j=1}^{\abs{\Sigma}} K(\sigma_{j} z_{i})}_{\mathcal{H}_{\kappa}} +  \norm{\sum_{j=1}^{\abs{\Sigma}} K(\sigma_{j} z^{'}_{i})}_{\mathcal{H}_{\kappa}}\Big\},
    \end{align}
    where 
    \begin{align}
        \norm{\sum_{j=1}^{\abs{\Sigma}} K(\sigma_{j} z_{i})}_{\mathcal{H}_{\kappa}} &= \left[\sum_{j=1}^{\abs{\Sigma}} \kappa(\sigma_{j} z_{i},\sigma_{j} z_{i}) + \sum_{j \neq l} \kappa(\sigma_{j} z_{i},\sigma_{l} z_{i})\right]^{\frac{1}{2}} \nonumber \\
        &= \left[\sum_{j=1}^{\abs{\Sigma}} \kappa(\sigma_{j} z_{i},\sigma_{j} z_{i}) + \sum_{\sigma_{j} \neq \textrm{id}} \kappa(\sigma_{j} z_{i}, z_{i})\right]^{\frac{1}{2}} \\ &\leq \sqrt{C_{\kappa} \abs{\Sigma}}\left[1+\varsigma_{\kappa,\Sigma}\left(\abs{\Sigma}-1\right)\right]^{\frac{1}{2}}.
    \end{align}
    As such, using only the one-sided McDiarmid's inequality, for every $t>0$ we have with probability $\geq 1 - \exp\Big\{-\frac{n\abs{\Sigma}t^{2}}{2C_{\kappa}\left[1+\varsigma_{\kappa,\Sigma}\left(\abs{\Sigma}-1\right)\right]}\Big\}$
    \begin{align} \label{mcd_MMD}
        d_{\mathcal{H}_{\kappa}}(\hat{\rho}_{n}, \rho) - \mathbb{E}[d_{\mathcal{H}_{\kappa}}(\hat{\rho}_{n}, \rho)] \leq t.
    \end{align}
    In order to upper bound the expectation, we use the symmetrization trick as follows
    \begin{align}
        &\mathbb{E}_{\{Z_{i}\}_{1}^{n} \sim \rho} \sup_{\norm{f}_{\mathcal{H}_{\kappa}} \leq 1} \abs{\frac{1}{n\abs{\Sigma}} \sum_{i=1}^{n} \sum_{j=1}^{\abs{\Sigma}} f(\sigma_{j} z_{i})- \mathbb{E}_{\rho} f} \nonumber \\ &= \mathbb{E}_{Z} \sup_{\norm{f}_{\mathcal{H}_{\kappa}} \leq 1} \abs{\frac{1}{n\abs{\Sigma}} \sum_{i=1}^{n} \sum_{j=1}^{\abs{\Sigma}} f(\sigma_{j} z_{i})- \mathbb{E}_{\{Z^{'}_{i}\}_{1}^{n} \sim \rho}\left(\frac{1}{n\abs{\Sigma}} \sum_{i=1}^{n} \sum_{j=1}^{\abs{\Sigma}} f(\sigma_{j} z^{'}_{i})\right)} \nonumber \\ &\leq \mathbb{E}_{Z,Z^{'}} \sup_{\norm{f}_{\mathcal{H}_{\kappa}} \leq 1} \abs{\frac{1}{n\abs{\Sigma}} \sum_{i=1}^{n} \sum_{j=1}^{\abs{\Sigma}} f(\sigma_{j} z_{i})- f(\sigma_{j} z^{'}_{i})} \nonumber \\ &= \mathbb{E}_{Z,Z^{'},\xi} \sup_{\norm{f}_{\mathcal{H}_{\kappa}} \leq 1} \abs{\frac{1}{n\abs{\Sigma}} \sum_{i=1}^{n} \xi_{i}\sum_{j=1}^{\abs{\Sigma}} f(\sigma_{j} z_{i})- f(\sigma_{j} z^{'}_{i})} \\ &\leq 2\sqrt{\frac{C_{\kappa}\left[1+\varsigma_{\kappa,\Sigma}\left(\abs{\Sigma}-1\right)\right]}{n\abs{\Sigma}}} \label{lemmaA.14},
    \end{align}
    where $(\xi_{1}, \xi_{2}, \cdots, \xi_{n}) \sim$ i.i.d. standard Rademacher and (\ref{lemmaA.14}) is due to (\citet{chen2023sample}, Lemma A.14). As such, \ref{mcd_MMD} implies that
    \begin{align*}
        d_{\mathcal{H}_{\kappa}}(\hat{\rho}_{n}, \rho) \leq 2\sqrt{\frac{C_{\kappa}\left[1+\varsigma_{\kappa,\Sigma}\left(\abs{\Sigma}-1\right)\right]}{n\abs{\Sigma}}}\left(1+\sqrt{\frac{1}{2}\ln{\frac{1}{\delta}}}\right) 
    \end{align*}
    holds with probability $1-\delta$ for $\delta > 0$. Observe that, the proof of this concentration bound acts as a generalization to the usual MMD, which arises in case $\abs{\Sigma}=1$. 
    
    Going back to \ref{MMD_frag}, we write 
    \begin{equation*}
        d_{\mathcal{H}_{\kappa}}\left(\phi_{\#}\hat{\mu}_{n}, \rho\right) - \sqrt{C_{\kappa}} \sup_{\mathcal{P}_{n}(\rho)}\Delta_{\Phi,n} \leq d_{\mathcal{H}_{\kappa}}\left(\phi_{\#}\hat{\mu}_{n}, \widehat{(\phi_{\#}\mu)}_n\right) + d_{\mathcal{H}_{\kappa}}\left(\hat{\rho}_{n}, \rho\right),
    \end{equation*}
    where the supremum is taken over possible estimators based on replicates from $\rho$ of size $n$. It becomes evident that the quantity on the left has finite expectation and is essentially bounded in probability. Based on the concentration found in Theorem \ref{MMD_IPT}, along with the bound as in Corollary \ref{IPT_ReLU_Group} we conclude that given $t > 0$
    \begin{align} \label{ult_conc_MMD}
         d_{\mathcal{H}_{\kappa}}\left(\phi_{\#}\hat{\mu}_{n}, \rho\right) &- \sqrt{C_{\kappa}} \sup_{\mathcal{P}_{n}(\rho)}\Delta_{\Phi,n} \leq \sqrt{t}(\sqrt{c_{g,n}L} + 2\sqrt{t}) + \mathcal{O}(\sqrt{c_{g,n}}(d^{2}n)^{-\frac{1}{2d}}) \\ \nonumber &+ \sqrt{\frac{2C_{\kappa}}{n}}\left[1 + \sqrt{\frac{1+\varsigma_{\kappa,\Sigma}\left(\abs{\Sigma}-1\right)}{\abs{\Sigma}}}\right] + \mathcal{O}(\sqrt{d D_{n}}N_{1}^{-\frac{2}{d}}N_{2}^{-\frac{2}{d}})
    \end{align}
    holds with probability at least $1 - 2\exp\Big(-\frac{2nt^{2}}{\max\left\{B^{2}_{x}, 4C_{\kappa}\left[1+\varsigma_{\kappa,\Sigma}\left(\abs{\Sigma}-1\right)\right]\right\}}\Big)$. Observe that, here $c_{g,n}$ is as described in Theorem \ref{MMD_IPT} and typically behave as $\mathcal{O}(n^{-\frac{1}{2}})$ due to the strong law. $N_{1}$ and $N_{2}$ specify the width $W = \mathcal{O}(d\floor{N_{1}^{\frac{1}{d}}} \vee N_{1}+1)$ and length $L = \mathcal{O}(N_{2})$ of the encoder. Applying a change in variables on (\ref{ult_conc_MMD}) we prove the theorem. 
\end{proof}

\begin{proof}[Proof of Theorem \ref{Recon_con}]
    Given an optimal encoder $E^{*}_{n}(t)$ that incurs latent loss $d_{\textrm{TV}}\left(E^{*}_{n}(t)_{\#}\hat{\mu}_{n}, \rho\right) \leq t$, fragmenting the reconstruction error according to (\ref{decoding}) yields
    \begin{equation*}
        W_{c_{x}}^{1}(\mu, (D \circ E^{*}_{n}(t))_{\#}\hat{\mu}_{n}) \leq W_{c_{x}}^{1}((D \circ E^{*}_{n}(t))_{\#}\hat{\mu}_{n}, D_{\#}\rho) + W_{c_{x}}^{1}(D_{\#}\rho, \hat{\mu}_{n}) + W_{c_{x}}^{1}(\hat{\mu}_{n}, \mu).
    \end{equation*}
    The decoder transform is constructed as $D = \phi^{'} \circ D_{0}$, where $\phi^{'}$ is according to lemma \ref{yang} and $D_{0}: \Omega_{z} \rightarrow \mathbb{R}$ is a linear map (or an ensemble of several). $D$ can be equivalently written as $D = \phi^{'} \circ \sigma_{I} \circ D_{0}$, where $\sigma_{I}$ is the identity activation, applied componentwise. As such, based on our definition, $D$ indeed belongs to $\Phi(W,L)_{k}^{d}$ with depth $\geq 3$. As discussed earlier, the resultant $D$ thus turns out to be Lipschitz continuous. If $c_{x}$ is taken as $L_{1}$, then observe
    \begin{align*}
        W_{c_{x}}^{1}\left((D \circ E^{*}_{n}(t))_{\#}\hat{\mu}_{n}, D_{\#}\rho\right) &\leq B_{x} d_{\textrm{TV}}\left((D \circ E^{*}_{n}(t))_{\#}\hat{\mu}_{n}, D_{\#}\rho\right) \\
        &= B_{x} \sup_{\omega \in \sigma(\mathcal{X})}\abs{E^{*}_{n}(t)_{\#}\hat{\mu}_{n}(D^{-1}(\omega)), \rho(D^{-1}(\omega))} \\ &\leq B_{x} \sup_{\omega^{'} \in \sigma(\mathcal{Z})}\abs{E^{*}_{n}(t)_{\#}\hat{\mu}_{n}(\omega^{'}), \rho(\omega^{'})} \\ &= B_{x} d_{\textrm{TV}}\left(E^{*}_{n}(t)_{\#}\hat{\mu}_{n}, \rho\right) \leq t B_{x},
    \end{align*}
    where the latter inequality is reached by taking supremum over all measurable sets belonging to the Borel $\sigma$-algebra on $\mathcal{Z}$ instead of the particular path directed by $D^{-1}$. Also, the definition of $D$ implies that given arvitrary $\varepsilon > 0$, $W_{c_{x}}^{1}(D_{\#}\rho, \hat{\mu}_{n}) < \varepsilon$ (lemma \ref{yang}). 
    
    As such, the concentration of the reconstruction error is essentially determined by the statistical estimation error in the input space. Taking expectation over the samples we write
    \begin{align*}
        \mathbb{E}\left[W_{c_{x}}^{1}(\mu, (D \circ E^{*}_{n}(t))_{\#}\hat{\mu}_{n})\right] - tB_{x} \leq \mathcal{O}(n^{-\frac{1}{d}}),
    \end{align*}
    whenever $d \geq 3$. Using the naive estimator, this is the sharpest rate one can achieve. However, to remove the influence of the dimensionality $d$, here also \cite{Weed2017Sharp}'s device can be applied (see Remark \ref{curse_dim}). 
    
    The bound, based on the empirical estimator does not appreciate the smoothness of the input density $p_{\mu}$. On the other hand, given $m>0$, if $m_{x} > m$ we have $\mathcal{W}^{m_{x},p}_{L} \subset \mathscr{B}^{m}_{pq}(L^{'})$, for some $L^{'} > 0$, where $1 \leq q \leq \infty$ and $1 \leq p < \infty$ i.e. belonging to the general Besov space (\citet{gine2021mathematical}, Section 4.3). Thus, if wavelet estimates (as in \citet{weed2019estimation}) are deployed instead, the rate can be expected to be $\mathcal{O}(n^{-\frac{1+m}{d+2m}})$ given $d \geq 3$. 
\end{proof}

\begin{proof}[Proof of Theorem \ref{conta}]
    Let us consider the kernel $\kappa(x,y)$ to be of the form $\kappa(x-y)$, without loss of generality. Given replicates $X_{1}, \cdots, X_{n} \sim \Tilde{p}$, the density estimate based on transformed (due to $G \in \mathcal{G}$) samples is given as 
    \begin{align*}
        \hat{p}_{h}(x) = \frac{1}{nh^d}\sum^{n}_{i=1} \kappa\left(\frac{G(X_{i}) - x}{h}\right),
    \end{align*}
    where $h$ is the bandwidth. Decomposing the $L_{1}$ reconstruction error yields
    \begin{align}
        \abs{\hat{p}_{h}(x) - p_{\mu}(x)} \leq \abs{\hat{p}_{h}(x) - \mathbb{E}[\hat{p}_{h}(x)]} + \abs{\mathbb{E}[\hat{p}_{h}(x)] - p_{\mu}(G^{-1}(x))} + \abs{p_{\mu}(x) - p_{\mu}(G^{-1}(x))}.
    \end{align}
    Now, observe that
    \begin{align*}
        \mathbb{E}\abs{\hat{p}_{h}(x) - \mathbb{E}[\hat{p}_{h}(x)]} &\leq \sqrt{\mathbb{E}\left(\hat{p}_{h}(x) - \mathbb{E}[\hat{p}_{h}(x)]\right)^{2}} \\ &= \sqrt{\frac{1}{n}\textrm{Var}_{\Tilde{p}}\left[\frac{1}{h^d} \kappa\left(\frac{G(X) - x}{h}\right)\right]} \\ &  \overset{(1)}{=}
        \sqrt{\frac{1}{n}\textrm{Var}_{\Tilde{p}}\left[\frac{1}{h^d} \kappa\left(\frac{X - G^{-1}(x)}{h}\right)\right]}  \\ &\overset{(2)}{\lesssim} \sqrt{\frac{1}{nh^d}}, 
    \end{align*}
    where ($1$) is due to the transformation invariance of the kernel. The inequality ($2$) is obtained as a result of $\int \kappa^{2}(u) du <\infty$. The suppressed constant in the process is the upper bound to the density $\Tilde{p}$. For the bias term, we write
    \begin{align*}
        \abs{\mathbb{E}[\hat{p}_{h}(x)] - p_{\mu}(G^{-1}(x))} &= \abs{\frac{1}{h^d} \mathbb{E}_{\Tilde{p}}\left[\kappa\left(\frac{G(X_{i}) - x}{h}\right)\right] - p_{\mu}(G^{-1}(x))} \\ &\leq \abs{\frac{1}{h^d} \mathbb{E}_{\Tilde{p}}\left[\kappa\left(\frac{X - G^{-1}(x)}{h}\right)\right] - \frac{1}{h^d} \mathbb{E}_{p_{\mu}}\left[\kappa\left(\frac{Y - G^{-1}(x)}{h}\right)\right]}  \\ &\quad + \abs{\frac{1}{h^d} \mathbb{E}_{p_{\mu}}\left[\kappa\left(\frac{Y - G^{-1}(x)}{h}\right)\right] - p_{\mu}(G^{-1}(x))} \\ &\overset{(3)}{\lesssim} \frac{1}{h^d} \mathbb{E}_{\Tilde{p}, p_{\mu}} \abs{\frac{X-Y}{h}} + h^{m_{x}} \lesssim \frac{\epsilon}{h^{2d}} \vee h^{m_{x}},
    \end{align*}
    where in ($3$) we use the invariance of $\kappa$ for the first term, and the second bound is obtained using \citet{gine2021mathematical}, Proposition 4.3.33. We emphasize the fact that kernels satisfying Lipschitz continuity are commonly taken in practice, which are readily invariant to translations. The latter inequality is due to the contamination model. As such, assuming without loss of generality that the transform $G$ preserves the origin, we get
    \begin{equation*}
        \mathbb{E}\abs{\hat{p}_{h}(0) - p_{\mu}(0)} \lesssim \sqrt{\frac{1}{nh^d}} \vee \frac{\epsilon}{h^{2d}} \vee h^{m_{x}}.
    \end{equation*}
    Taking $h = n^{-\frac{1}{d + 2m_{x}}} \vee \epsilon^{\frac{1}{2d+m_{x}}}$, we prove the theorem.
\end{proof}

\section*{Appendix: Experiments and Simulations}\label{appB} 
    
\subsection*{Encoded vs. Latent Distributions}

To check the efficacy of a WAE-encoding statistically, we perform two-sample non-parametric tests of equality on target latent and encoded observations. \citet{peacock1983two} suggested a multi-dimensional generalization to the well-known Kolmogorov–Smirnov test, which, however, has high computational complexity. In our study, we identify the test suggested by \citet{fasano1987multidimensional} (FF) as a suitable alternative based on its manageable complexity without sacrificing the power and consistency\footnote{We follow the fasano.franceschini.test implementation \citep{puritz2021fasano} on R.}. Many suggestions have been made to improve the reliability of two-sample tests in higher dimensions since. However, given the ease of implementation, we use FF's version of the KS test. To ascertain our findings, we additionally carry out a referral test of equality of distributions, based on kernelized distances between pairs of observations, namely the Cram\'{e}r test \citep{baringhaus2004new}.
Unlike FF, the test statistic corresponding to Cram\'{e}r\footnote{We implement the cramer package \citep{franz2006cramer} in R with underlying kernel specified to $\kappa(z) = \sqrt{z}/2$.} is not distribution-free, and requires bootstrapping to obtain the $p$-value. We utilize two of such methods, namely the usual \textit{Monte-Carlo} and calculating the approximate \textit{eigenvalues} as the weights to the limiting distribution (of the statistic). Test results on the Five Gaussian data at $5\%$ level of significance, given $\lambda = 0.8$, are as follows:

\begin{table}[!ht]
    \centering
    \caption{Two-sample tests of equality on latent and encoded distributions.}
    \label{tab:parameters}
    \vspace{5pt}
    \footnotesize
    \begin{threeparttable}
    \begin{tabular}{l|c|c|cc}
    \toprule
    Architecture & Latent Distribution & KS test & \multicolumn{2}{c}{Cram\'{e}r test} \\
    \cmidrule{4-5}
    &  &  & Monte-Carlo & Eigenvalue\\
    \midrule
    \multirow{2}{*}{WAE-GAN} & Gaussian & \xmark & \xmark & \xmark \\
    & Exponential & \xmark & \xmark & \xmark \\
    \midrule
    \multirow{2}{*}{WAE-WAE} & Gaussian & \xmark & \xmark & \xmark \\
    & Exponential & \xmark & \xmark & \xmark \\
    \bottomrule
    \end{tabular}
    \begin{tablenotes}
         \item $^{*}$The decisions `Accept' and `Reject' against the null hypothesis that the two distributions are equal are denoted by the symbols (\cmark) and (\xmark) respectively.
    \end{tablenotes}
    \end{threeparttable}
\end{table}

The test results corroborate our theoretical findings. Though we only establish an upper bound to the latent loss, it is apparent that there lies an optimization error due to the minimum distance estimate. As such, under a metrizing measure of discrepancy, the target latent law and the encoded estimate must be distinct in distribution. The rejection of the null hypothesis reiterates the same observation. 

\begin{figure}[ht]
    \centering
    \includegraphics[width=\linewidth]{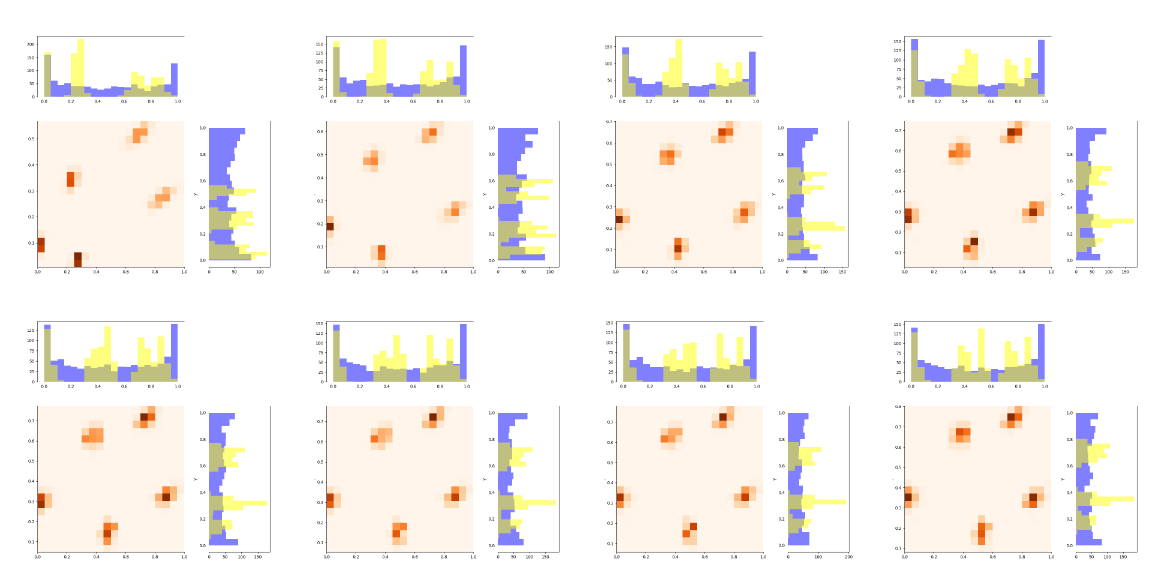}
    \caption{Concentration of bin estimates under ReLU encoders ({\color{yellow}yellow}) for latent Beta$(0.5,0.8)$ copula ({\color{blue}blue}), over epochs $(200,400,600,800,1000,1200,1400,2000)$ (left to right) in a WAE-GAN setup.}
    \label{fig:JS_Beta_Latent_whole}
\end{figure}

It becomes even more clear looking at the histograms corresponding to samples from the two, overlaid. The interesting observation from Fig \ref{fig:JS_Beta_Latent_whole}, \ref{fig:MMD_Gauss_Latent_whole} is the visual manifestation of information preservation. Semantic information, in the form of cluster structures originally present in the data set, remains intact in encoded distributions while trying to maximize similarity with their target counterparts. The evolution of this `maximization' is clear from the histograms obtained over epochs. Another viewpoint that attests to this finding is the quantile-quantile plot [Fig \ref{fig:MMD_qqplot}] of the marginals.

\begin{figure}
    \centering
    \includegraphics[width=\linewidth]{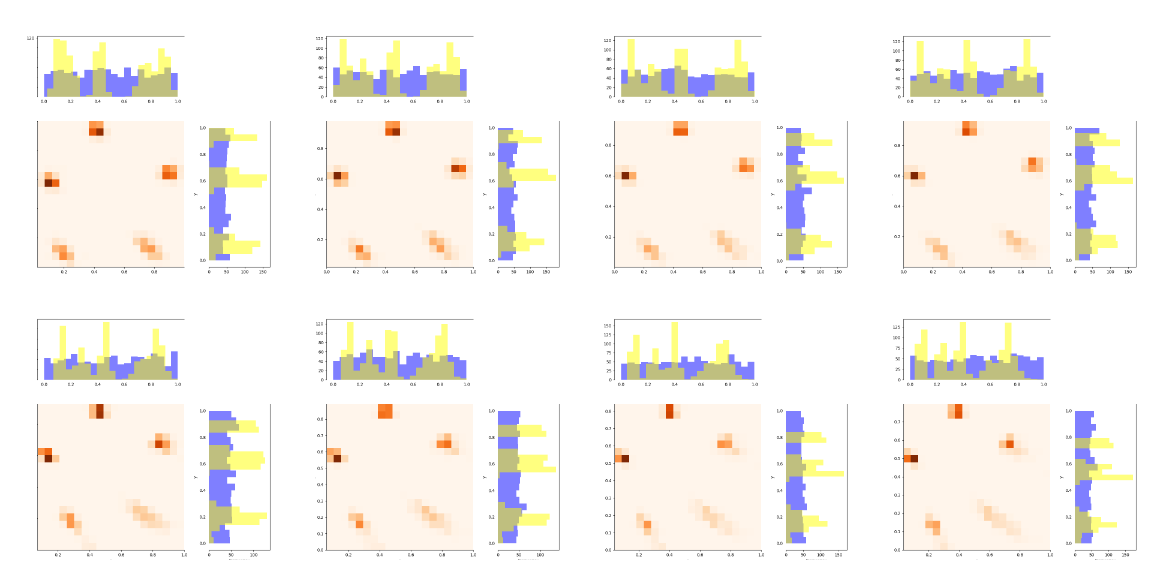}
    \caption{Concentration of bin estimates under ReLU encoders ({\color{yellow}yellow}) for latent bivariate Gaussian distribution ({\color{blue}blue}), over epochs $(200,400,600,800,1200,1400,1600,1800)$ (left to right) in a WAE-MMD setup with regularization $\lambda=0.1$.}
    \label{fig:MMD_Gauss_Latent_whole}
\end{figure}

\begin{figure}
    \begin{subfigure}{\textwidth}
  \includegraphics[width=\linewidth]{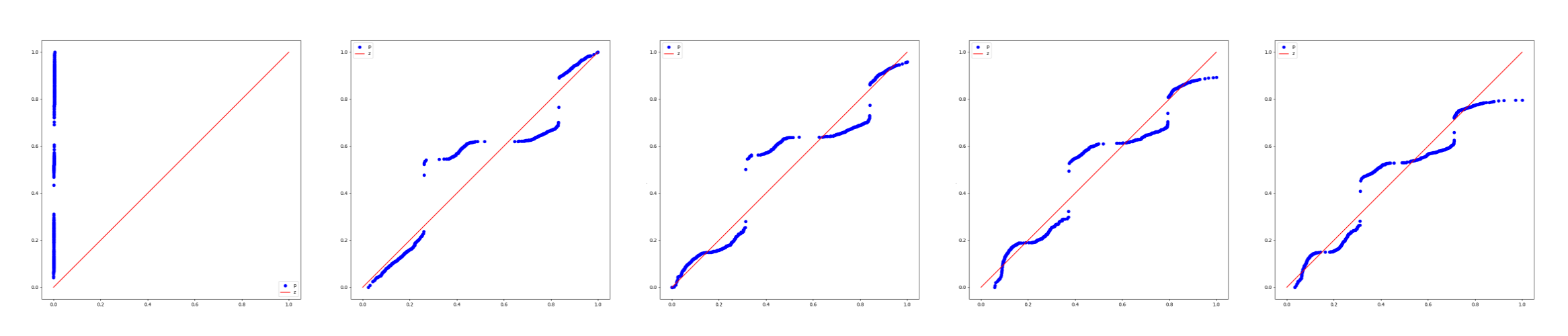}
  \caption{Gaussian}
  \label{fig:MMD_Gauss_qqplot}
\end{subfigure} 
\begin{subfigure}{\textwidth}
  \includegraphics[width=\linewidth]{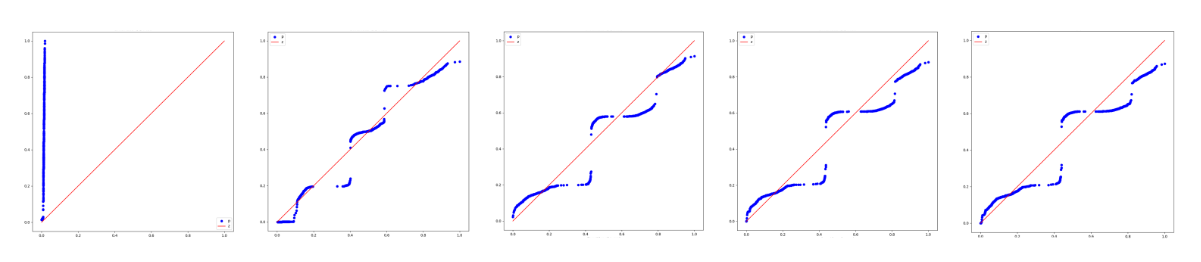}
  \caption{Beta Copula}
  \label{fig:MMD_Beta_qqplot}
\end{subfigure} 
\caption{Propagation of quantile-quantile (QQ) plots of marginals corresponding to encoded vs latent distribution under ReLU encoders, over epochs $(0,200,800,1400,1800)$ in a WAE-MMD setup with regularization $\lambda=0.1$.}
\label{fig:MMD_qqplot}
\end{figure}


\begin{figure}
     \centering 
    \begin{subfigure}{0.32\textwidth}
  \includegraphics[width=\linewidth]{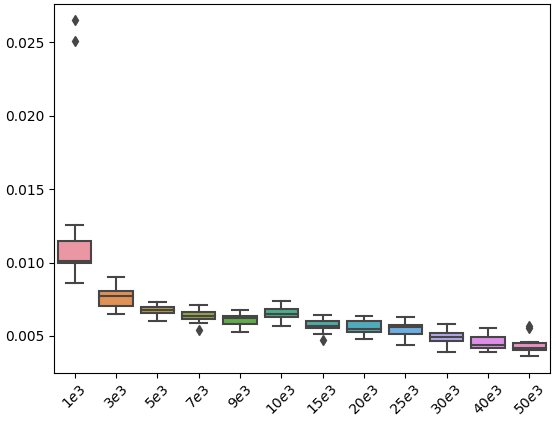}
  \caption{Gaussian}
  \label{fig:GB_recon_JS_Gauss_Group}
\end{subfigure} 
\begin{subfigure}{0.32\textwidth}
  \includegraphics[width=\linewidth]{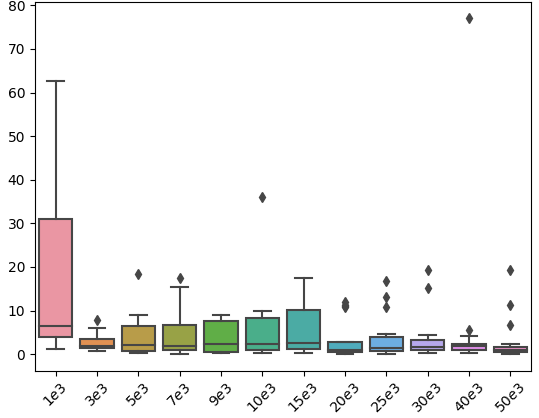}
  \caption{Beta}
  \label{fig:GB_recon_MMD_Beta_Group}
\end{subfigure} 
\begin{subfigure}{0.32\textwidth}
  \includegraphics[width=\linewidth]{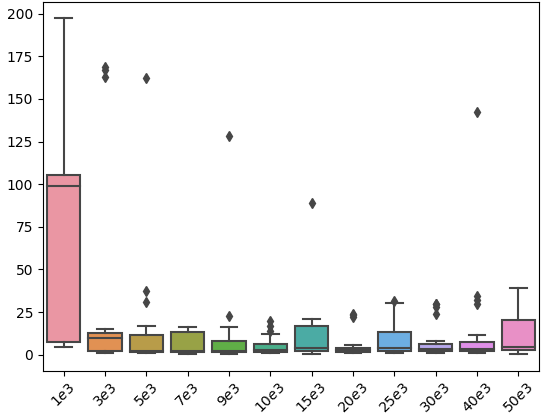}
  \caption{Exponential}
  \label{fig:GB_recon_MMD_Exp_Group}
\end{subfigure} 
\caption{Reconstruction error of Five Gaussian data under (a) JS and (b), (c) sample corrected ($\times n^{\frac{1}{2}}$) MMD latent loss, using GroupSort encoders (grouping 2).}
\label{fig:GB_recon_MMD_Group}
\end{figure}


\begin{figure}
    \centering 
\begin{subfigure}{0.31\textwidth}
  \includegraphics[width=\linewidth]{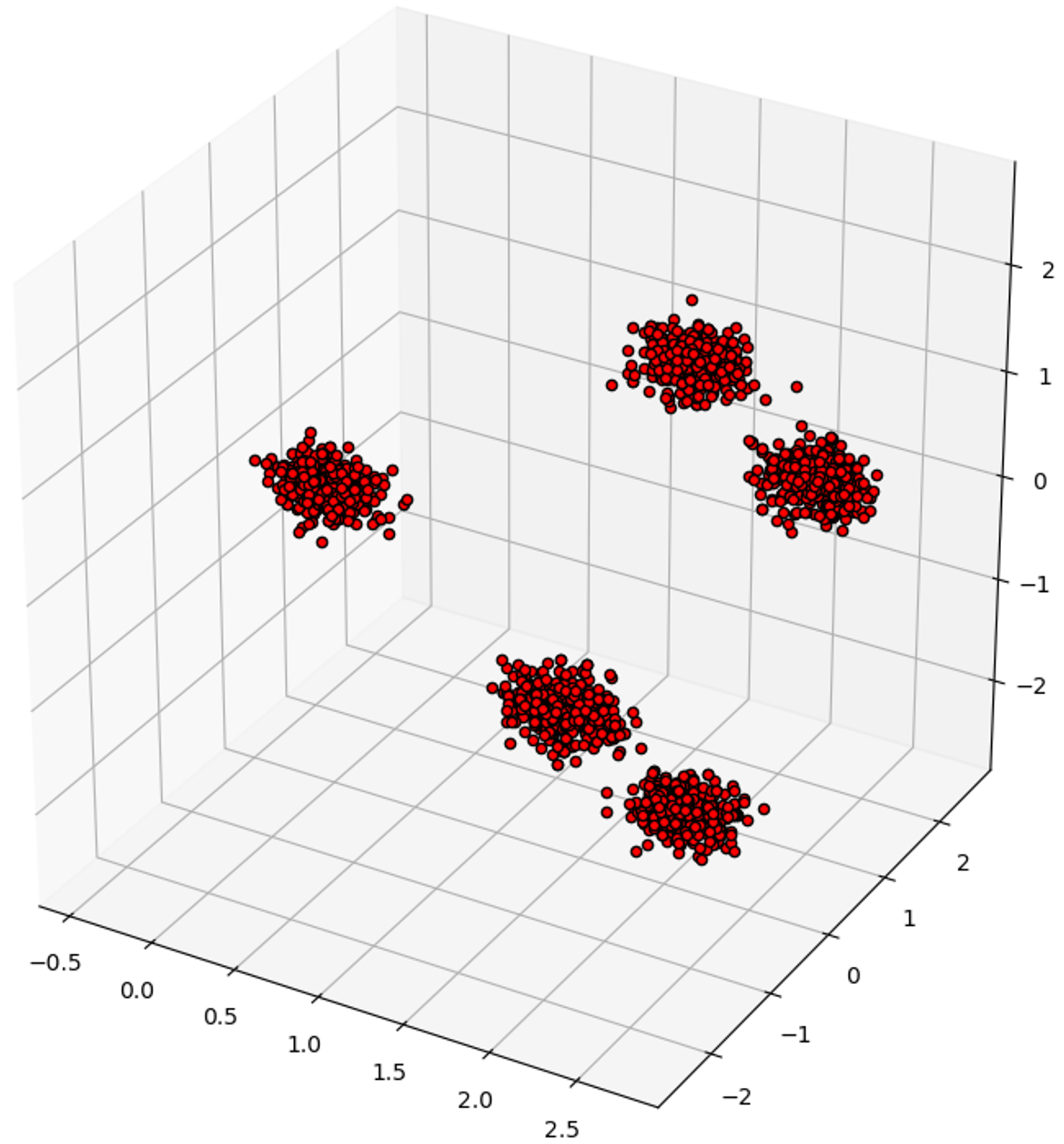}
  \caption{Actual Data}
  \label{fig:Target_FG_Cauchy_MMD}
\end{subfigure}\hspace{3pt} 
\begin{subfigure}{0.31\textwidth}
  \includegraphics[width=\linewidth]{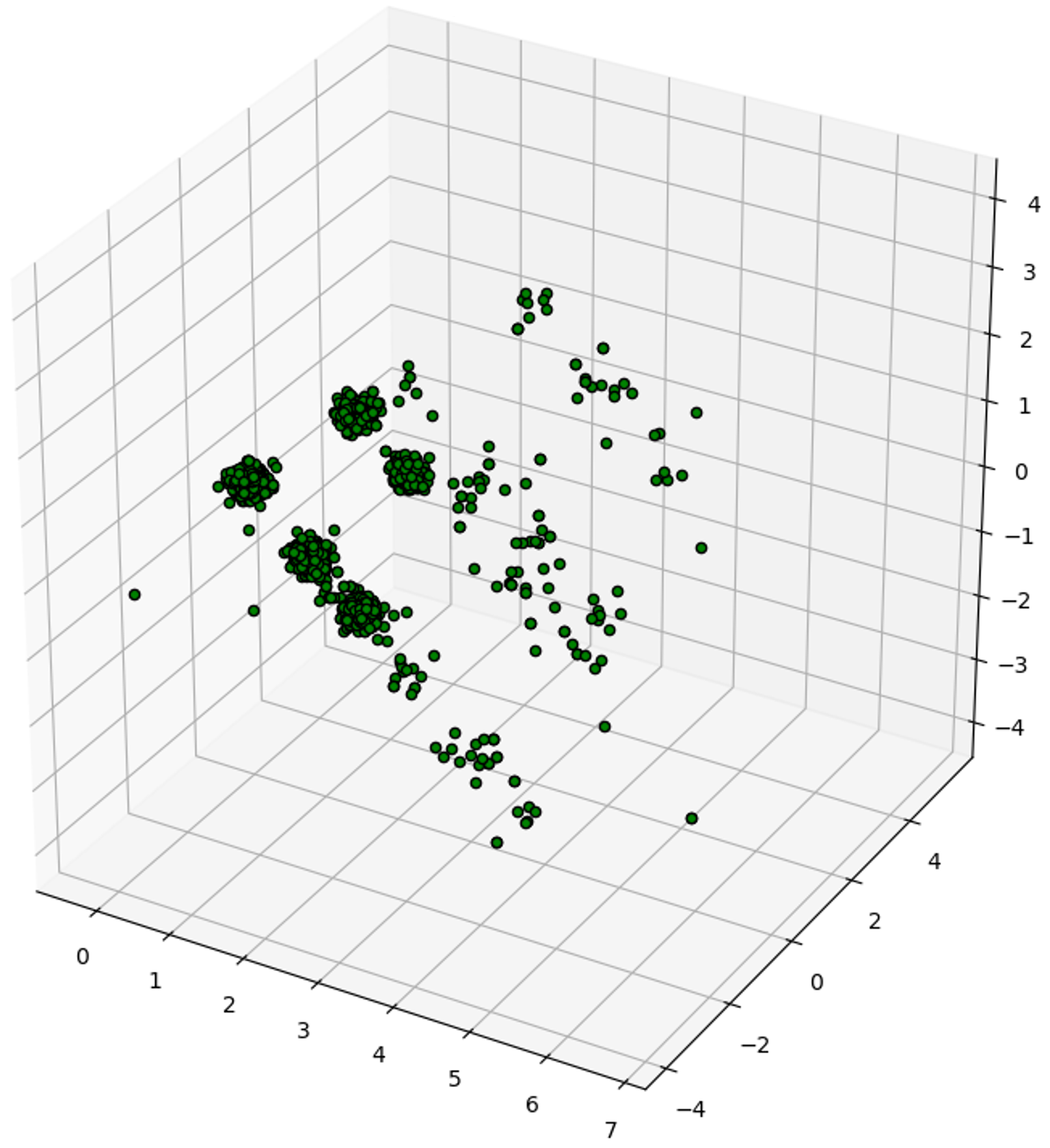}
  \caption{Contaminated Data}
  \label{fig:Cont_FG_Cauchy_MMD}
\end{subfigure}\hspace{3pt} 
\begin{subfigure}{0.31\textwidth}
  \includegraphics[width=\linewidth]{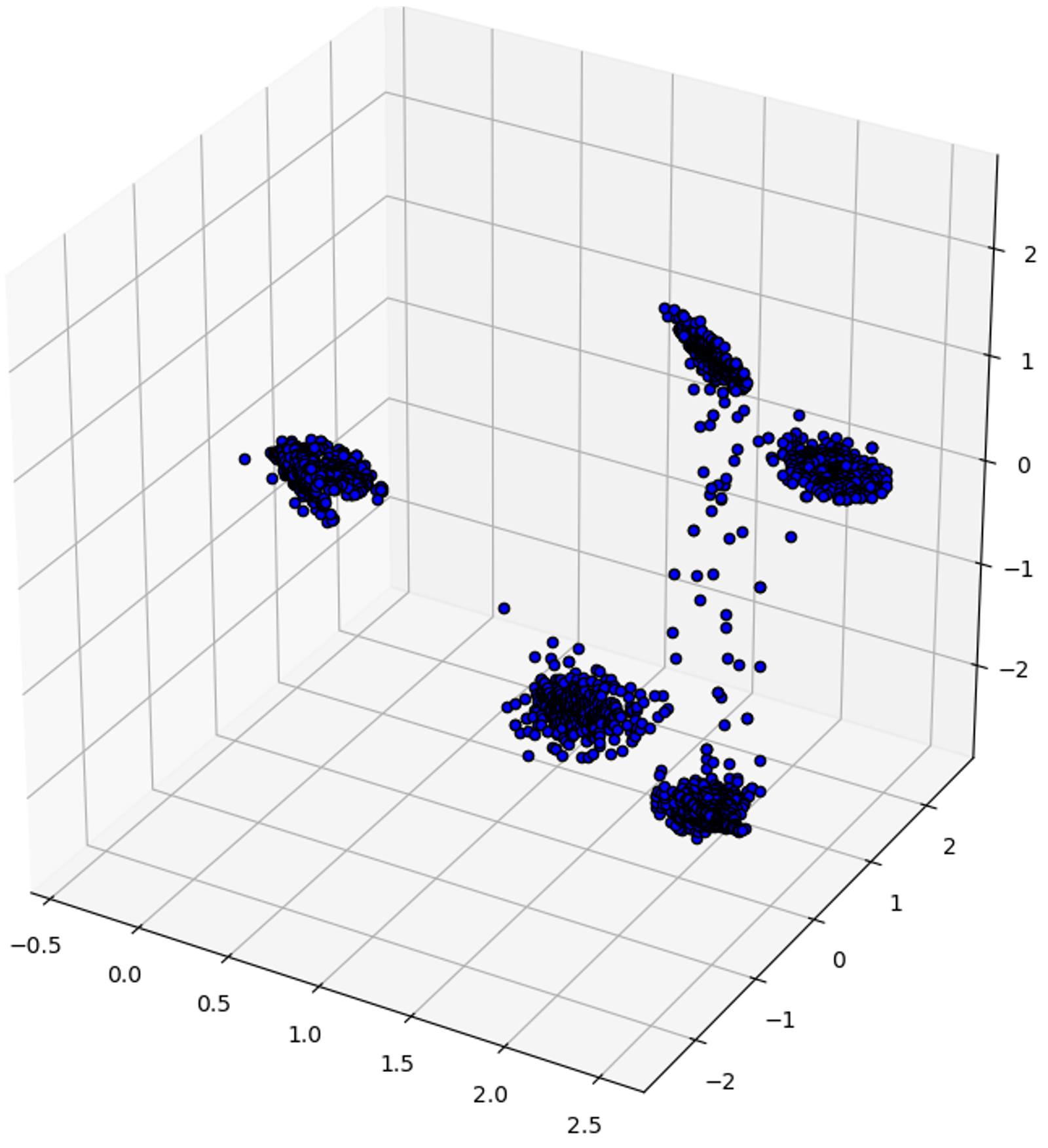}
  \caption{Reconstructed Data}
  \label{fig:Recon_FG_Cauchy_MMD}
\end{subfigure}\hfil 
\caption{Reconstructed samples ($n=10,000$) from Five Gaussian dataset with $10\%$ observations contaminated at level $0.2$, under MMD latent loss. The corrupting distribution is taken to be standard tri-variate Cauchy.}
\label{fig:MMD_Robust_FG_Cauchy}
\end{figure}

\begin{figure}
    \centering 
\begin{subfigure}{0.55\textwidth}
  \includegraphics[width=\linewidth]{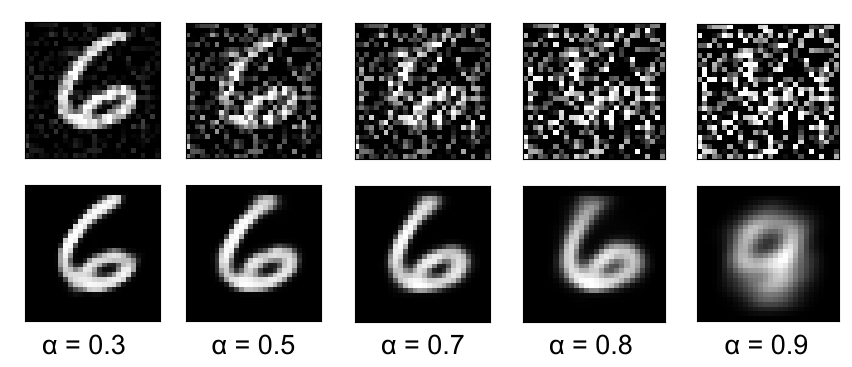}
  \caption{Gaussian}
  \label{fig:Gauss_diff}
\end{subfigure}\hspace{4pt} 
\begin{subfigure}{0.6\textwidth}
  \includegraphics[width=\linewidth]{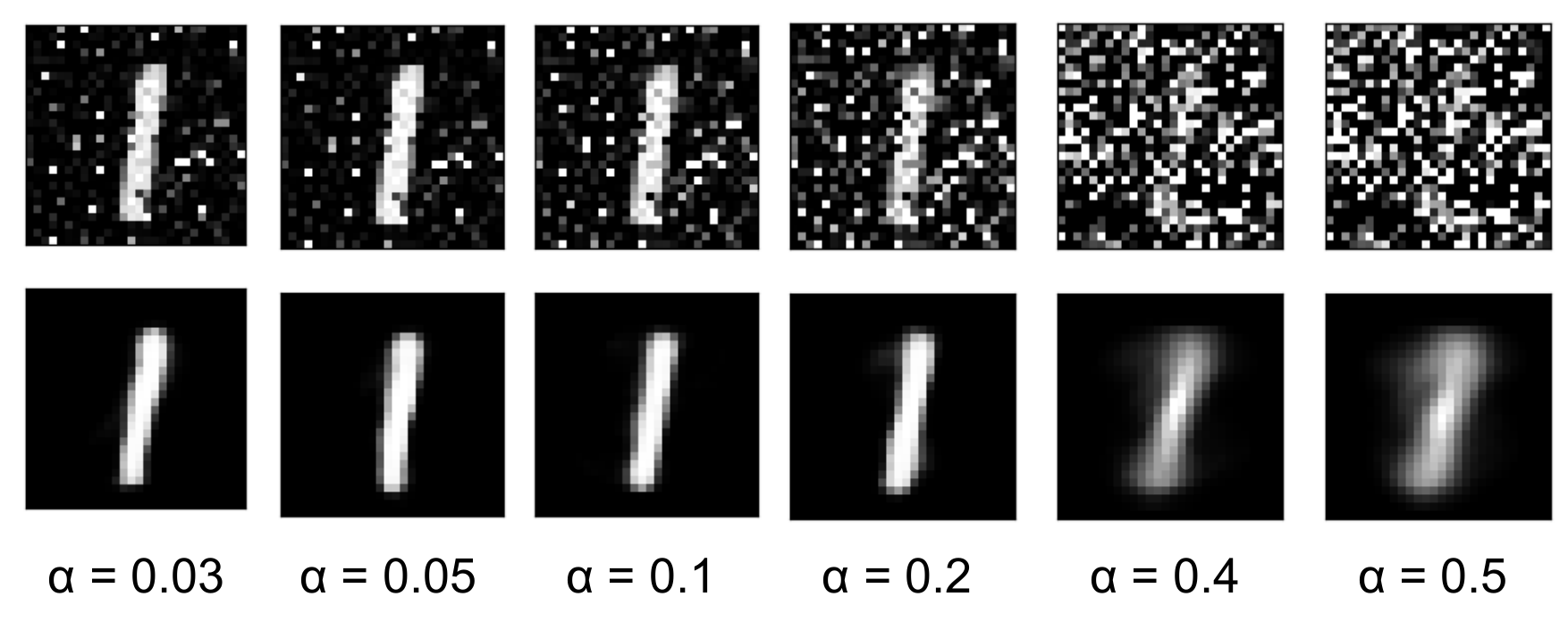}
  \caption{Cauchy}
  \label{fig:Cauchy_diff}
\end{subfigure}\hfil 
\caption{Reconstructed samples ($n=10,000$) from Five Gaussian dataset with $10\%$ observations contaminated at level $0.2$, under MMD latent loss. The corrupting distribution is taken to be standard tri-variate Cauchy.}
\label{fig:diffusion_robust}
\end{figure}


\bibliographystyle{elsarticle-num-names.bst} 
\bibliography{Ref}

\end{document}